\documentclass[final]{siamltex}
\usepackage{hyperref} %BUG if put after
\pdfoutput=1

\usepackage{url}

\usepackage[english]{babel}

\usepackage{amsmath,amssymb} % define this before the line numbering.
\usepackage{xcolor}
\usepackage{gpeyre}
\usepackage{array}% http://ctan.org/pkg/array
\usepackage{rotating}
\usepackage{graphicx}

\DeclareGraphicsExtensions{.png,.pdf,.eps,.jpg,.jpeg} 
\graphicspath{{./images/}}

\hypersetup{  
   bookmarks=true,
   backref=true,
   pagebackref=false,
   colorlinks=true,
   linkcolor=blue,
   citecolor=red,
   urlcolor=blue,
   pdftitle={Regularized Optimal Transport},
   pdfauthor={S. Ferradans et al.},
   pdfsubject={Regularized Optimal Transport}
}

\newcommand{\U}{\mathbb{I}}

\newcommand{\Perm}{\Pp}

\newcommand{\Matr}{\Ss}
\newcommand{\regul}[1]{J_{p,q}(#1)}

\newcommand{\cost}{C}
\newcommand{\Cost}[2]{ \cost_{#1,#2} }

\title{Regularized Discrete Optimal Transport}

\author{Sira Ferradans\footnotemark[2] 
		\and
		Nicolas Papadakis\footnotemark[3]	
		\and~\\ 
        Gabriel Peyr\'e\footnotemark[2] 
        \and 
        Jean-Fran\c{c}ois~Aujol\footnotemark[3] \ \footnotemark[4]
        }

\begin{document}

\maketitle

\renewcommand{\thefootnote}{\fnsymbol{footnote}}
\footnotetext[1]{This work has been supported by the European Research Council (ERC project SIGMA-Vision) and the French National Research Agency (project NatImages).}
\footnotetext[2]{CEREMADE, Universit\'e Paris-Dauphine, Place du Marechal De Lattre De Tassigny, 75775 PARIS CEDEX 16, FRANCE.}
\footnotetext[3]{IMB, UMR 5251, Universit\'e Bordeaux 1, 351 cours de la Lib\'eration F-33405 TALENCE, France}
\footnotetext[4]{Member of Institut Universitaire de France.}

\begin{abstract}
	This article introduces a generalization of  the discrete optimal transport, with applications to color image manipulations. This new formulation includes a relaxation of the mass conservation constraint and a regularization term.  These two features are crucial for image processing tasks, which necessitate to take into account families of multimodal histograms, with large mass variation across modes.  	 
	 The corresponding relaxed and regularized transportation problem is the solution of a convex optimization problem. Depending on the regularization used, this minimization can be solved using standard linear programming methods or first order proximal splitting schemes.
	 The resulting transportation plan can be used as a color transfer map, which is robust to mass variation across images color palettes. Furthermore, the regularization of the transport plan helps to remove colorization artifacts due to noise amplification.
	We also extend this framework to the computation of barycenters of distributions. The barycenter is the solution of an optimization problem, which is separately convex with respect to the barycenter and the transportation plans, but not jointly convex. A block coordinate descent scheme converges to a stationary point of the energy. We show that the resulting algorithm can be used for color normalization across several images. The relaxed and regularized barycenter defines a common color palette for those images. Applying color transfer toward this average palette performs a color normalization of the input images.  
\end{abstract}

% \keywords{Optimal Transport, color transfer, variational regularization, convex optimization, proximal splitting, manifold learning}

%%%%%%%%%%%%%%%%%%%%%%%%%%%%%%%%%%%%%%%%%%%%%%%%
\section{Introduction}

A large class of image processing problems involves probability densities estimated from local or global image features. In contrast to most distances from information theory (e.g. the Kullback-Leibler divergence), optimal transport (OT) takes into account the spatial location of the density modes~\cite{Villani03}. Furthermore, it also provides as a by-product a warping (the so-called transport plan) between the densities. This plan can be used to perform image modifications such as color transfer. However, an important flaw of this OT plan is that it is in general highly irregular, thus introducing unwanted artifacts in the modified images. In this article, we propose a variational formalism to relax and regularize the transport. This novel regularized OT improves visually the results for color image modifications. 

\subsection{Color Normalization and Color Transfer}

The problem of imposing some histogram on an image has been tackled since the beginning of image processing. Classic problems are histogram equalization or histogram specification (see for example~\cite{Gonzalez:2001}). Given two images, the goal of color transfer is to impose on one of the images the histogram of the other one. An approach to color transfer based on matching statistical properties (mean and covariance) is proposed by Reinhard et al.~\cite{Reinhard01} for the $\ell\al\beta$ color space, and generalized by Xiao and Ma~\cite{Xiao:2006} to any color space. Wang and Huang~\cite{WangH04} use similar ideas to generate a sequence of the same image with a changing histogram.  Morovic and Sun~\cite{Morovic03} and Delon~\cite{Delon04} show that histogram transfer is directly related to the OT problem.

A special case of color transfer is color normalization where the goal is to impose the same histogram, normally some ``average'' histogram, on a set of different images. An application for the color balancing of videos is proposed by Delon~\cite{Delon:2006} to correct flickering in old movies. In the context of canceling illumination, this problem is also known as color constancy and it has been thoroughly studied by Land and McCann who propose the Retinex theory (see~\cite{Land:71} and~\cite{Amestoy09} for a modern formulation). Canceling the illumination of a scene is an important component in the computer vision pipeline, and it is regularly used as a preprocessing to register/compare several images taken with different cameras or illumination conditions, as a preprocessing before registration, see~\cite{Csink98} for instance. 

%%%%%%%%%%%%%%%%%%%%%%%%%%%%%%%%%%%%%%%%%%%%%%%%%%%%%%%%%%%%%
\subsection{Optimal Transport and Imaging}
\label{subsec-ot-imaging}

\paragraph{Discrete optimal transport}

The discrete OT is the solution of a convex linear program originally introduced by Kantorovitch~\cite{Kantorovitch-OT}. It corresponds to the convex relaxation of a combinatorial problem when the densities are sums of the same number of Dirac masses. This relaxation is tight (i.e. the solution of the linear program is an assignment) and it extends the notion of OT to an arbitrary sum of weighted Diracs, see for instance~\cite{Villani03}. Although there exist dedicated linear solvers (transportation simplex~\cite{Dantzig-Book}) and combinatorial algorithms (such as the Hungarian~\cite{Kuhn-hungarian} and auction algorithms~\cite{Bertsekas1988}), computing OT is still a challenging task for densities composed of thousands of Dirac masses. 

\paragraph{Optimal transport distance}

The OT distance (also known as the Wasserstein distance or the Earth Mover distance) has been shown to produce state of the art results for the comparison of statistical descriptors, see for instance~\cite{Rubner98}. Image retrieval performance as well as computational time are both greatly improved by using non-convex cost functions, see~\cite{Pele-ICCV}. 

%%% NICOLAS: If we do not extend this paragraph, we may merge the 2 paragraphs in a single one: Optimal transport distance and map

% Many methods involving Wasserstein distance have been proposed to solve problems such as data  retrieval, texture analysis, histogram-based image segmentation or color transfer \cite{}. % CITATIONS NEEDED

%%
\paragraph{Optimal transport map}

Another line of applications of OT makes use of the transport plan to warp an input density onto another. OT is strongly connected to fluid dynamic partial differential equations~\cite{Benamou00}. These connections have been used to perform image registration~\cite{haker-ijcv}.  The estimation of the transport plan is also an interesting way of tackling  the challenging problem of color transfer between images, see for instance~\cite{Reinhard01,Morovic03,McCollum07}. For grayscale images, the usual histogram equalization algorithm corresponds to the application of the 1-D OT plan to an image, see for instance~\cite{Delon04}. It thus makes sense to consider the 3-D OT as a mathematically-sound way to perform color palette transfer, see for instance~\cite{Pitie07} for an approximate transport method. When doing so, it is important to cope with variations in the modes of the color palette across images, which makes the mass conservation constraint of OT problematic. A workaround is to consider parametric densities such as Gaussian mixtures and defines ad-hoc matching between the components of the mixture, see~\cite{Tai-cvpr-colortransfer}. In our work, we tackle this issue by defining a novel notion of OT well adapted to colors manipulation. 

\paragraph{Optimal transport barycenter}

It is natural to extend the classical barycenter of points to barycenter of densities by minimizing a weighted sum of OT distances toward a family of input distributions. In the special case of two input distributions, this corresponds to the celebrated displacement interpolation defined by McCann~\cite{mccann1997convexity}. Existence and uniqueness of such a barycenter is proved by Agueh and Carlier~\cite{Carlier_wasserstein_barycenter}, which also show the equivalence with the multi-marginal transportation problem introduced by Gangbo and {\'S}wi\c{e}ch~\cite{gangbo1998optimal}.  Displacement interpolation (i.e. barycenter between a pair of distributions) is used by Bonneel et al.~\cite{Bonneel-displacement} for computer graphics applications.  Rabin et al.~\cite{Rabin_ssvm11} apply this OT barycenter for texture synthesis and mixing. The image mixing is achieved by computing OT barycenters of empirical distributions of wavelet coefficients.  A similar approach is proposed by Ferradans et al.~\cite{2013-ssvm-mixing} for static and dynamic texture mixing using Gaussian distributions. \\

%%%%%%%%%%%%%%%%%%%%%%%%%%%%%%%%%%%%%%%%%%%%%%%%%%%%%%%%%%%%%
\subsection{Regularized and relaxed transport}
\label{subsec-regul-intro}

% Another kind of artifacts can nevertheless be still observed with these last approaches. Pixels that were originally close in the color space can be very different after the color transfer. Generalizing the Optimal Transport framework with regularity priors on the transport map would therefore be a good solution in this case. Few theoretical results exist on the regularity of the transport map \cite{}, and no satisfying non combinatorial algorithms have been proposed up to our knowledge. 

%%
\paragraph{Removing transport artifacts}

The OT map between complicated densities is usually irregular. Using directly this transport plan to perform color transfer creates artifacts and amplifies the noise in flat areas of the image. Since the transfer is computed over the 3-D color space, it does not take into account the pixel-domain regularity of the image. The visual quality of the transfer is thus improved by denoising the resulting transport using a pixel-domain regularization either as a post-processing~\cite{Papadakis_ip11} or by solving a variational problem~\cite{Papadakis_ip11,Rabin_icip11}.

% The perfect transfer of color is not satisfying in real applications and one may generally observe outliers in the final images. Indeed, as the transfer is realized in the color space, it does not take into account the fact that coherent colors should be transferred to neighbor pixels. As a consequence, methods have been proposed to consider the spatial nature of images and model some regularity priors on the image domain. In \cite{Papadakis_ip11}, the color transfer is formalized as an energy minimization problem in the image domain, which allows directly incorporating spatial regularization of the colors. The energy then involves the $L_2$ distance between cumulated color histograms instead of relying on the Wasserstein distance. The post-regularization of the  image color has also been proposed in \cite{Rabin_ip11}, where the color transfer is realized with the Sliced Wasserstein Distance.

%%
\paragraph{Transport regularization}

A more theoretically grounded way to tackle the problem of colorization artifacts should use directly a regularized OT. This corresponds to adding a regularization penalty to the OT energy. This however leads to difficult non-convex variational problems, that have not yet been solved in a satisfying manner either theoretically or numerically. The only theoretical contribution we are aware of is the recent work of Louet and Santambrogio~\cite{louet-regularizaton-1d}. They show that in 1-D the (un-regularized) OT is also the solution of the Sobolev regularized transport problem.
\paragraph{Graph regularization and matching}

For imaging applications, we use regularizations built on top of a graph structure connecting neighboring points in the input density. This follows ideas introduced in manifold learning~\cite{isomap}, that have been applied to various image processing problems, see for instance~\cite{elmoataz-graph}. Using graphs enables us to design regularizations that are adapted to the geometry of the input density, that often has a manifold-like structure. 

This idea of graph-based regularization of OT can be interpreted as a soft version of the graph matching problem, which is at the heart of many computer vision tasks, see~\cite{Belongie-graph-match,Yefeng-graph-match}. Graph matching is a quadratic assignment problem, known to be NP-hard to solve.  Similarly to our regularized OT formulation, several convex approximations have been proposed, including for instance linear programming~\cite{Almohamad-graph-match} and SDP programming~\cite{schellewald-ivc}. 

\paragraph{Transport relaxation}

The result of Louet and Santambrogio~\cite{louet-regularizaton-1d} is deceiving from the applications point of view, since it shows that, in 1-D, no regularization is possible if one  maintains a 1:1 assignment between the two densities. This is our first motivation for introducing a relaxed transport which is not a bijection between the densities.  Another (more practical) motivation is that relaxation is crucial to solve imaging problems such as color transfer. Indeed, the color distributions of natural images are multi-modals. An ideal color transfer should match the modes together. This cannot be achieved by classical OT because these modes often do not have the same mass. A typical example is for two images with strong foreground and background dominant colors (thus having bi-modal densities) but where the proportion of pixels in foreground and background are not the same. Such simple examples cannot be handled properly with OT. Allowing a controlled variation of the matched densities thus requires an appropriate relaxation of the mass conservation constraint. Mass conservation relaxation is related to the relaxation of the bijectivity constraint in graph matching, for which a convex formulation is proposed in~\cite{Zaslavskiy-graph-match}.

%%%%%%%%%%%%%%%%%%%%%%%%%%%%%%%%%%%%%%%%%%%%%%%%%%%%%%%%%%%%%
\subsection{Contributions}

In this paper, we generalize the discrete formulation of OT to tackle the two major flaws that we just mentioned: i) the lack of regularity of the transport and ii) the need for a relaxed matching between densities. Our main contribution is the integration of these two properties in a unified variational formulation to compute a regular transport map between two empirical densities. The corresponding optimization problem is convex and can be solved using standard convex optimization procedures. We propose two optimization algorithms adapted to the different class of regularizations. We apply this framework to the color transfer problem and obtain results that are comparable to the state of the art. 
Our second contribution takes advantage of the proposed regularized OT energy to compute the barycenter of several empirical densities. 
We develop a block-coordinate descent method that converges to a stationary point of the non-convex barycenter energy. We show an application
 to color normalization between a set of photographs. Numerical results show the relevance of these approaches to imaging problems. The matlab code to reproduce the figures of this article is available online\footnote{\url{https://www.ceremade.dauphine.fr/~sira/regularizeddiscreteOT}}.

Part of this work was presented at the conference SSVM 2013~\cite{2013-ssvm-regul-ot}.

%%%%%%%%%%%%%%%%%%%%%%%%%%%%%%%%%%%%%%%%%%%%%%%%
\section{Discrete Optimal Transport}

Monge's original formulation of the OT problem corresponds to minimizing the cost for transporting a distribution $\mu_X$ onto another distribution $\mu_Y$ using a map $T$
\eql{ 
	\min_{T} \int_X c(x,T(x)) \d \mu_X(x),   
	\qwhereq
	 T\#\mu_X=\mu_Y.
}
Here, $\mu_X,\mu_Y$ are measures in $\RR^d$, $T: \RR^d \rightarrow \RR^d$ is a $\mu_X$-measurable function, $c : \RR^d \times \RR^d \rightarrow \RR^+$ is a $\mu_X \otimes \mu_Y$-measurable function, and $\#$ is the push forward operator. 

We focus here on the case where the measures are discrete, have the same number of points, and all points have the same mass, thus 
\eq{
	\mu_X = \frac{1}{N} \sum_{i=1}^{N} \delta_{X_i} 
	\qandq
	\mu_Y = \frac{1}{N} \sum_{j=1}^{N} \delta_{Y_j},
}
where $\delta_x$ is the Dirac measure at location $x \in \RR^d$, and where the position of the supporting points are $X = (X_i)_{i=1}^N$, and $Y = (Y_j)_{j=1}^N$, where  $X_i,Y_j \in \RR^d$. In this context, the transport between $X$ and $Y$ is a one-to-one assignment,  i.e. 
% \todo{It was written $T(X_i) = X_{\sigma(i)}$ ! } 
$T(X_i) = Y_{\sigma(i)}$ where $\sigma$ is a permutation of $\{1,\ldots,N\}$, which can be encoded using a permutation matrix $\Sigma$ such that
\eq{ \Sigma_{i,j} = 
 	\choice{ 
		1 \qifq j=\sigma(i), \\ 
		0 \quad \text{otherwise}.  
	}
} 
A more compact way to denote the transport is $T(X_i)=\left(\Sig Y\right)_i, \forall i=\{1,\ldots,N\}$. Introducing the cost matrix
\eq{
	C_{X,Y} \in \RR^{N \times N}
	\qwhereq
	\foralls (i,j) \in \{1,\ldots,N\}^2, \quad (C_{X,Y})_{i,j} = c(X_i,Y_j),
}
this permutation matrix $\Sig$ is thus the solution to the following optimization problem
\eql{
	\umin{\Sig \in \Perm}  
		\dotp{C_{X,Y}}{\Sigma} = \sum_{i,j=1}^N c(X_i,Y_j)  \Sig_{i,j},
\label{eqW}} 
where $\Perm$ is the set of permutation matrices  
\eq{
	\Perm = \enscond{\Sig \in \RR^{N \times N}}{ \Sig^* \U = \U, \Sig \U = \U, \Sig_{i,j} \in \{0,1\}},
} 
see~\cite{Villani03} for more details. We have denoted $\U=(1,\ldots,1)^* \in  \RR^N$, and $A^*$ as the adjoint of the matrix $A$, that for real matrices amounts to the transpose operation.

In the special case where 
\eq{
	\left(C_{X,Y}\right)_{i,j} = c(X_i,Y_j)=\norm{X_i-Y_j}^{\al}
} where $\norm{\cdot}$ is the Euclidean norm in $\RR^d$ and $\al \ge 1$,  the value of the optimization problem~\eqref{eqW} is called the $L^\alpha$-Wasserstein distance (to the power $\alpha$), and is denoted $W_\alpha(\mu_X,\mu_Y)^\alpha$. It can be shown that $W_\alpha$ defines a distance on the set of distributions that have moments of order $\al$.

%  Note that computing the Wasserstein distance requires the computation of the transport map $T(X)=\Sig X$.

%%%
\paragraph{ Kantorovich OT formulation}

The set of permutation matrices $\Perm$ is not convex.  Its convex hull is the set of bi-stochastic matrices
\eq{
	\Matr_1 = \enscond{\Sig \in \RR^{N \times N}}{ \Sig \U = \U, \Sig^* \U = \U, \Sig_{i,j} \in [0,1]}.
}
One can show that the relaxation
\eql{\label{eqMK}
	\min_{\Sig \in \Matr_1}  \dotp{C_{X,Y}}{\Sigma}
} 
of~\eqref{eqW} is tight, meaning that there exists a solution of~\eqref{eqMK} which is a binary matrix, hence being also a solution of the original non-convex problem~\eqref{eqW}, see~\cite{Villani03}.

%%%%%%%%%%%%%%%%%%%%%%%%%%%%%%%%%%%%%%%%%%%%%%%%
\section{Relaxed and Regularized Transport}

In the previous section, we introduced the Monge-Kantorovich formulation for the computation of the OT between two distributions, as the minimization of the energy~\eqref{eqMK}. In this section, we modify this energy in order to obtain a regular OT mapping, which is important for applications such as color transfer. 

\subsection{Relaxed Transport}

Section~\ref{sec-appli-color} tackles the color transfer problem, where, as in many applications in imaging, strict mass conservation should be avoided.  As a consequence, it is not desirable to impose a one-to-one mapping between the points in $X$ and $Y$. 

The relaxation we propose allows each point of $X$ to be transported to multiple points of $Y$ and vice versa. This corresponds to imposing the constraints 
\eq{
	k_X \U \leq \Sig \U \leq K_X \U
	\qandq 
	k_Y \U \leq \Sigma^* \U \leq K_Y \U
} 
on the matrix $\Sig$, where 
$\kappa=(k_X,K_X,k_Y,K_Y) \in (\RR^+)^4$
are the parameters of the method. 
% Note that these restrictions on $\Sig$ do not rule out the option of decreasing the overall mass. 
To impose the total amount of mass $M$ transported between the densities, we further impose the constraint $\U^* \Sig\U=M$, where $M>0$ is a parameter.
The initial OT problem~\eqref{eqMK} now becomes:
\eql{\label{eq-relax-map}
	\umin{\Sigma \in \Matr_\kappa} \dotp{C_{X,Y}}{\Sig}
}  
\eq{
	\qwhereq
	\Matr_\kappa = \enscond{\Sig \in [0,1]^{N \times N}}{ 
		\begin{array}{ll}
			k_X \U \leq \Sigma \U \leq K_X \U, \\
			k_Y \U \leq \Sigma^* \U \leq K_Y \U, 
		\end{array} \: 
		\U^* \Sig\U=M
		 }.
}
To ensure that $\Matr_\kappa$ is non empty, we impose that
\eq{ 
 	\max(k_X,k_Y) \leq \frac{M}{N} \leq \min(K_X,K_Y)
}
For the application to the color manipulations considered in this paper, we set once and for all this parameter to $M=N$. 

%\todo{Either provide a mathematical statement with a proof, or remove this sentence. }
Note that if $\min(K_X,K_Y) \geq N$, there is no restriction on the number of connections of each element of $X$ or $Y$, then the optimal solution increases (always under the constraints $\U^*\Sig\U=M$) the weight given to the connection between the closest points in $X$ to the closest points in $Y$, that is to say, the minima $(C_{X,Y})_{i,j}$ are assigned the maximum possible weight, see Fig.~\ref{im:relaxationk} for an example.  

%the solution of \eqref{eq-relax-map} is the nearest neighbor assignment.
  
Problem~\eqref{eq-relax-map} is  a convex linear program, which can be solved using standard linear programming algorithms. 

\paragraph{Relaxed OT map} Optimal matrices $\Sig$ minimizing~\eqref{eq-relax-map} are in general non binary and furthermore their non zero entries do not define one-to-one maps between the points of $X$ and $Y$. It is however possible to define a map $T$ from $X$ to $Y$ by mapping each point $X_i$ to a weighted barycenter of its neighbors in $Y$ as defined by $\Sig$. This corresponds to defining
\eq{\label{eqT}
T(X_i) = \frac{\sum_{j=1}^N \Sig_{i,j} Y_j }{\sum_{j=1}^N \Sig_{i,j}}\\
} which in vectorial form can be expressed as $T(X_i)= Z_i $, where $Z=(\diag(\Sig \U))^{-1} \Sig Y$, and where the operator $\diag(v)$ creates a diagonal matrix in $\RR^{N \times N}$ with the vector $v \in \RR^N$ on the diagonal.
To insure that the map is well defined, we impose that $k_X > 0$.  Note that it is possible to define a map from $Y$ to $X$ by replacing $\Sig$ by $\Sig^*$ in the previous formula and exchanging the roles of $X$ and $Y$.
 
The following proposition shows that an optimal $\Sig$ is binary when the parameters $\kappa$ are integers. Such a binary $\Sig$ can be interpreted as a set of pairwise assignments between the points in $X$ and $Y$. Note that this is not true in general when the parameters $\kappa$ are not integers.

\begin{prop}\label{prop}
	For $(k_X,K_X,k_Y,K_Y,M) \in (\NN^*)^5$, there exists a solution  $\tilde{\Sig}$ of~\eqref{eq-relax-map} which is binary, i.e. $\tilde \Sigma \in \{0,1\}^{N \times N}$.
	\if 0
		which is solution of
	\eql{\label{eq-relaxed-transport}
		\umin{\Sig \in \Bb_{\kappa}} \dotp{C_{X,Y}}{\Sig},
	}
	\eq{
		\qwhereq
		\Bb_{\kappa}= \enscond{\Sig \in \{0,1\}^{N \times N}}{
		\begin{array}{ll}
			k_X \U \leq \Sigma \U \leq K_X \U, \\
			k_Y \U \leq \Sigma^* \U \leq K_Y \U, 
		\end{array} \: 
		\U^* \Sig\U=M
		}
	}
	that is to say, $\tilde{\Sig}$ is a binary matrix.
	\fi
\end{prop}
\begin{proof}
	One can write 
	$\Matr_\kappa = \enscond{ \Sig \in \RR^{N \times N} }{ \Aa(\Sig) \leq b_\kappa }$
	where $\Aa$ is the linear mapping
	$\Aa(\Sig)=(-\Sig,\Sig \U,-\Sig \U,\Sig^* \U,-\Sig^* \U,\U\Sig\U,-\U\Sig\U)$, where $\Sig^* \U, \Sig\U \in \RR^{N}$ and $\U\Sig\U \in \RR$, 	and $b_\kappa = (0_{N,N}, K_X\U, -k_X\U, K_Y\U, -k_Y\U, M, -M)$. A standard result shows that $\Aa$ is a totally unimodular matrix~\cite{schrijver-book}. For any $(k_X,K_X,k_Y,K_Y,M) \in (\NN^*)^5$, the vector $b_\kappa$ has integer coefficients, and thus the polytope $\Matr_\kappa$ has integer vertices. Since there is always a solution of the linear program~\eqref{eq-relax-map} which is a vertex of $\Matr_\kappa$, it has coefficients in $\{0,1\}$.
\end{proof}

\begin{figure}
\centering
\begin{tabular}{@{}|@{}c@{}|@{}c@{}|@{}c@{}|}
\hline
\includegraphics[width=.33\linewidth]{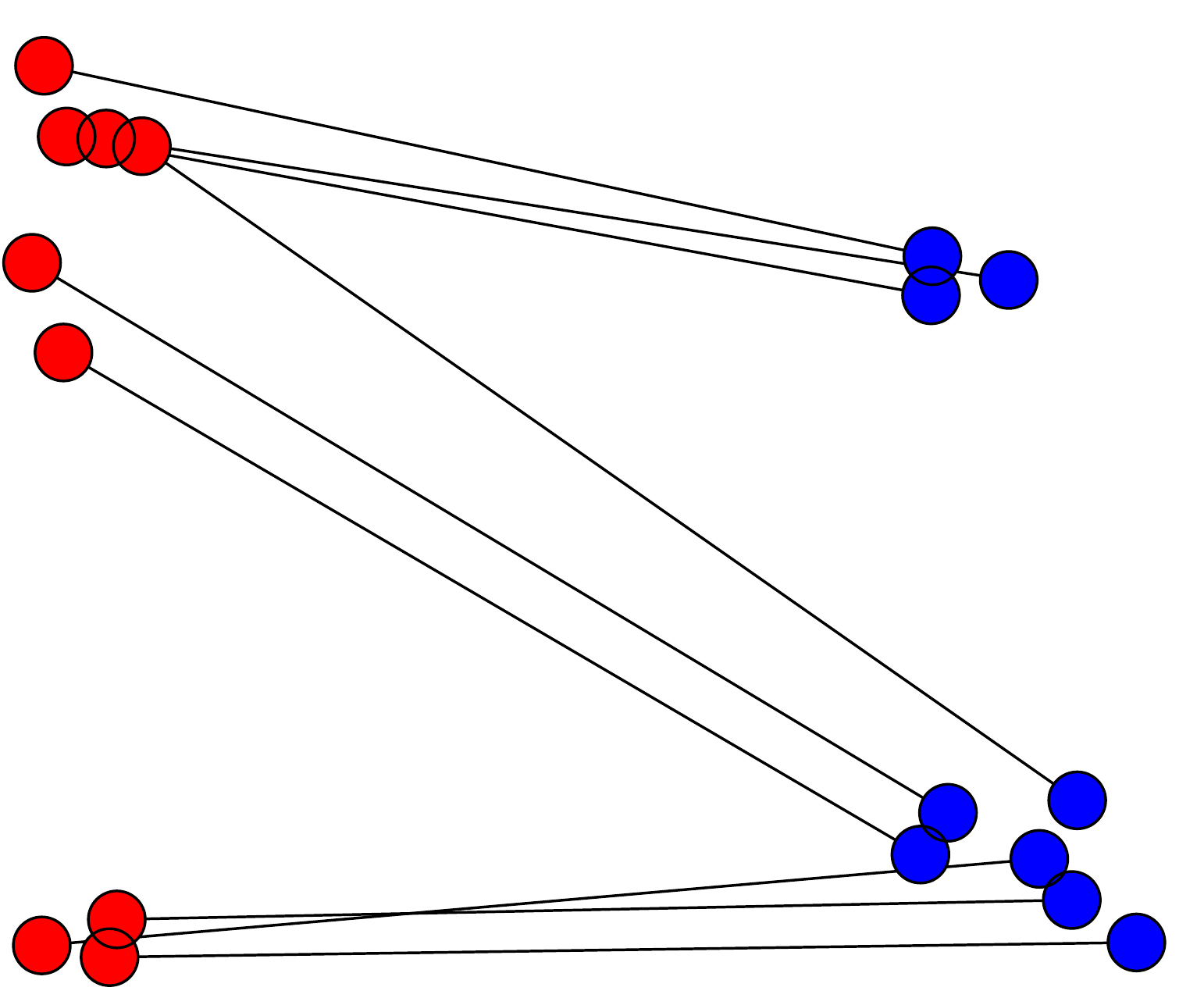} & 
\includegraphics[width=.33\linewidth]{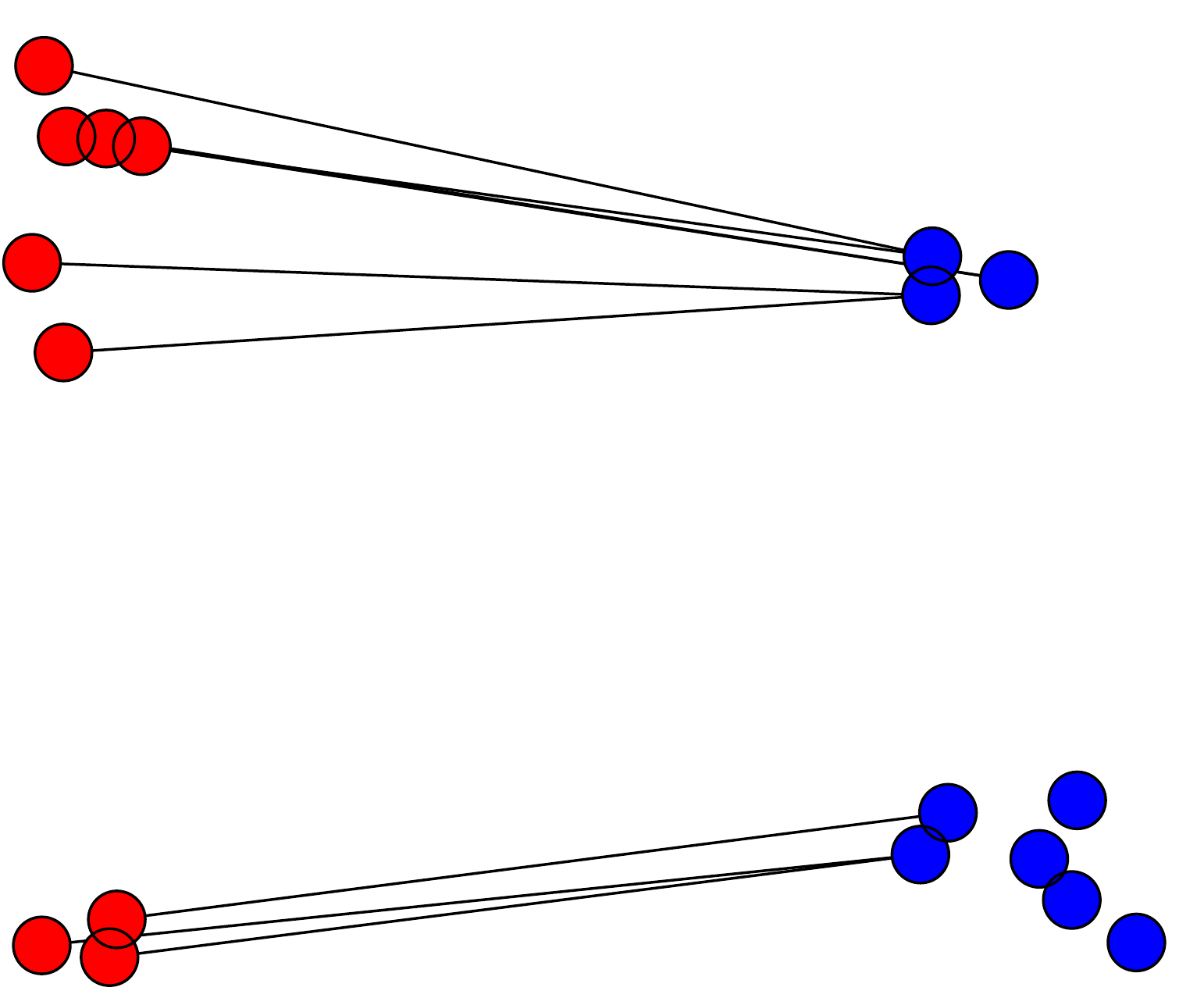} &
\includegraphics[width=.33\linewidth]{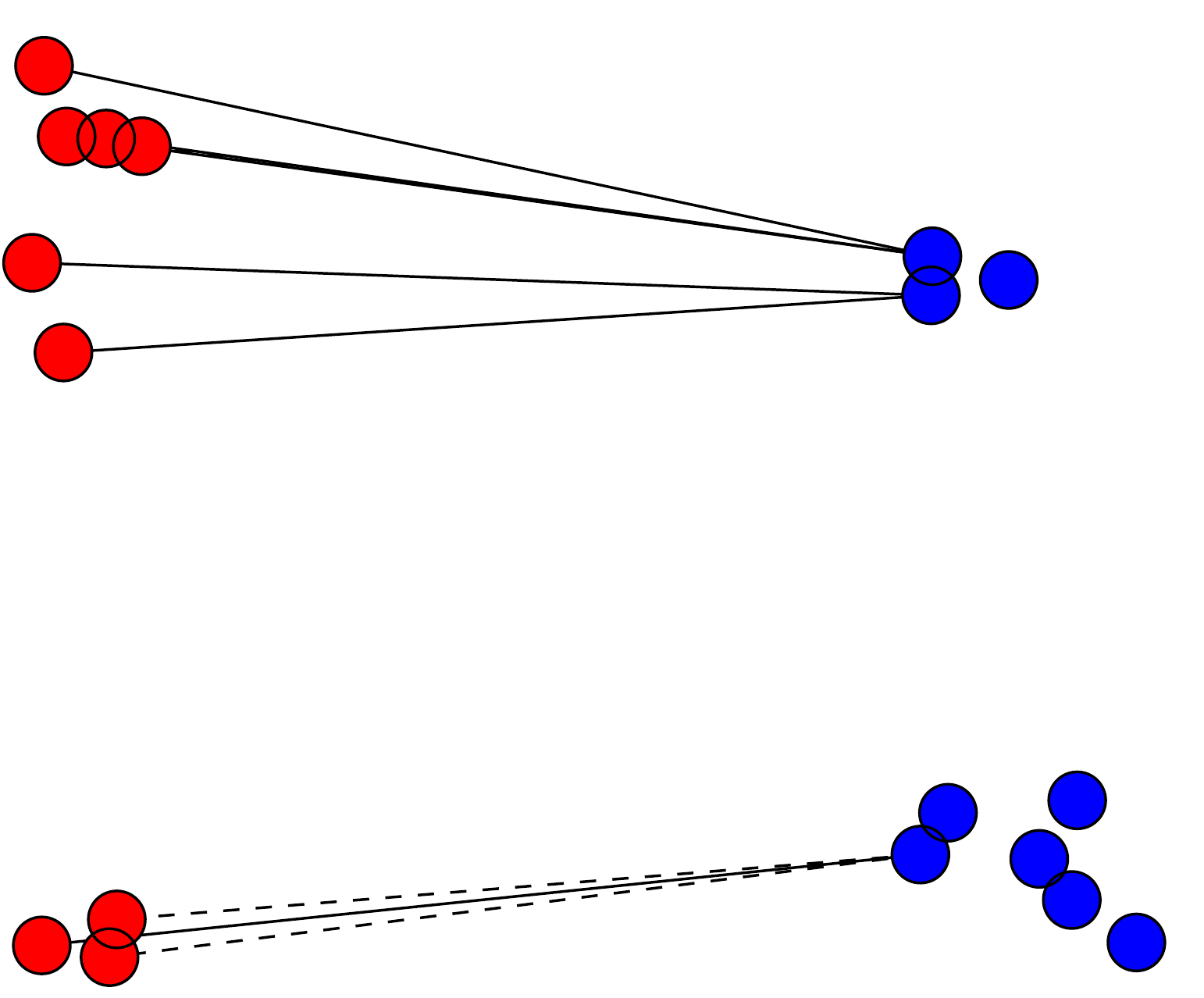} \\ 
$\kappa=(1,1,1,1)$ & $\kappa=(1,1,0,2)$ & $\kappa=(1,1,0.1,10)$ \\\hline 
\includegraphics[width=.33\linewidth]{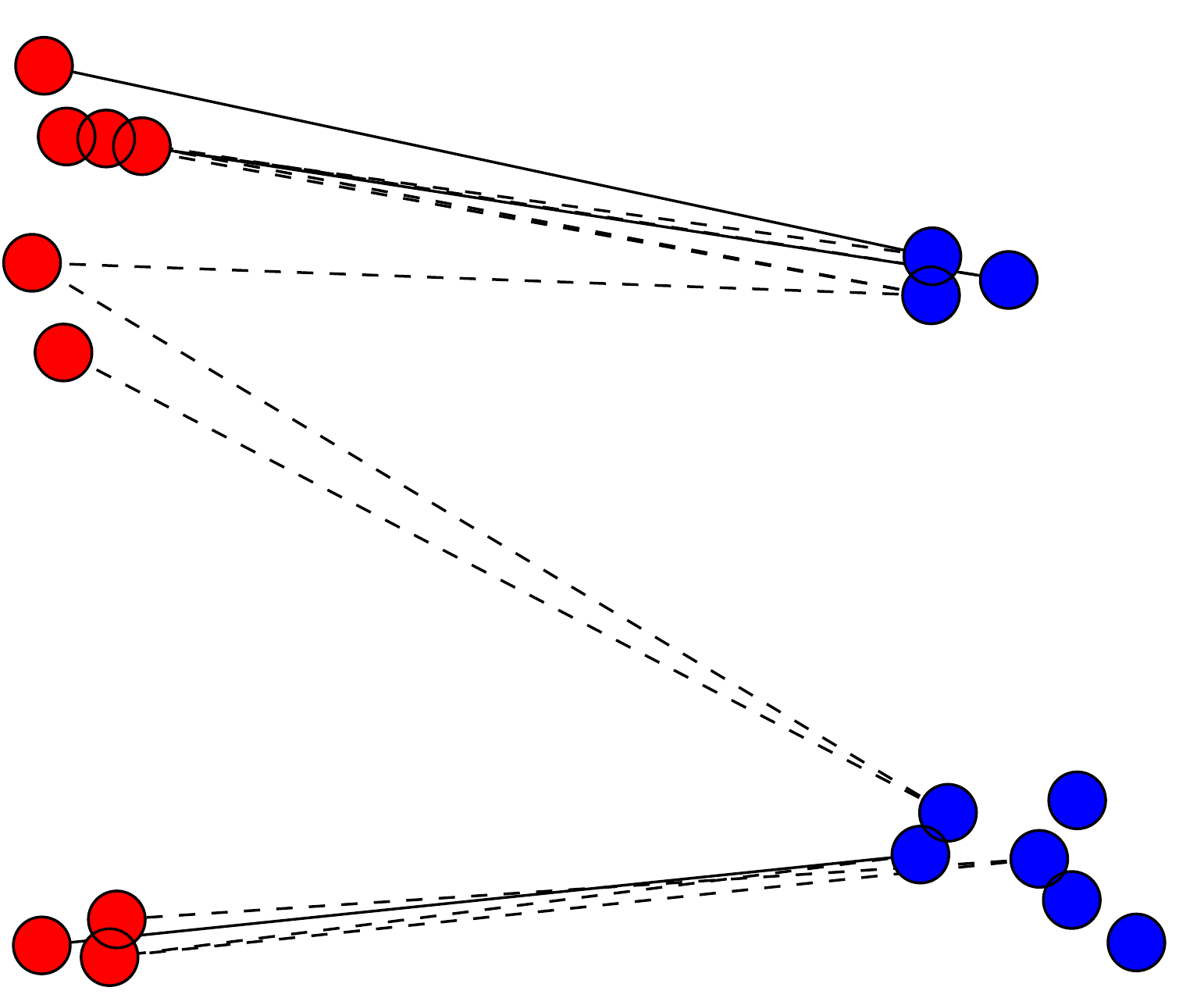} & 
\includegraphics[width=.33\linewidth]{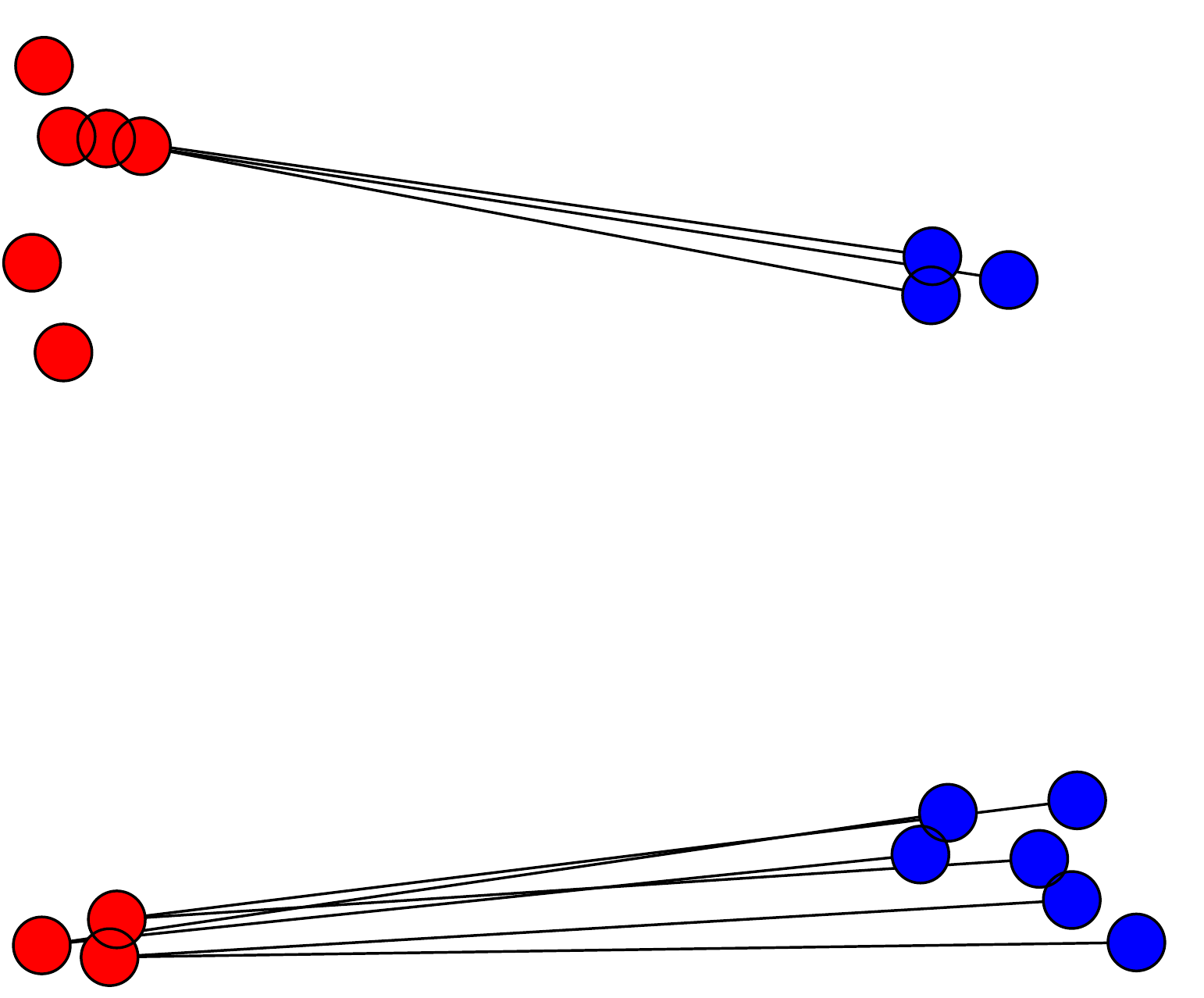} &
\includegraphics[width=.33\linewidth]{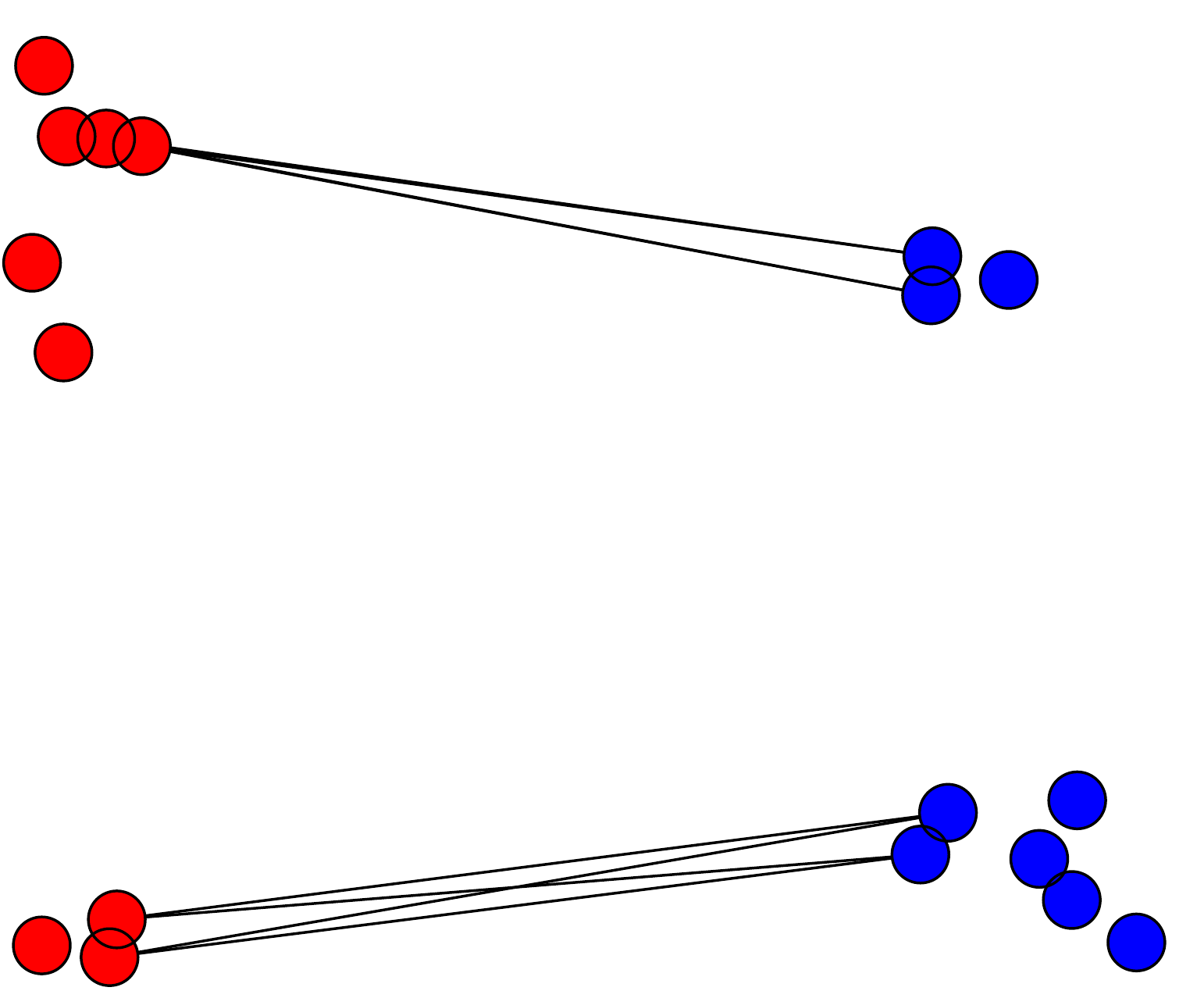}  \\
  $\kappa=(1,1,0.1,1.5)$ & $\kappa=(0,2,1,1)$  & $\kappa=(0.1,10,0.1,10)$ \\\hline 
\end{tabular}
\caption{\label{im:relaxationk}Relaxed transport computed between $X$ (blue dots) and $Y$ (red dots) for different values of $\kappa$.
Note that $\kappa=(1,1,1,1)$ corresponds to classical OT.
A dashed line between $X_i$ and $Y_j$ indicates that $\Sigma_{i,j}$ is not an integer. }
\end{figure}

%%%
\paragraph{Numerical Illustrations}

In Fig.~\ref{im:relaxationk}, we show a simple example to illustrate the properties of the method proposed so far. Given a set of points $X$ (in blue) and $Y$ (in red), we compute the optimal $\Sig$ solving~\eqref{eq-relax-map} for different values of $\kappa$. For each values of $\kappa$, we draw a line between $X_i$ and $Y_j$ if the value of the associated optimal $\Sigma_{i,j} > 0.1$, solid if $\Sigma_{i,j}=1$, and dashed otherwise.

As we prove in the Proposition~\ref{prop}, for non integer values of $K_X,K_Y$, the mappings $\Sigma_{i,j}$ are in $[0,1]$ while for integer values, $\Sigma_{i,j} \in \{0,1\}$. Note that as we increase the values of $K_X,K_Y$ (Fig.~\ref{im:relaxationk}, right), the points in $X$ tend to be mapped to the closer points in $Y$.

%\begin{figure}
%\begin{tabular}{cccc}
%\includegraphics[height=2.5cm]{./images/k1_l0.png} &
%\includegraphics[height=2.5cm]{./images/k2_l0.png} &
%\includegraphics[height=2.5cm]{./images/k25_l0.png}&
%\includegraphics[height=2.5cm]{./images/k10_l0.png} \\
% (a) & (b) & (c) & (d)
% \end{tabular}
%\caption{Relaxed transport computed between $X$ (blue dots) and $Y$ (red dots) with the parameters \textbf{(a)} $k=1$ (classical OT) \textbf{(b)} $k=2$ \textbf{(c)} $k=2.5$ \textbf{(d)} $k=10$. The color of the line between $X_i$ and $Y_j$ indicates the value of the mapping $\Sigma_{i,j}$. }
%\label{im:relaxationk}
%\end{figure}

%%%%%%%%%%%%%%%%%%%%%%%%%%%%%%%%%%%%%%%%%%%%%%%%
\subsection{Discrete Regularized Transport} \label{sec:regsymme}

So far, we have introduced a transport problem where the mass conservation constraint is relaxed. The second step is to define its regularization. A classic way of imposing regularity on a mapping $V: \RR^d \rightarrow \RR^d$ is by measuring the amplitude of its derivatives. Two examples for continuous functions are the quadratic Tikhonov regularizations such as the Sobolev semi-norm $\|\nabla V\|^2$, and the anisotropic total variation semi-norm $\|\nabla V\|_1$ regularization. Nevertheless, the differential operator $\nabla$  cannot be applied directly to our point clouds due to the lack of neighborhood definition. To extend the definition of the gradient operator, we need to impose graph structures on the point clouds.

 In our setting, we want to regularize the discrete map $T$ defined in~\eqref{eqT}, which is only defined at the location of the points as $X_i \mapsto \tilde{V}_i=X_i - \diag(\Sig \U)^{-1} (\Sig Y)_i $.  To avoid the normalization $\diag(\Sig \U)$ (which typically leads to non-convex optimization problems), and further regularize the variation of the weights $\Sig \U \in \RR^N$, we impose a regularity on the map $X_i \mapsto V_i = \diag(\Sig \U)X_i-(\Sigma Y)_i$.

\paragraph{Gradient on Graphs}

A natural way to define a gradient on a point cloud $X$ is by using  the gradient on a weighted graph $\Gg_X = (X,E_X,W_X)$ where $E_X \subset \{1,\ldots,N\}^2$ is the set of edges and $W_X$ is the set of weights, $W_X = (w_{i,j})_{i,j=1}^N:  \{1,\ldots,N\}^2 \mapsto \RR^+$, satisfying $w_{i,j}=0$ if $(i,j) \notin E_X$. The edges of this graph are defined depending on the application. A typical example is the $n$-nearest neighbor graph, where every vertex $X_i$ is connected to $X_j$ if $X_j$ is one of the $n$-closest points to $X_i$ in $X$, creating the edge $(i,j) \in E_X$, with a weight $w_{i,j}$. Because the edges are directed, the adjacency matrix is not symmetric. 

% We assume there are no loops and no multiple edges between the same vertices. 

The gradient operator on $\Gg_X$ is defined as $G_X : \RR^{N \times d} \rightarrow \RR^{P \times d}$, where $P=\|E_X\|$ is the number of edges and where, for each $V=(V_i)_{i=1}^N \in \RR^d$, 
\eq{
	G_X V = \pa{ w_{i,j}(V_i-V_j) }_{(i,j) \in E_X} \in \RR^{P \times d}.
}
A classic choice for the weights to ensure consistency with the directional derivative is $w_{i,j} = \norm{X_i-X_j}^{-1}$, see for instance~\cite{guilboa07}.

%%%%%%%%%%%%%%%
%\paragraph{Convex Formulation}
\paragraph{Regularity Term}

The regularity of a transport map $V \in \RR^{N \times d}$ is then measured according to some norm of $G_X V$, that we choose here for simplicity to be the following
\eq{
	\regul{G_X  V} = \sum_{(i,j) \in \Gg_x} \left(\norm{ w_{i,j} (V_i-V_j) }_q \right)^p,
}
where $\|.\|_q$ is the $\ell^q$ norm in $R^d$. 

The case $(p,q)=(1,1)$ is the graph anisotropic total variation, $(p,q)=(2,2)$ is the graph Sobolev semi-norm, and $(p,q)=(1,2)$ is the graph isotropic total variation, see for instance~\cite{elmoataz-graph} for applications of these functionals to imaging problem such as image segmentation and regularization. 

\subsection{Symmetric Regular OT Formulation}

Given two point clouds X and Y, our goal is to compute a relaxed OT mapping between them which is regular with respect to both point clouds. To simplify notation, we conveniently re-write the displacement fields we aim to regularize as:
\eq{
	\De_{X,Y}(\Sig) = \diag( \Sigma \U ) X - \Sigma Y
	\qandq
	\De_{Y,X}(\Sig^*) = \diag( \Sigma^* \U ) Y -  \Sigma^* X.
}

Our goal is to obtain a partial matching that is regular according to $X$ and $Y$, so we create two graphs $\Gg_X$ and $\Gg_Y$ as described in Section~\ref{sec:regsymme} and we denote the corresponding gradient operators $G_X \in \RR^{P_X \times N}$ and $G_Y \in \RR^{P_Y \times N}$ where $P_X$ and $P_Y$ are the number of edges in the respective graphs. The symmetric regularized discrete OT energy is defined as: 
\eql{\label{eq-symm-reg-energy}
\umin{\Sig \in \Matr_\kappa} \dotp{\Sig}{\Cost{X}{Y}} + \la_X \regul{G_X \De_{X,Y}(\Sig)} + \la_Y \regul{G_Y \De_{Y,X}(\Sig^*)},
}
where $(\la_X,\la_Y) \in (\RR^+)^2$ controls the desired amount of regularity.
The case $\kappa=(1,1,1,1)$ and $(\la_X,\la_Y)=(0,0)$ corresponds to the usual OT defined in~\eqref{eqMK}, and $(\la_X,\la_Y)=(0,0)$ corresponds to the un-regularized formulation~\eqref{eq-relax-map}.% Points $X_i$ (reps. $Y_j$) such that $(\Sigma \U)_i = \sum_j \Sig_{i,j} = 0$ (reap. $(\Sig^* \U)_j=0$) are singular, and are discarded from the matching. They typically correspond to outliers in the input data clouds.

%%%%%%%%%%%%%%%
%\subsection{Minimization Algorithm}\label{sec:minalgo}
%\paragraph{Minimization algorithm}
%\input{sections/sec-prox-algo}

%%%%%%%%%%%%%%%

\subsection{Algorithms} \label{secalgosymm}

Specific values of the parameters $p$ and $q$ lead to different regularization terms, which in turn necessitate different optimization methods. In the following, for the sake of concreteness, we concentrate on the specific cases $(p,q)=(2,2)$ and $(p,q)=(1,1)$.
%REMOVED 
% In both cases, the optimization problem is solved with linear programming methods which in our numerical tests were solved with standard interior point methods.

%%%
\paragraph{Sobolev regularization}

Defining $q=p=2$ fixes the regularization term as a graph-based Sobolev regularization. In this specific case, the minimization~\eqref{eq-symm-reg-energy} becomes a quadratic programming problem 
\begin{equation} \label{eq-symm-S}
\underset{\Sig \in \Matr_\kappa}{\text{min}} f(\Sig) =  \dotp{ \Cost{X}{Y}}{\Sig } + \frac{\la_X}{2} \|\Gamma_{X,Y}(\Sig)\|^2 + \frac{\la_Y}{2} \|\Gamma_{Y,X}(\Sig) \|^2\\,
\end{equation} where $\Gamma_{X,Y}(\Sig) = G_X \De_{X,Y}(\Sig)$ and $\Gamma_{Y,X}(\Sig) = G_Y \De_{Y,X}(\Sig^*)$. The Frank-Wolfe algorithm is well tailored to solve such problems, as noticed for instance in~\cite{Zaslavskiy09}, given that $f$ is convex and differentiable, and $\Matr_\kappa$ is a convex set. 
The Frank-Wolfe method (also known as conditional gradient) iterates the following steps until convergence
\begin{equation}
\begin{aligned}\label{eq-frankwolfe-update}
\tilde\Sig^{(\ell+1)} &\in \uargmin{\tilde\Sig \in \Matr_\kappa} \dotp{\nabla f(\Sig^{(\ell)})}{ \tilde\Sig } \\
\Sig^{(\ell+1)} &= \Sig^{(\ell+1)} + \tau_\ell ( \tilde\Sig^{(\ell+1)}-\Sig^{(\ell+1)} ),\\
\end{aligned}
\end{equation}
where $\tau_\ell$ is obtained by line-search. The first equation of~\eqref{eq-frankwolfe-update} is a linear program which is efficiently solved using interior point methods~\cite{Nesterov-Nemirovsky-Book}.  In our case, one has
\eq{
	\nabla f(\Sig) = \Cost{X}{Y} + \la_X  \De^*_{X,Y}(G_X^*  \Gamma_{X,Y}(\Sig)) + \la_Y \De^*_{Y,X}(G_Y^* \Gamma_{Y,X}(\Sig)), 
} where 
\eq{\De^*_{X,Y}(U)= \diag^*(U X^*)\U^*-U Y^* \mbox{ and } \De^*_{Y,X}(U)= (\diag^*(U Y^*)\U^*)^*-X U^*,}
 where $\diag^*: \RR^{N \times N} \mapsto \RR^{N}$ is the adjoint of the $\diag$ operator, and given $A \in \RR^{N \times N}$, $\diag^*(A)$ is a vector composed by the elements on the diagonal of $A$. 

The line search optimal step can be explicitly computed as
\eq{\tau_{\ell} = \frac{ - \dotp{ E^{(\ell)} }{ \Cost{X}{Y} } - \dotp{  \Gamma_{X,Y}(E^{(\ell)}) }{  \Gamma_{X,Y}(\Sig^{(\ell)}) } - \dotp{  \Gamma_{Y,X}( E^{(\ell)} )  }{  \Gamma_{Y,X} (\Sig^{(\ell)}) }	}{ \la_X \norm{\Gamma_{X,Y}(E^{(\ell)})}^2 +   \la_Y \norm{\Gamma_{Y,X}(E^{(\ell)})}^2}} where $E^{(\ell)} = \Sig^{(\ell+1)} - \tilde\Sig^{(\ell+1)}$.

\paragraph{Anisotropic TV Regularization}

We define an anisotropic total variation (TV) norm by setting the parameters $q=p=1$.
Problem~\eqref{eq-symm-reg-energy} can be re-written as a linear program by introducing the auxiliary variables $U_X \in \RR^{P_X \times d}$ and $U_Y \in \RR^{P_Y \times d}$:

\eql{
\begin{aligned}
& \underset{\Sig, U_X,U_Y}{\text{min}}  & & \dotp{ \Cost{X}{Y}}{\Sig}  + \lambda_X \dotp{U_X}{\U}+ \lambda_Y \dotp{U_Y}{ \U} \\
& \mbox{subject to} & &
\left\{ \begin{aligned}
-U_X  & \leq G_X ( \Sig Y  - \diag(\Sig \U) X)\leq U_X, \\
-U_Y & \leq G_Y( \Sig^* X - \diag(\Sig^* \U) Y) \leq U_Y, \\
\Sig & \in \Matr_\kappa. \\
\end{aligned}
\right.
\end{aligned}\label{eq-symm-TV}
}

\paragraph{Numerical Illustrations}
In Fig.~\ref{exlk}, we can observe, on a synthetic example, the influence of the parameters $\kappa$ and $(\lambda_X, \lambda_Y)$, from equation~\eqref{eq-symm-reg-energy}. %We plot $X$ in blue and $Y$ in red, compute the transport map, obtain $Z=\diag(\Sigma \U)^{-1} \Sigma Y$, where $\Sigma$ is a solution of~\eqref{eq-symm-reg-energy}, and we plot $Z$ in green. In order to visualize the transport map, we plot a line segment between $X_i$ and $Y_j$ if $\Sig_{i,j}>\epsilon$, where here is $\epsilon=0.1$.   We set  $\kappa=(0.1,8,0.1,8)$ and compute the OT for several values of $\la_X=\la_Y$. 

\begin{figure*}\label{figlambda}
\centering
\begin{tabular}{@{}|@{}c@{}|@{}c@{}|@{}}
\hline
  \includegraphics[width=.4\linewidth]{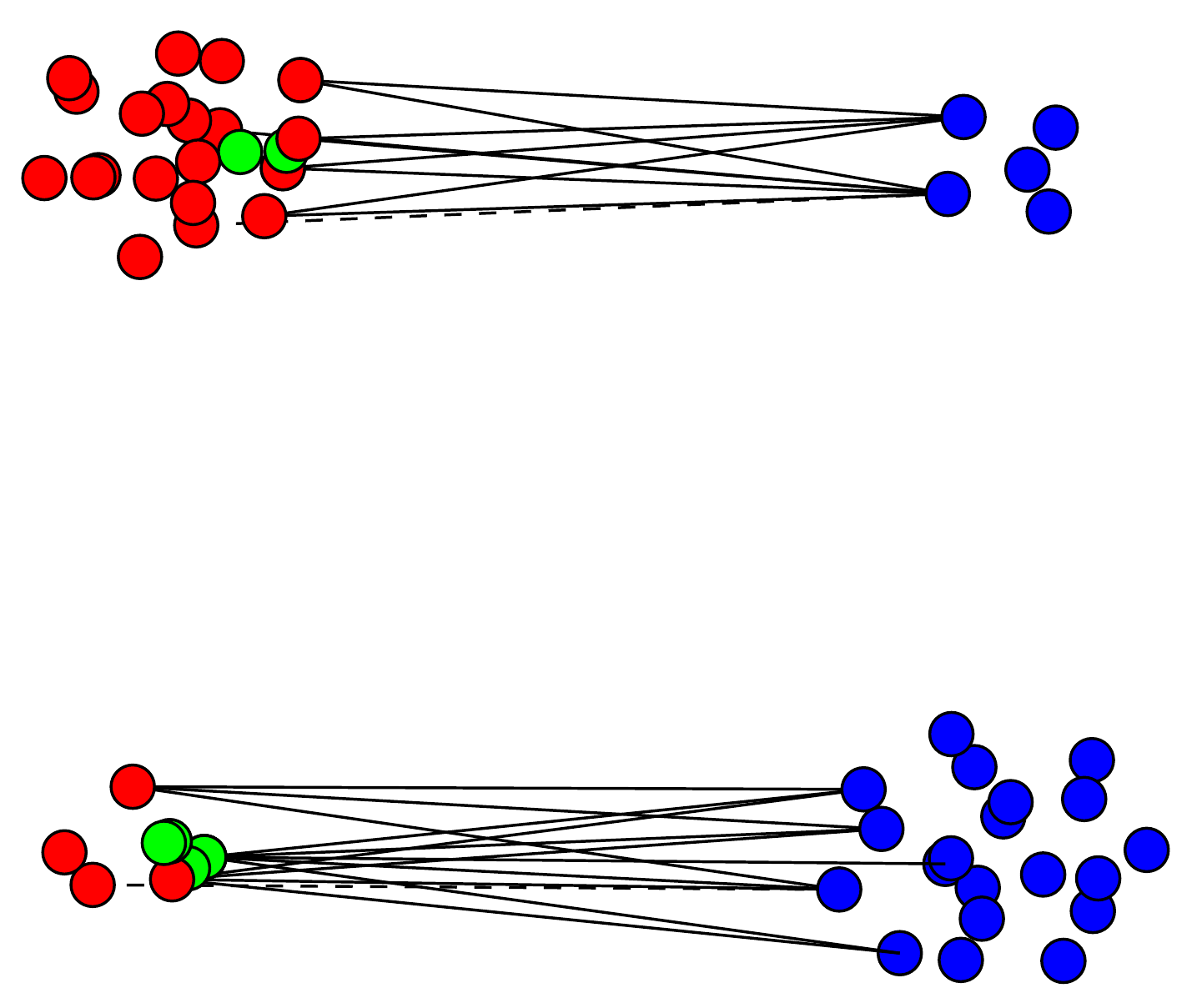} &  
  \includegraphics[width=.4\linewidth]{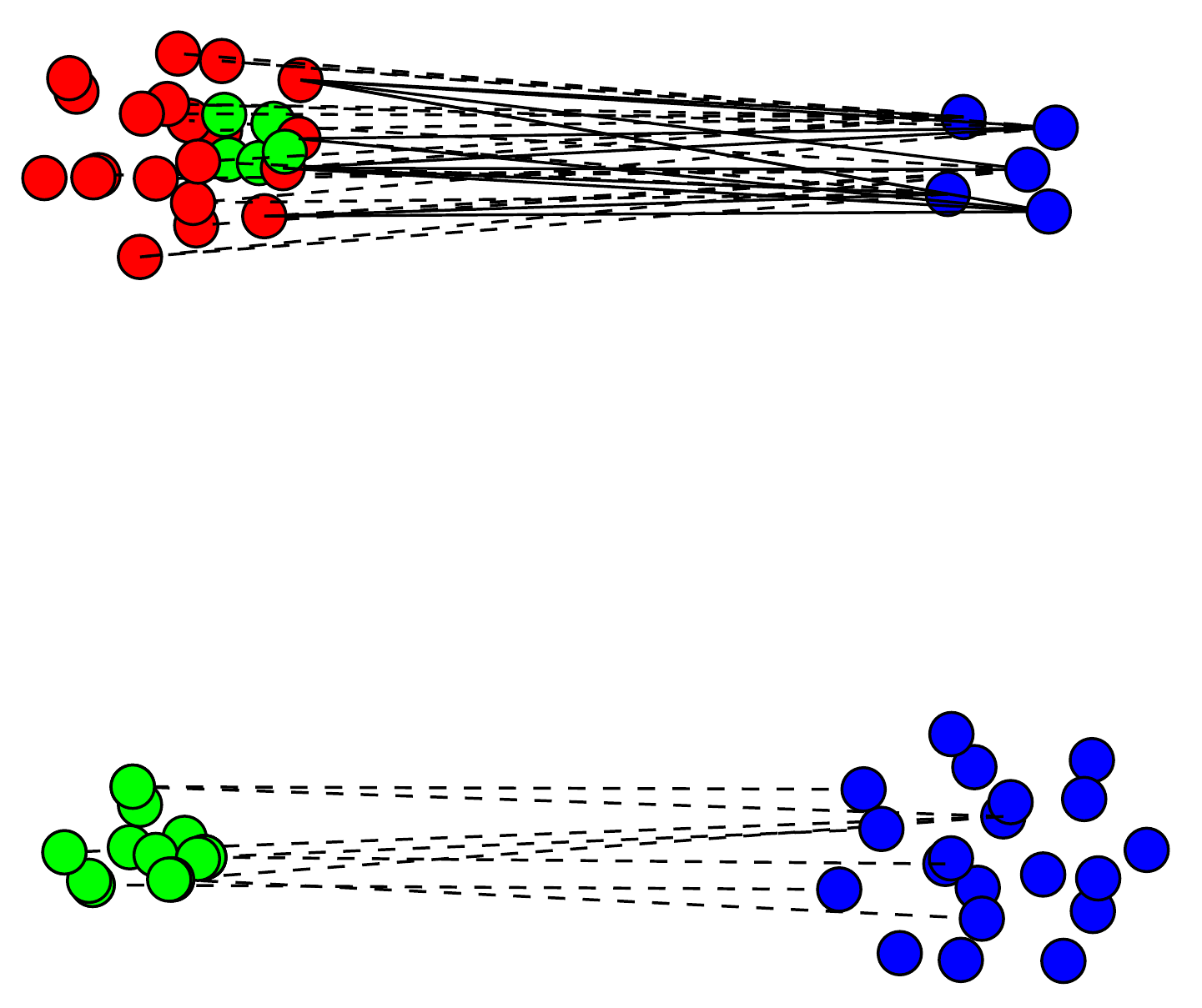} \\
	 {$\la_X=\la_Y=0$} & {$\la_X =\la_Y=0.001$}\\\hline
  \includegraphics[width=.4\linewidth]{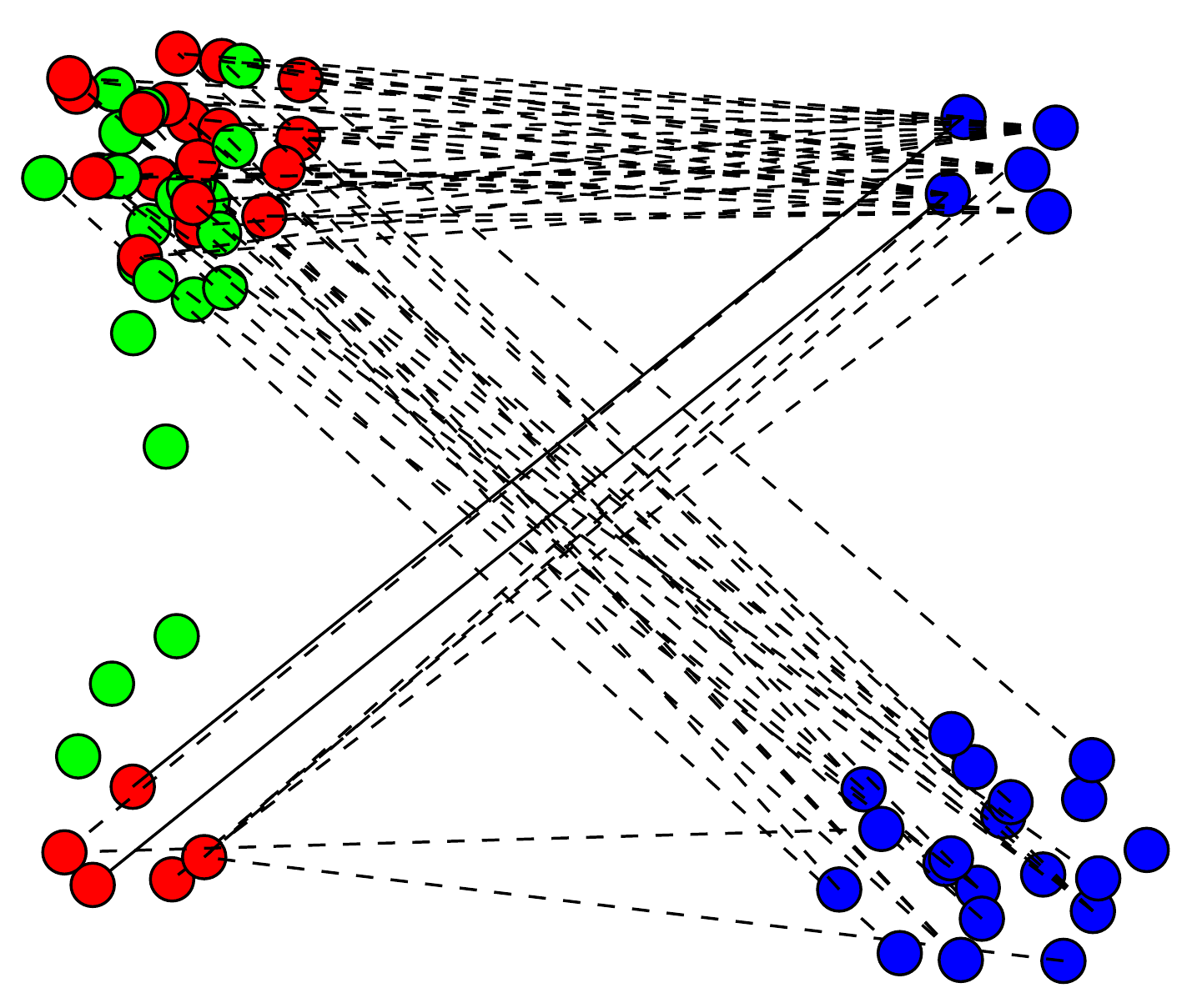} &
  \includegraphics[width=.4\linewidth]{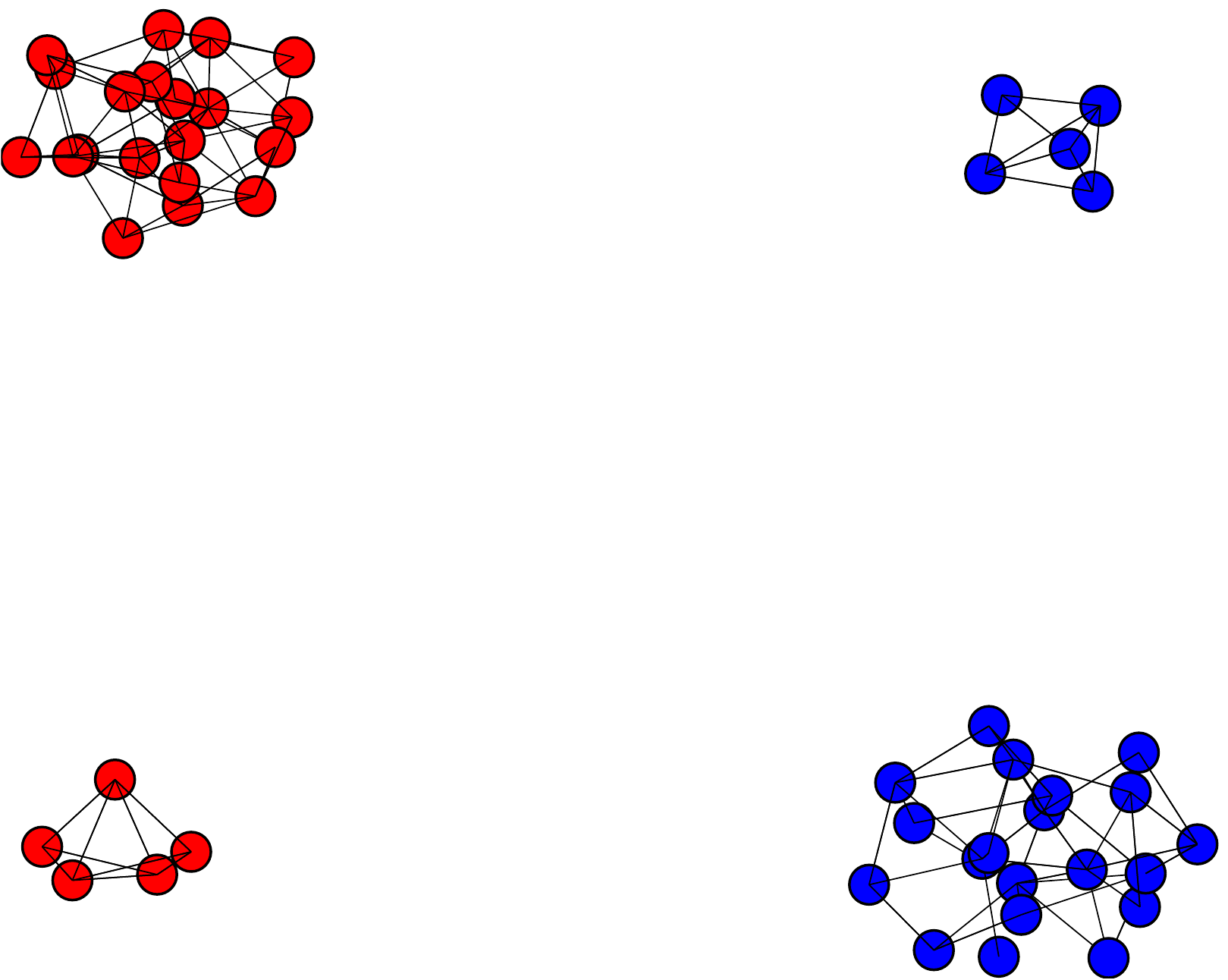} \\  
  {$\la_X =\la_Y=10$} & Graphs \\\hline 
\end{tabular}
\caption{Given two sets of points $X$ (in blue) and $Y$ (in red), we show the points $Z=\diag(\Sigma \U)^{-1} \Sigma Y$ (in green), and the mappings $\Sigma_{i,j}$ as line segments connecting $X_i$ and $Y_j$, which are dashed if $\Sigma_{i,j} \in ]0.1,1[$  and solid if $\Sigma_{i,j}=1$. The results were obtained with the relaxed and regularized OT formulation, setting the parameters to $\kappa=(0.1,8,0.1,8)$. Note the influence of a change in $\la_X$ and $\la_Y$ on the final result: with no regularization ($\la_X=\la_Y=0$) only few points in the data set are matched. The introduction of regularization ($\la_X=\la_Y=0.001$) spreads the connections among the clusters, while maintaining the cluster-to-cluster matching. For  a high value of $\la_X=\la_Y=10$, the regularization tends to match the clusters with similar shape with each other, where the shape is defined by the graph structure. The graphs $\Gg_X$ and $\Gg_Y$ are represented with the nodes on blue and red respectively, and the edges as solid lines.}\label{exlk}
\end{figure*}

For $\la_X=\la_Y=0$ one obtains the relaxed symmetric OT solution, where the transport maps the points in $X$ to the closest point on $Y$, and vice versa. As we increase the values of $\la_X$ and $\la_Y$ to $0.001$, we can see how the regularization affects the mapping. Let us analyze $\regul{G_X \De_{X,Y}(\Sig)} = \|G_X \diag(\Sig \U) X - G_X \Sig Y\|^2 $, for instance. The  term $G_X \diag(\Sig \U) X$ is measuring the regularity of the weights $\diag(\Sig \U)$ on $X$ and the consequence is that for $\la_X=\la_Y=0.001$ there are 
plenty of connections with low weight (there are few solid lines), while for $\la_X=\la_Y=0$ there are several mappings with $\Sig_{i,j} = 1$ (solid lines). So, the regularization promotes a spreading of the matchings. 

The minimum of $\regul{G_X \De_{X,Y}(\Sig)}$ is reached when $G_X \diag(\Sig \U) X = G_X \Sig Y$, that is, when the graph structure of $X$ has the same shape as the graph structure of $\Sig Y$, which both can be observed in the last column and row. For high values of $\la_X=\la_Y$ the matchings tend to link the clusters by their shape, that is, the big cluster on $X$ with the big cluster of $Y$, and similarly for the small clusters (note that the links with higher value are between the small clusters). 

%We can see the same behavior in the second row, where the results were obtained with the regularized asymmetric method, by just setting $\kappa=(1,1,0.1,10)$ and $\la_Y=0$. The main difference with the previous case is that all the points in $X$ (blue) are all being mapped to some point in $Y$ (red). 

%%%%%%%%%%%%%%%%%%%%%%%%%%%%%%%%%%%%%%%%%%%%%%%%
%%%%%%%%%%%%%%%%%%%%%%%%%%%%%%%%%%%%%%%%%%%%%%%%
%%%%%%%%%%%%%%%%%%%%%%%%%%%%%%%%%%%%%%%%%%%%%%%%
%\section{Application in Image Processing}

%%%%%%%%%%%%%%%%%%%%%%%%%%%%%%%%%%%%%%%%%%%%%%%%
%\subsection{Application to Dimensionality Reduction for Vizualization}
%\label{sec-appli-vizu}
%
%Here only make use of the regularized distance, not the transport in itself. 
%
%This is a simple application to visualization. Although this is not the purpose of this article, one could use similar technics to perform data retrieval (shape or image Google). 
%
%%%%%%%%%%%%%%%%%%%%%%%%%%%%%%%%%%%%%%%%%%%%%%%%
%\todo{Intro explaining that for some application, the transport itself (and not the distance) is actually useful) }

\section{Application to Color Transfer}
\label{sec-appli-color} 

This section shows how the relaxed and regularized OT formulation can be applied to imaging problems, more specifically to color transfer, and how the regularization and the relaxation improve the results obtained by previous methods. The color transfer problem consists in modifying an input image $X^0$ so that its colors match the colors of another input image $Y^0$. 

\begin{figure*}[h]
\centering
\begin{tabular}{@{}c@{\hspace{1mm}}c@{\hspace{1mm}}c@{}}
\includegraphics[width=.32\linewidth]{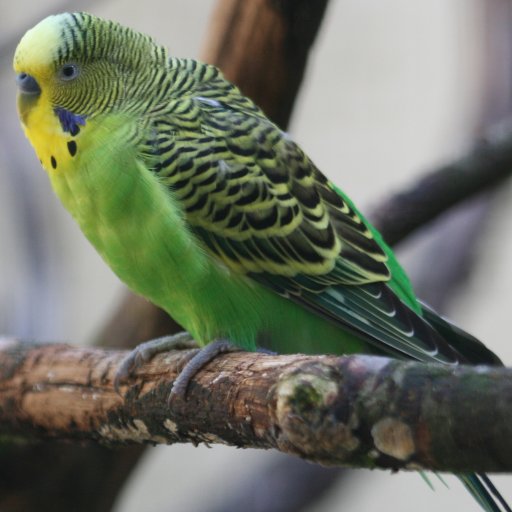} &
\includegraphics[width=.32\linewidth]{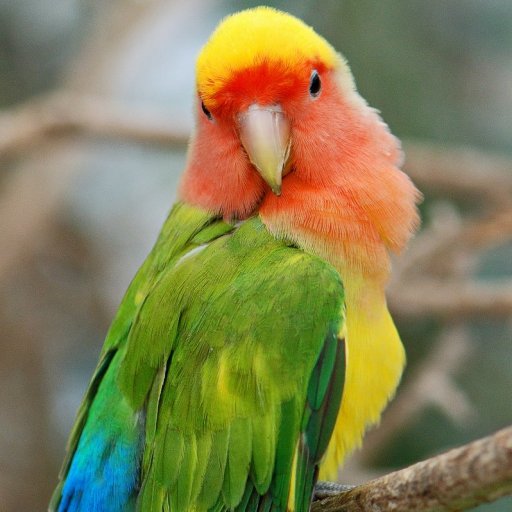} &
\includegraphics[width=.32\linewidth]{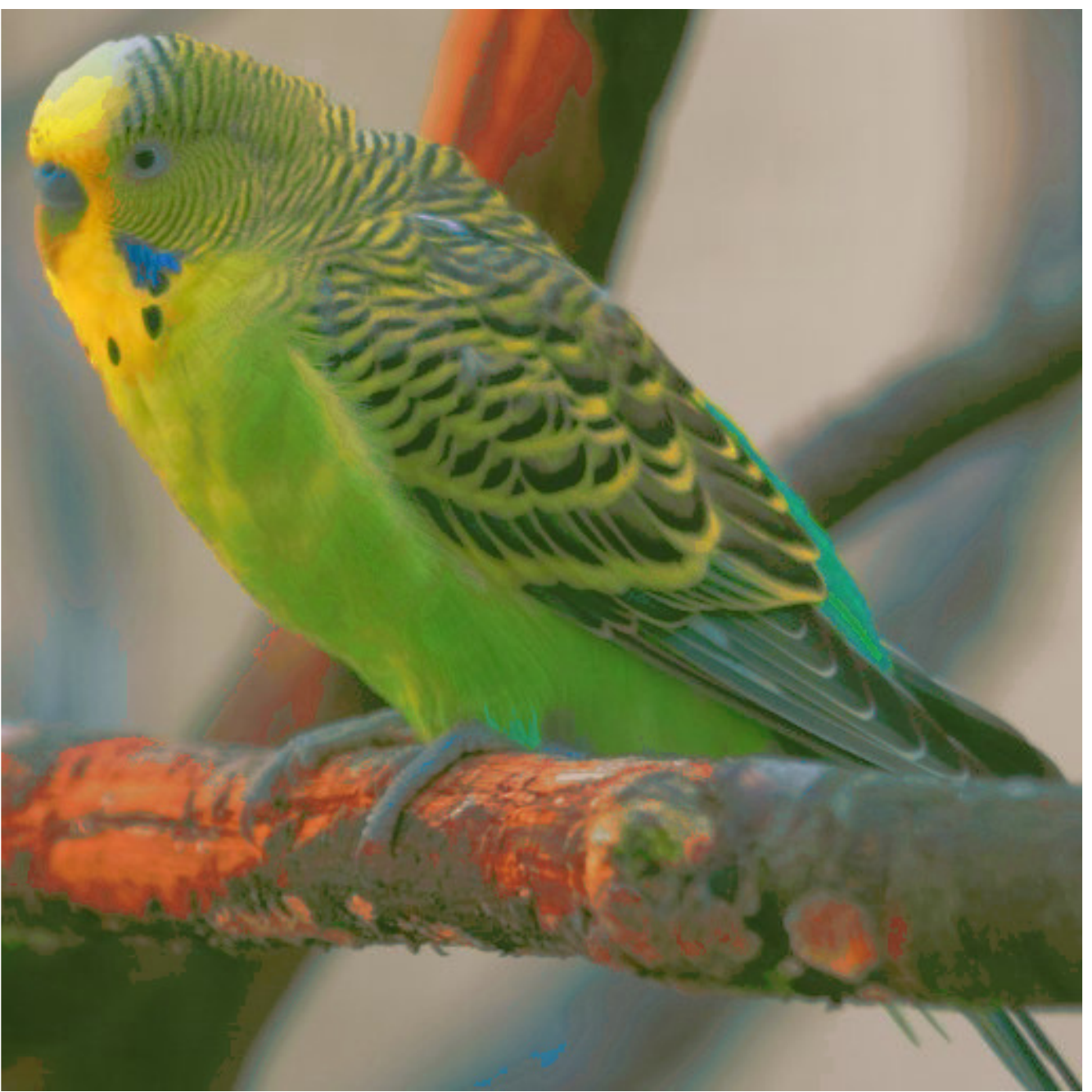} \\ 
$X^0$ & $Y^0$ & ${\tilde X^0}$ \\
\includegraphics[width=.32\linewidth]{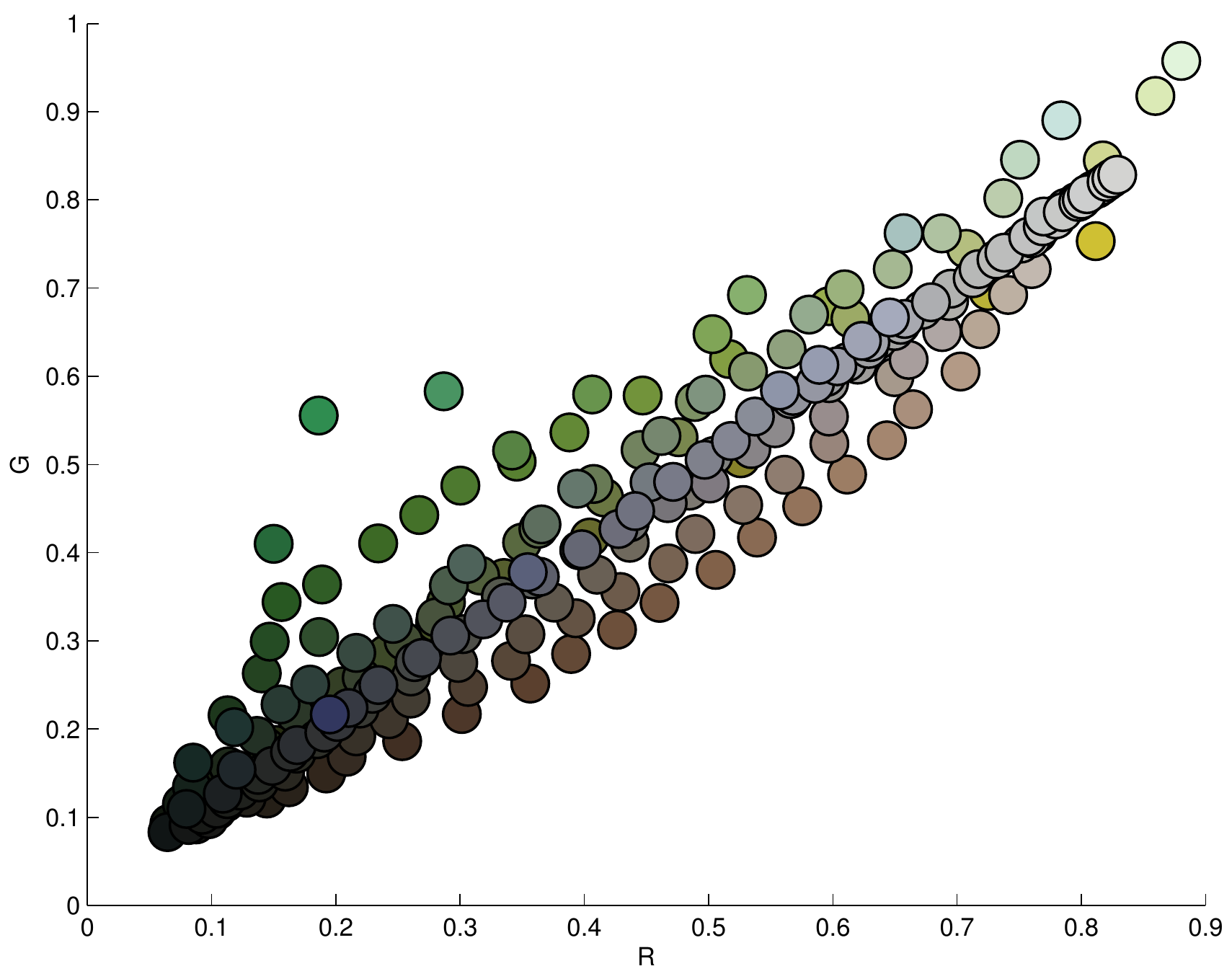} &
\includegraphics[width=.32\linewidth]{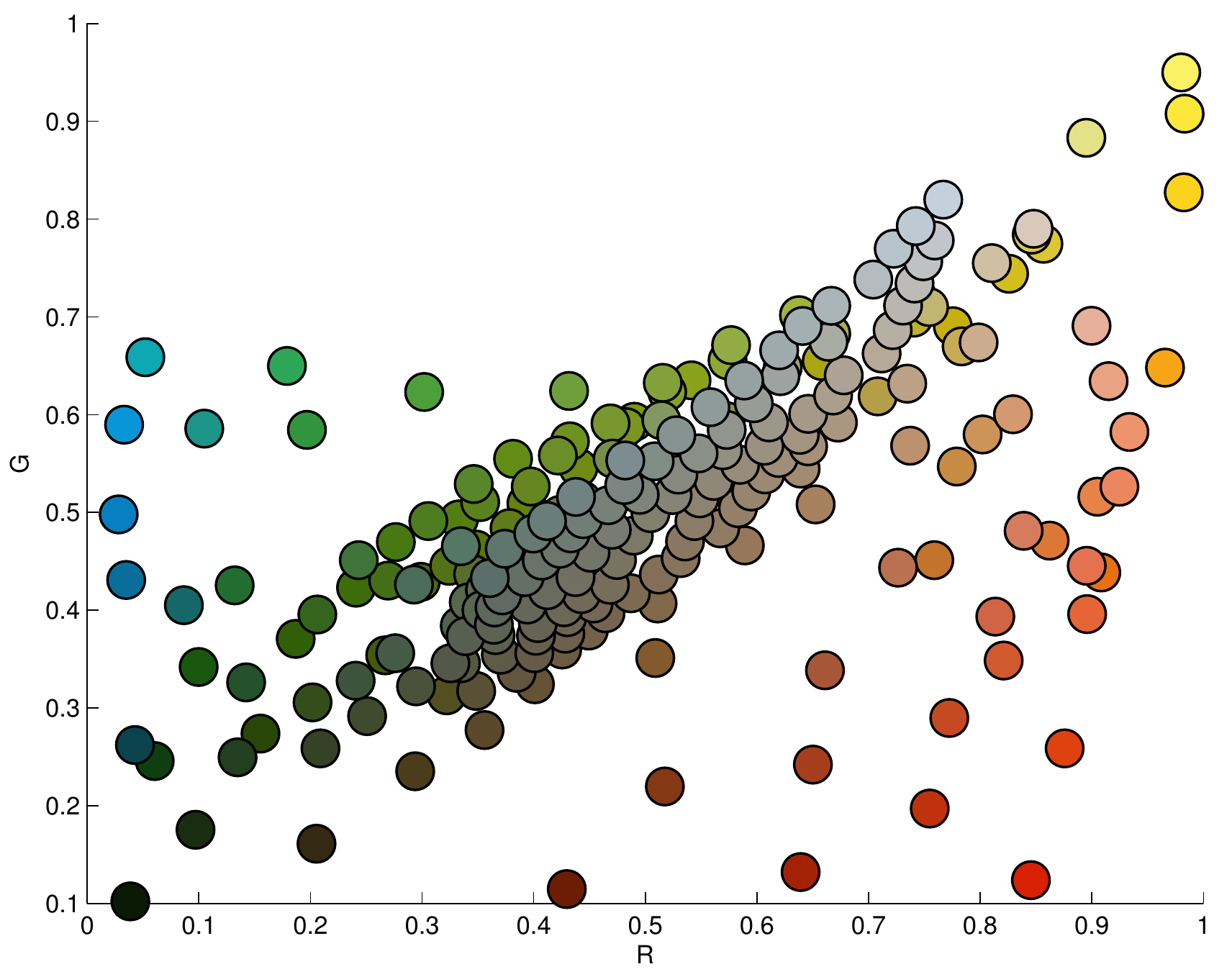} &
\includegraphics[width=.32\linewidth]{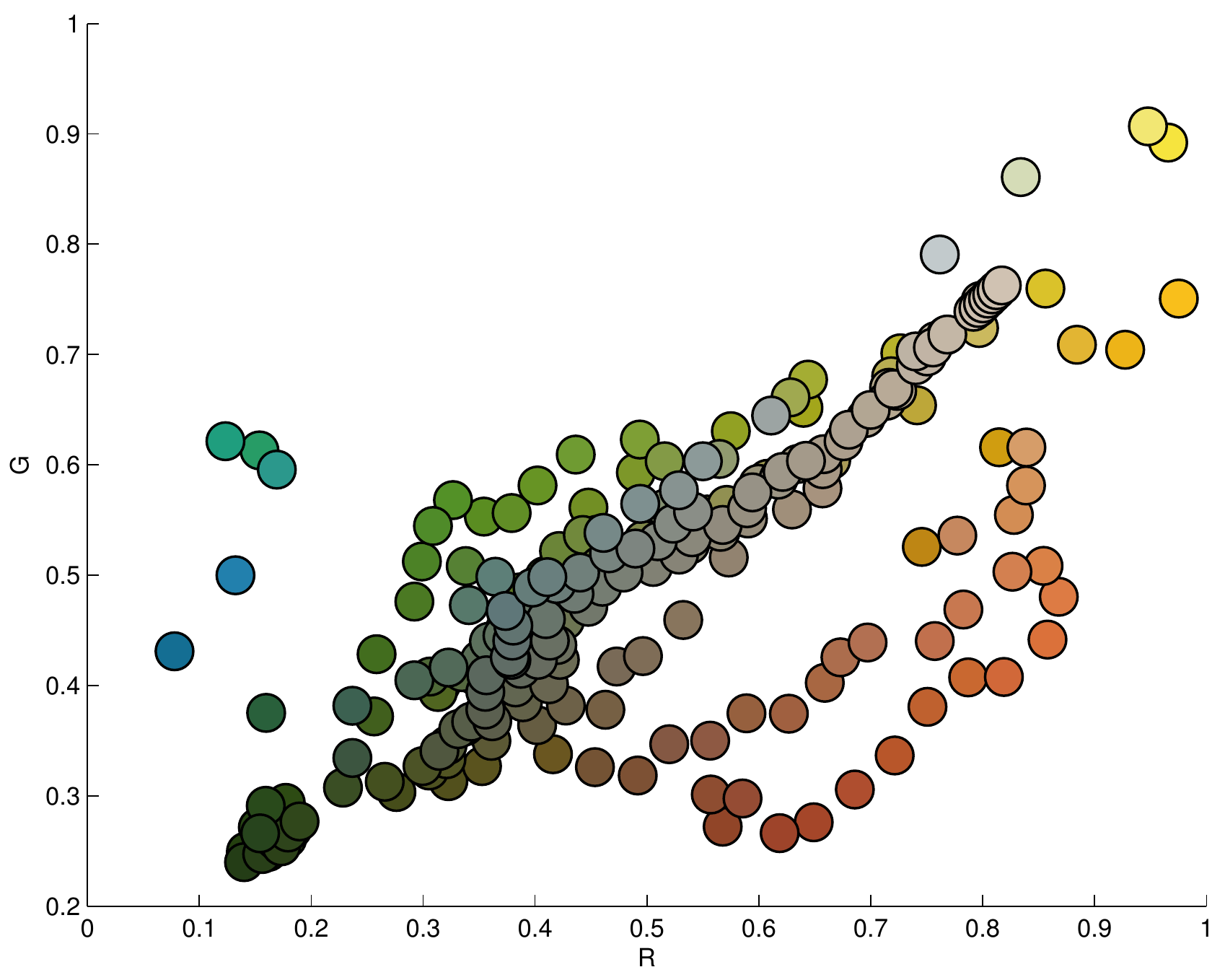}\\
$\mu_{X^0}$ & $\mu_{Y^0}$ & $\mu_{\tilde X^0}$
\end{tabular}
\caption{\label{imcolorization}Example of the colorization problem. Given images $X^0$ and $Y^0$ with their corresponding 3-D color distributions $\mu_{X^0}$ and $\mu_{Y^0}$ (represented here using their 2-D projection on the RG plane), the goal of colorization methods is to define an image ${\tilde X^0}$ that has the geometry of $X^0$ and a histogram $\mu_{{\tilde X}^0}$  that is similar to $\mu_{Y^0}$.}
\end{figure*}

%%%%%%%%%%%%%%%%%%%%%%%%%%%%%%%%%%%%%%%%%%%%%%%
\subsection{Color Images and Histograms}

In the following, an image is stored as a vector $X^0 \in \RR^{N_0 \times d}$ where $d=3$ is the number of channels (here $d=3$ since we handle color images, with R, G and B color channels) and where $N_0=N_1 N_2$ is the number of pixels ($N_1$ being horizontal and $N_2$ vertical dimensions). The color histogram of such an image $X^0$ can be estimated using the empirical distribution $\mu_{X^0}$. The goal of color transfer algorithms is to compute a transformation $T^0$ such that $\left(\tilde X^0\right)_i = T^0(X^0_i)$, where the new empirical distribution $\mu_{\tilde X^0}$ is close (or equal) to $\mu_{Y^0}$. Figure~\ref{imcolorization} shows an example where $X^0$, $Y^0$ are the original input images, the second row displays the 2-D projection of the 3-D distribution of pixels $\mu_{X^0}$ and $\mu_{Y^0}$, and in the third column, we show the $\mu_{\tilde X^0}$ which is the result of applying $T^0$ to $X^0$, where $T^0$ is computed using the method described below. The associated image ${\tilde X^0}$ has the geometry of $X^0$ and the color palette (3-D histogram) of $Y^0$.

%%%%%%%%%%%%%%%%%%%%%%%%%%%%%%%%%%%%%%%%%%%%%%%
\subsection{Regularized OT Color Transfer}

As exposed in Section~\ref{subsec-ot-imaging}, OT is now routinely used to perform color palette modification, and in particular color transfer. As we illustrate below in the numerical examples, relaxing the mass conservation constraint is crucial in order to better match the modes (i.e. the dominant colors) of each distribution. Regularizing the transport is also important to reduce colorization artifacts.

To make the optimization problem~\eqref{eq-symm-reg-energy} tractable for histograms obtained from large scale images, we apply the method on a sub-sampled point cloud. That is to say, before computing the relaxed and regularized transport, we define two smaller point clouds $X$ and $Y$ from $X^0$ and $Y^0$. These clouds are created such that their respective distributions $\mu_X$ and $\mu_Y$ are close to the two original distributions $\mu_{X^0}$ and $\mu_{Y^0}$. The mapping $T$ between these small clouds is then extended by interpolation to the original clouds. The complete algorithm for regularized OT color transfer between a pair of images $(X^0,Y^0)$ is exposed in Algorithm~\ref{alg-rot}. We now detail each step of the method. 

\begin{algorithm}[ht!]
\caption{Regularized OT Color Transfer}
\label{alg-rot}
% \begin{algorithmic}[1]
\Require Images $X^0,Y^0 \in \RR^{N_0 \times d}$, $\la_X,\la_Y \in \RR^+ $, and $k_X,K_X,k_Y,K_Y \in \RR^+$, where $k_X \le K_X$ and $k_Y \le K_Y$.

\Ensure Image $\tilde X^0 \in \RR^{N_0 \times d}$.
% \Statex
\begin{enumerate}
	\algostep{Histogram down-sample} Compute $X,Y$ from $X^0,Y^0$ respectively \\ using K-means clustering.
	\algostep{Compute Mapping} Compute the optimal $\Sig$ such that  $T(X) = \diag(\Sig \U)^{-1} \Sig Y$ by solving eq.~\eqref{eq-symm-reg-energy} with algorithm~\eqref{eq-frankwolfe-update} or the linear program~\eqref{eq-symm-TV} solving with an interior point algorithm.
	%\algostep{Transport up-sample} Compute $\tilde T^0$ by solving eq.~\eqref{eq-upsample}, where  $T(X) = diag(\Sig \U)^{-1} \Sig Y$.
	\algostep{Obtain high resolution result} Compute $\tilde X^0$ with eq.~\eqref{eq-upsample}.
\end{enumerate}
% \end{algorithmic}
\end{algorithm}

%%%%%%%%
\paragraph{Pixels down-sampling} 

We construct a smaller data set $X \in \RR^{N \times d}$ by clustering the set $X^0$ into $N$ clusters with the K-means algorithm (see~\cite{Lloyd57}).
% and~\cite{Quantization} for its application to image quantization). 
Each cluster corresponds to a point $X_i$ in our smaller data set $X$.  The same procedure is done for $Y^0$ to obtain $Y \in \RR^{N \times d}$.%, and we compute  

%%%%%%%%
\paragraph{Graph and $(G_X,G_Y)$ operator} 

As exposed in Section~\ref{sec:regsymme}, the regularization is defined using gradient operators $(G_X,G_Y)$ on graphs $(\Gg_X,\Gg_Y)$ connecting the points in $X$ and $Y$. Inspired by several recent works on manifold learning (see Section~\ref{subsec-regul-intro}), we use here a $n$-nearest neighbor graph, where $n$ is the number of edges adjacent to each vertex, i.e. $\abs{\enscond{j}{(i,j) \in E_X}} = n$ where $E_X$ is the set of edges of $X$. The weights of the graphs are defined as $w_{i,j}=\norm{X_i - X_j}^{-1}$ (same applies for $Y$), which is consistent with the computation of the directional derivatives. An example of this graph can be observed in Figure~\ref{im:matching}. Note that this graph does not need to be fully connected. 

%%%%%%%%
\paragraph{Transport map computation}

The regularized transport map $T$ between the sub-sampled data $(X,Y)$ is computed as
\eq{
	T(X_i)=\left(\diag(\Sig \U)^{-1}\Sig Y\right)_i \qforq i={1,\ldots,N}
} 
where $\Sig$ is a solution of~\eqref{eq-symm-reg-energy}.

%%%%%%%%
\paragraph{Transport map up-sampling}

The transport map $T$ is extended to the whole space using a nearest neighbor interpolation 
% to every point the transport map of the nearest point in $X$: 
\eql{\label{eq-upsample}
	% (\tilde X^0)_{i} = 
	\foralls x \in \RR^d, \quad
	T^0(x) = T(X_{i(x)}) + x - X_{i(x)}, 
	\qwhereq
	i(x) = \uargmin{1 \leq i \leq N} \norm{ x - X_i }. 
} 
Note that this interpolation scheme contains an additive term $x-X_{i(x)}$. This corresponds to adding back the quantization error (due to the K-means sub-sampling) to the nearest neighbors interpolation, which helps to restore small scale textural details, and improves the visual quality of the result. This transport can now be  applied to the input image $X^0$ to obtain the new pixel values $(\tilde X^0)_{i} = T^0(X^0_i)$.

\begin{figure}[ht]
\centering
\begin{tabular}{@{}c@{\hspace{1mm}}c}
\includegraphics[width=.4\linewidth]{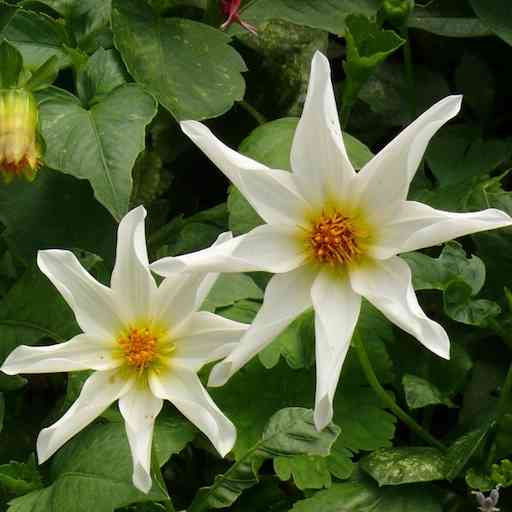} &
\fbox{\includegraphics[width=.55\linewidth,height=5.1cm]{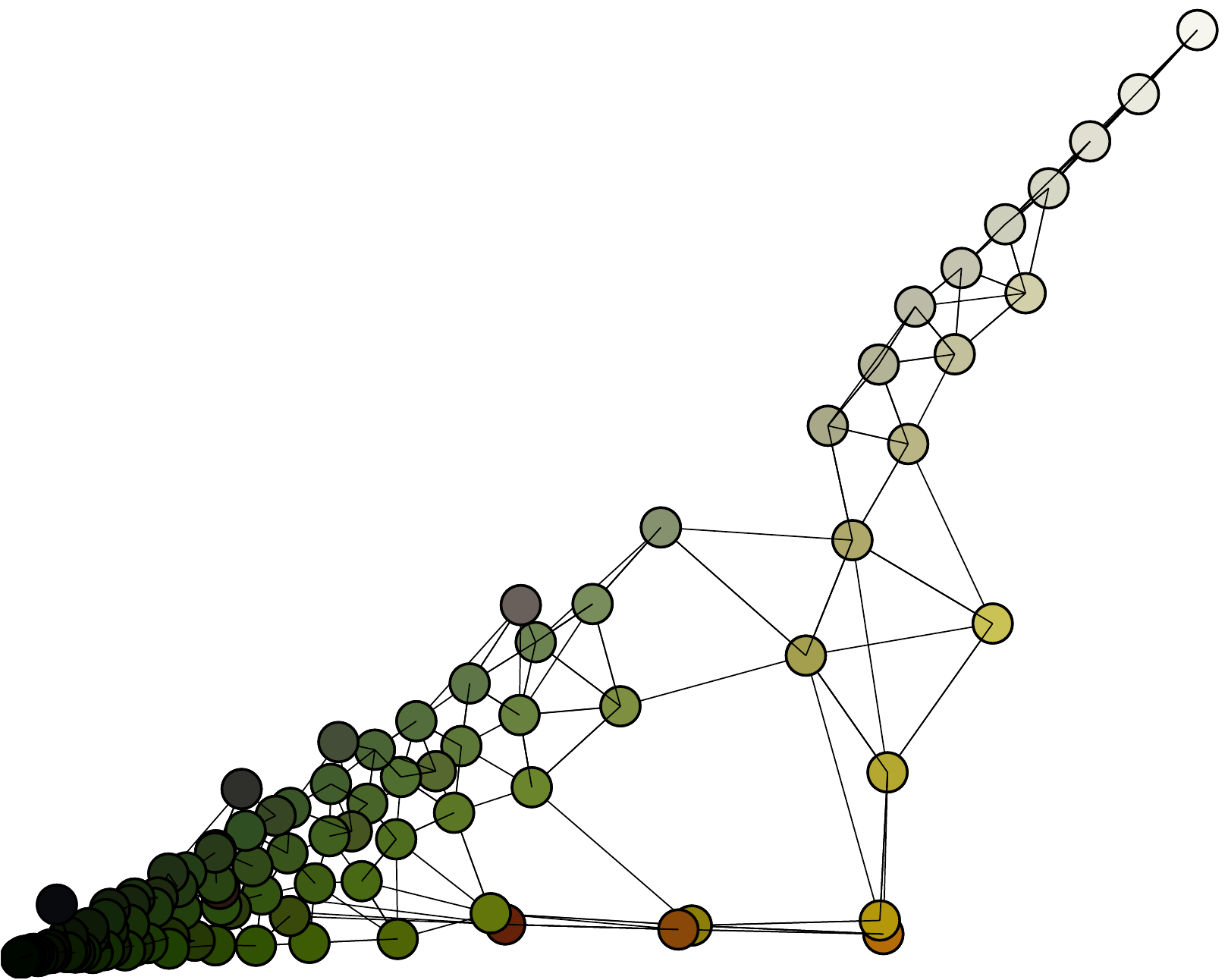}} \\
(a) & (b)
\end{tabular}
\caption{(a)Flower image (b)its empirical distribution projected on the Red-Blue plane. The line segments represent the edges $E_X$ of the $n$-nearest neighbor graph computed with $n=4$. }
\label{im:matching}
\end{figure}

\newlength{\mylenX}\settowidth{\mylenX}{\includegraphics[width=1.75cm]{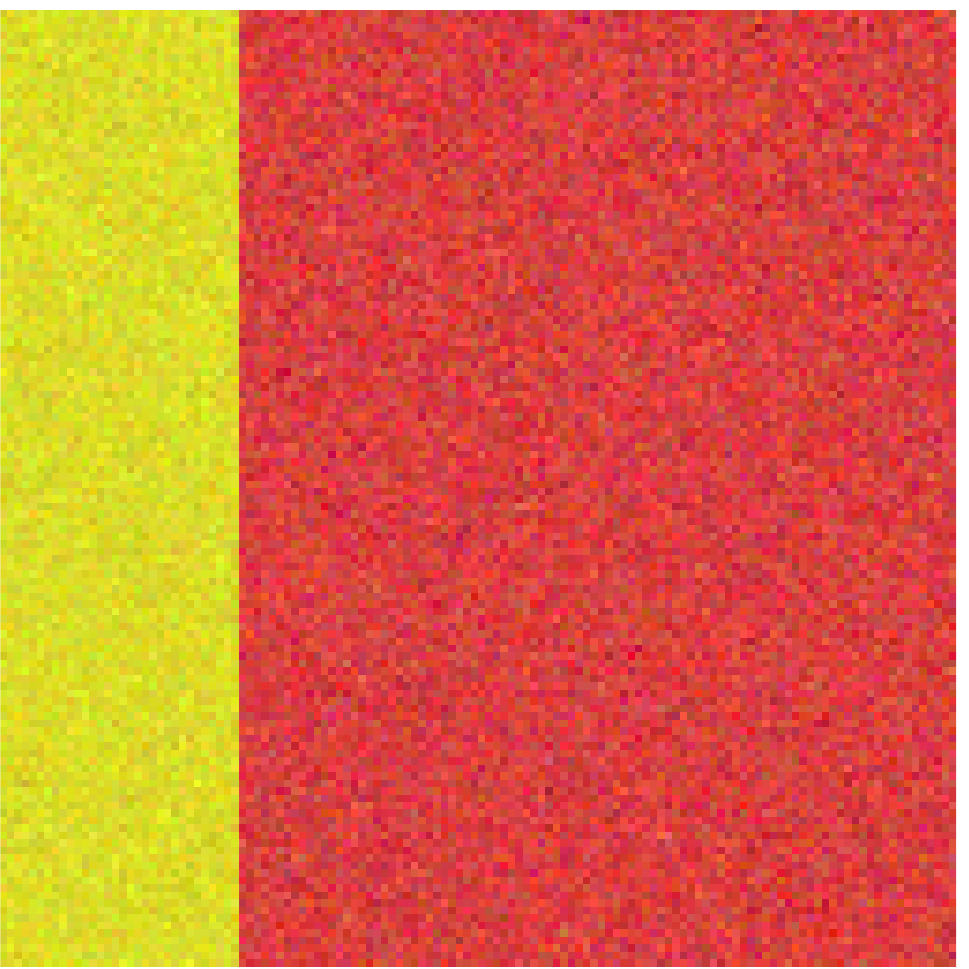} } % Widest element
\newcommand{\sidecapX}[1]{ {\begin{sideways}\parbox{1.65cm}{\centering #1}\end{sideways}} }

\begin{figure}[ht]
\centering
\begin{tabular}{@{}c@{}}
\sidecapX{\scriptsize $X^0$} \includegraphics[width=.13\linewidth]{../images/syntheticexamples/X} \\
\sidecapX{\scriptsize $Y^0$} \includegraphics[width=.13\linewidth]{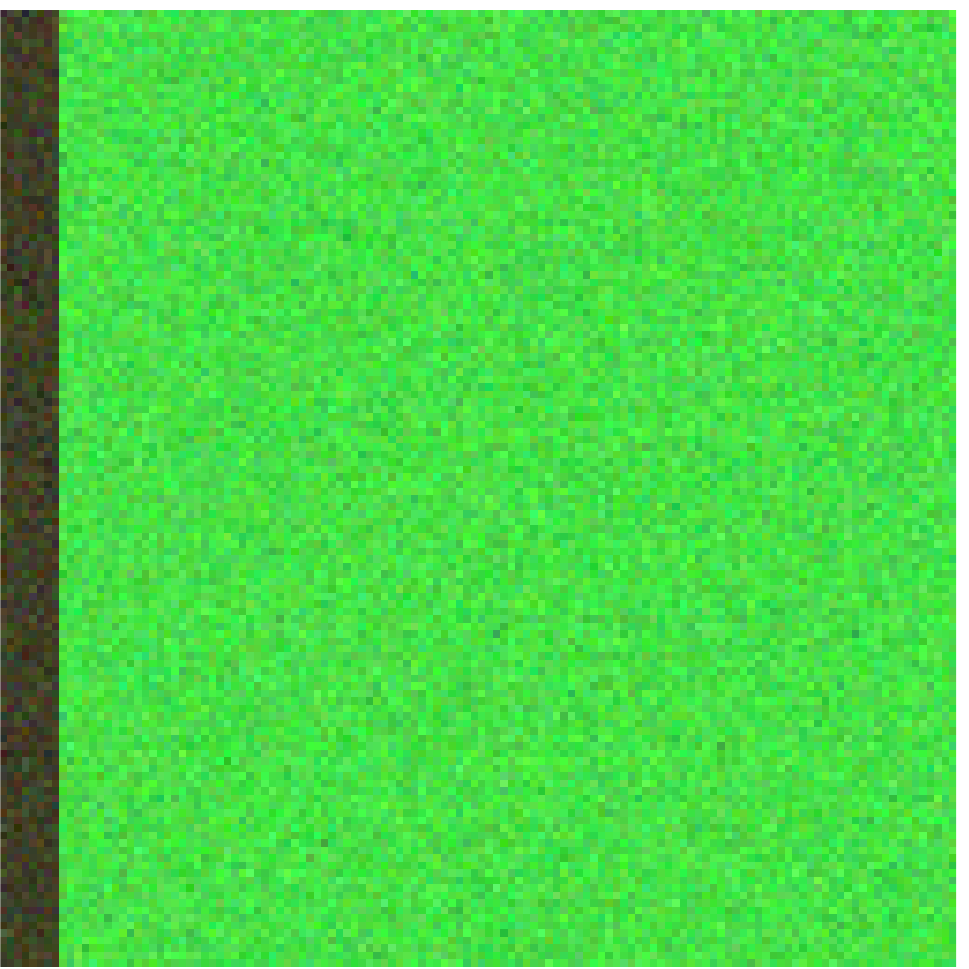} \\
 (a)
\end{tabular} 
\begin{tabular}{@{}c@{\hspace{1mm}}c@{\hspace{1mm}}c@{}}
\includegraphics[width=.27\linewidth]{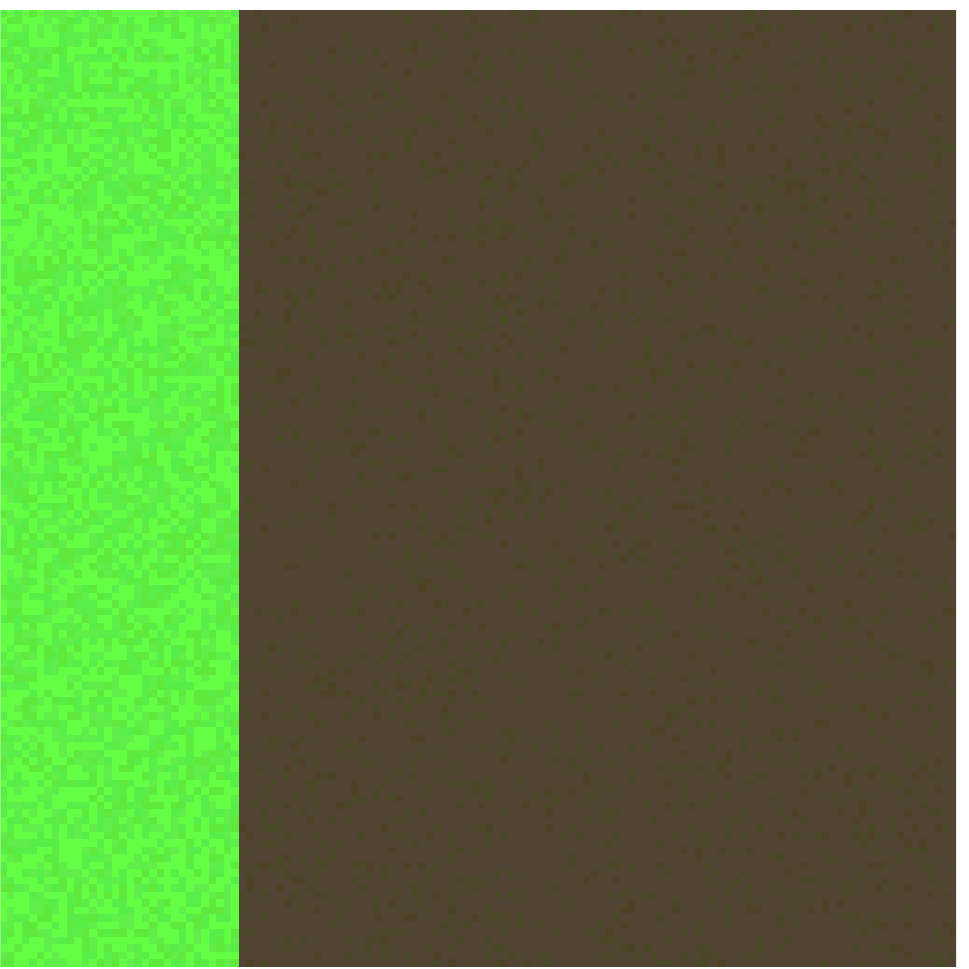} &
\includegraphics[width=.27\linewidth]{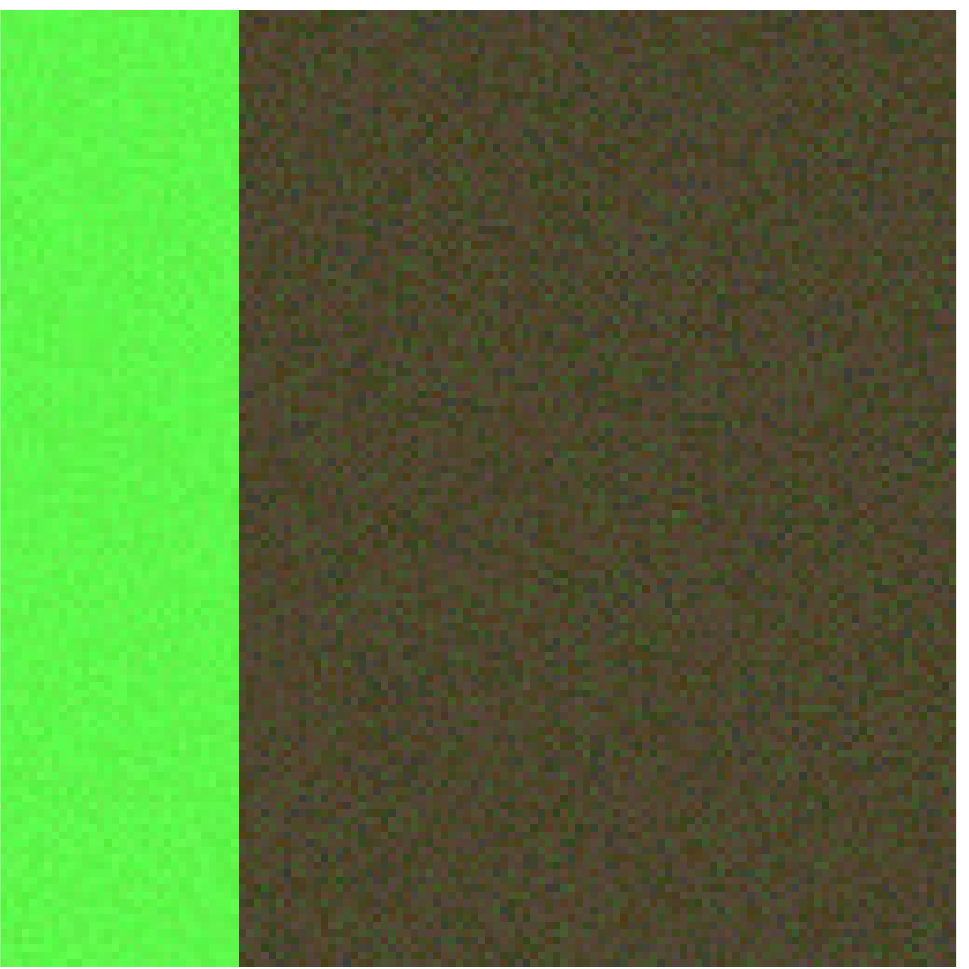} &
\includegraphics[width=.27\linewidth]{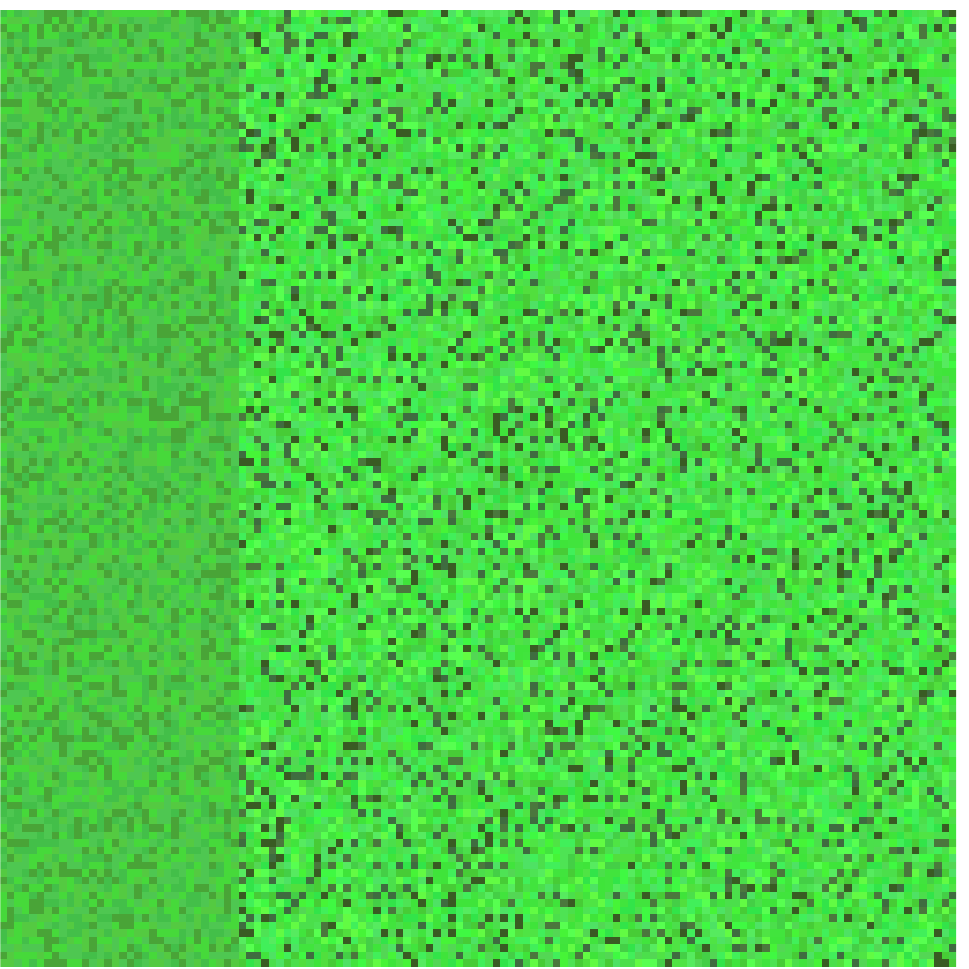} \\
(b) &  (c) & (d) 
\end{tabular} 
\caption{Effect of changing the parameters $\la_X$ and $\la_Y$ of the relaxed and regularized OT formulation presented in section~\ref{sec:regsymme}, using parameters $\kappa=(0.1,8,0.1,8)$.  
	{(a)} original input images, 
	{(b)} relaxed OT, $\la_X=\la_Y=0$; 
	{(c)} $\la_X=\la_Y=0.001$; 
	{(d)} $\la_X=\la_Y=10$. 
	Each of these mappings can be observed in Figure~\ref{exlk}.}
  \label{im:synth}
\end{figure}

\if 0
%% REMOVED
\begin{figure}[ht]\centering
\begin{minipage}{2.15cm}
\sidecapX{\scriptsize $X^0$} \includegraphics[width=1.75cm]{../images/syntheticexamples/X} \\
\sidecapX{\scriptsize $Y^0$} \includegraphics[width=1.75cm]{../images/syntheticexamples/Y} 
\end{minipage} 
\begin{minipage}{12cm}
\includegraphics[width=3.6cm]{../images/syntheticexamples/symmetricsyntheticinv_l0_KX8_KY8_nn4} \hspace{-0.2cm}
\includegraphics[width=3.6cm]{../images/syntheticexamples/symmetricsyntheticinv_l0001_KX8_KY8_nn4} 
\hspace{-0.2cm}
\includegraphics[width=3.6cm]{../images/syntheticexamples/symmetricsyntheticinv_l10_KX8_KY8_nn4}
 \end{minipage}
 \begin{flushleft}\hspace{0.9cm} (a) \hspace{2.25cm} (b) \hspace{2.85cm} (c) \hspace{2.85cm} (d) \end{flushleft} 
 
\caption{Effect of changing the parameters $\la_X$ and $\la_Y$ of the relaxed and regularized OT formulation presented in section~\ref{sec:regsymme}, using parameters $\kappa=(0.1,8,0.1,8)$.  
	{(a)} original input images, 
	{(b)} relaxed OT, $\la_X=\la_Y=0$; 
	{(c)} $\la_X=\la_Y=0.001$; 
	{(d)} $\la_X=\la_Y=10$. 
	Each of these mappings can be observed in the columns of Figure~\ref{exlk}, see text for more details.}
  \label{im:synth}
\end{figure}
\fi

%%%%%%%%%%%%%%%%%%%%%%%%%%%%%%%%%%%%%%%%%%%%%%%
\subsection{Results}

Figure~\ref{im:synth} shows an example of color transfer between two synthetic images $X^0$ and $Y^0$ shown in Figure~\ref{im:synth}~(a). We apply Algorithm~\ref{alg-rot} to obtain the image ${\tilde X^0}$ with a color palette close to $Y^0$, but with the geometry of the original $X^0$. We now study the influence of the parameters $\la_X$ and $\la_Y$. Figure~\ref{exlk} shows a 2-D projection in the Red-Green plane of $X$ and $Y$, displayed using respectively red and blue, and $\tilde X$ in green. As already pointed out in Section~\ref{secalgosymm},  a low value of $\la_X$ and $\la_Y$ (zero for the first column) tends to match the points in X to the closest point in Y. This behavior can be observed in the map of the column (b). Many points in the big cluster of $X$ are mapped to very few points in the small cluster of $Y$, which corresponds in the images to mapping many red values of $X^0$ to very few brown values in $Y^0$. The consequence is that the color resolution of $\tilde X$ is reduced, the brown area of Figure~\ref{im:synth}~(b) is flat, unlike the original brown values in $Y$. As we increase the value of $\la_X$ and $\la_Y$, the mapping spreads within the small cluster of $Y$ in Figure~\ref{exlk}(b) and we gain color resolution, as can be observed in Figure~\ref{im:synth}~(c). 
On the other hand, if we increase too much the values of $\la_X$ and $\la_Y$, many points in $X$ get matched to the big cluster in $Y$ in Figure~\ref{exlk}~(c) which leads to a single dominant color in the final image $\tilde X^0$, in Figure~\ref{im:synth}~(d).

\paragraph{Comparison with the state of the art} 

Figure~\ref{star} shows some results on natural images and compare them with the methods of Piti\'e et al.~\cite{Pitie07} and Papadakis et. al~\cite{Papadakis_ip11}.  The goal of the experiment is to transfer the color palette of the images in the second row to the image on the first row.  Note that the methods in the state of the art introduce color artifacts (in the first column there is violet outside the flower, and in the second column the wheat is blueish), which can be avoided with the proposed method by an appropriate choice of $\la_X$, $\la_Y$ and $\kappa$. These results were obtained setting $N=400$ and constructing the graph as a $4$-nearest neighbor graph. By column, the values of $\la_X=\la_Y$ are $9\times 10^{-4}, 5\times 10^{-4}$, and $10^{-3}$,  and $\kappa$ was set to (0.1,1.1,0.1,1.1), (0.1,1.3,0.1,1.3), and (0.1,1,0.1,1), respectively.

\newlength{\mylen}\settowidth{\mylen}{\includegraphics[width=3.8cm,height=3.3cm]{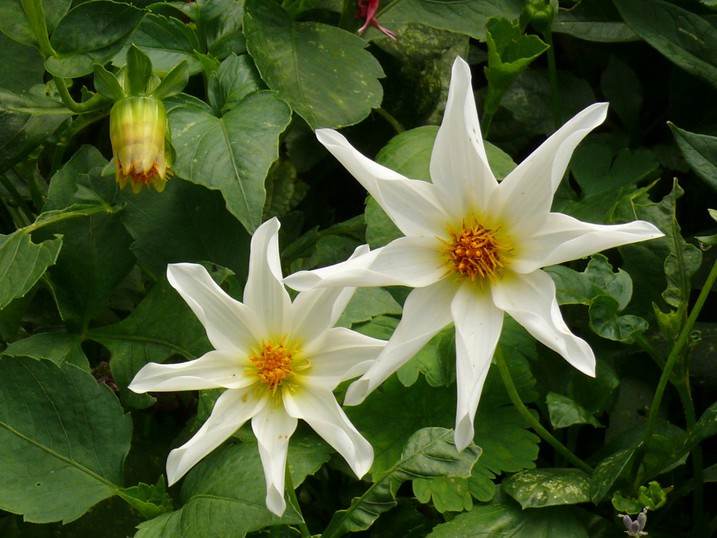}} % Widest element

\newcommand{\sidecap}[1]{{\begin{sideways}\parbox{3.3cm}{\centering #1}\end{sideways}} \hspace{-4mm}}

\newcommand{\myimg}[1]{\includegraphics[width=.3\linewidth,height=.26\linewidth]{#1}}

\begin{figure}[!ht]
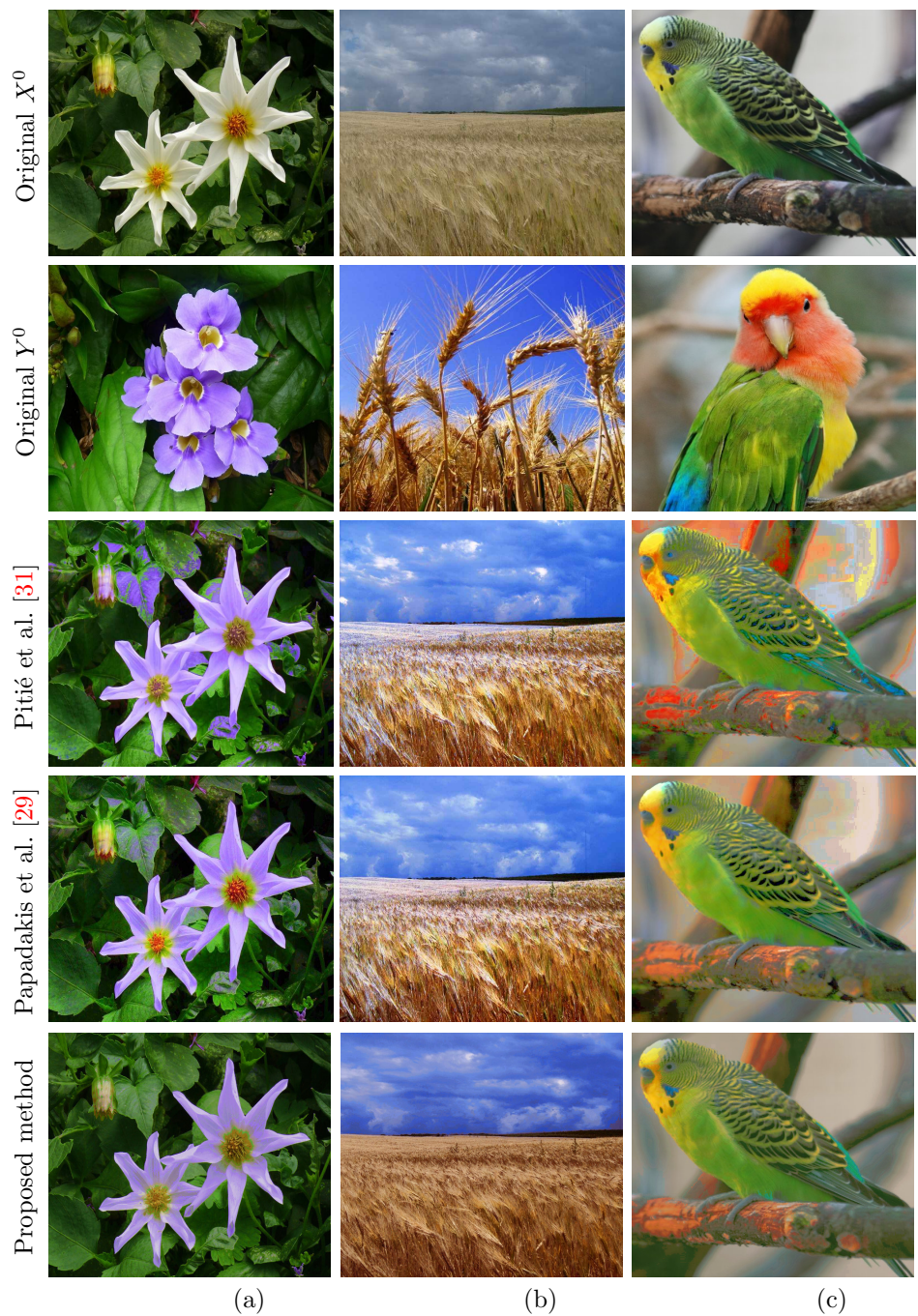

\centering
\begin{tabular}{@{}@{}cc@{\hspace{1mm}}c@{\hspace{1mm}}c@{}}
\sidecap{ Original $X^0$ }  & 
\myimg{star/fleur_1} &
\myimg{star/wheat_1.jpg} &
\myimg{star/parrot_1.jpg} \\
\sidecap{ Original $Y^0$ } & 
\myimg{star/fleur_2} &
\myimg{star/wheat_2.jpg} &
\myimg{star/parrot_2.jpg} \\
\sidecap{ Piti\'e et al.~\cite{Pitie07} } & 
\myimg{star/fleur_pitie} &
\myimg{star/wheat_pitie} &
\myimg{star/parrot_pitie} \\
\sidecap{ Papadakis et al.~\cite{Papadakis_ip11} } & 
\myimg{star/fleur_papadakis} &
\myimg{star/wheat_papadakis}  &
\myimg{star/parrot_papadakis} \\
\sidecap{ Proposed method } & 
\myimg{symmetric/symmetricfleur_l00009_KX11_KY11_nn4} &
\myimg{symmetric/symmetricwheat_l00005_KX13_KY13_nn4} &
\myimg{symmetric/symmetricparrot_l0001_KX1_KY1_nn4} \\
 & \hspace{1.6cm}(a) & \hspace{1.6cm}(b) &\hspace{1.6cm}(c)
\end{tabular}
\caption{Comparison between the results obtained with our method and with the methods of~\cite{Pitie07} and~\cite{Papadakis_ip11} for image colorization. Note how the proposed method is able to generate results without color artifacts for example, in \textbf{(a)} the violet color of the flower is not spread outside the flower, in \textbf{(b)} the wheat does not become bluish and in \textbf{(c)} the result does not enhance or colorize differently the flat areas of the background.}
\label{star}
\end{figure}

\section{Regularized OT Barycenters}\label{sec-barycenters}

As presented in the introduction, Section~\ref{subsec-ot-imaging}, for certain applications in imaging such as texture mixing or color normalization, it may be useful to compute the barycenter distribution of a set of input distributions. Until now we focused on the computation of the mapping between two given distributions, now we are interested in finding a new distribution in-between two or more distributions. 

%%%%%%%%%%%%%%%%%%%%%%%%%%%%%%%%%%%%%%%%%%%%%%%%%%%%%%%%%%%%%%%%%%%%%%%%%%%

\paragraph{Asymmetric regularized OT metric}

To simplify the optimization process, we consider the asymmetric version of the regularized OT energy~\eqref{eq-symm-reg-energy}. We maintain one data set as a reference, let say $X$, by taking into account all its points ($k_X=K_X=1$) and only perform regularization with respect to its own graph, i.e. $\la_Y=0$. Thus, we simplify our expression into the following asymmetric distance: 
\eql{\label{eq-assym}
	D(\mu_X,\mu_Y)=\umin{\Sig \in \Dd_k}
	E(\Sigma) = \dotp{\Cost{X}{Y}}{\Sigma} + 
	\lambda \regul{G_X \De_{X,Y}(\Sig)} 
} 
\eq{
	\qwhereq \Dd_k=
	\enscond{\Sig \in [0,1]^{N \times N} }{ \Sig   \U  = 1, \; \Sig^* \U \leq k \U }.
}
Note that $\Matr_{\kappa}=\Dd_k$ for $\kappa=(1,1,0,k)$. In general, $D$ is not a distance, since it is not symmetric and one can have $D(\mu_X,\mu_Y)=0$ while having $\mu_X \neq \mu_Y$ (which is crucial to allow relaxing of mass conservation condition).  

\paragraph{Barycenter}Given a set of input clouds $(X^{[r]})_{r \in R}$ indexed by $R$ and weights $(\rho_r)_{r \in R} \in (\RR^+)^R$, we define a barycenter cloud $X$ as a local minimizer of 
\eql{\label{eqbar}
	\umin{X  \in \RR^{N \times d}}
		\Ee_\rho(X)=
		\sum_{r \in R} \rho_r \: D(\mu_{X^{[r]}},\mu_{X}).
}
In the case $\lambda=0$ and $k=1$, one recovers barycenters over the Wasserstein space, see the introduction for more details.

%%%%%%%%%%%%%%%%%%%%%%%%%%%%%%%%%%%%%%%%%%%%%%%%%%
\subsection{Block-coordinate Descent}\label{algobar}

The minimization of~\eqref{eqbar} can be performed by doing a joint minimization on both the barycenter cloud $X$ and a set of matrices $\Sigma^{[r]} \in \Dd_k$ 
\eql{\label{eq-barycenter-min}
	\min_{ X, (\Sig^{[r]})_{r \in R} 
	} 
	\sum_{r \in R} \rho_r \left(
		\dotp{ \Cost{X^{[r]}}{X} }{\Sigma^{[r]}} + 
		\lambda \regul{ G_{X^{[r]}} (X^{[r]} - \Sig^{[r]} X) } \right).
}
This is a non-convex optimization problem. Fortunately, it is separately convex with respect to each of its variables $X$ and $( \Sig^{[r]} )_{r \in R}$, so one can use the block coordinate descent scheme. The block coordinate descent method consists in optimizing a given energy by iteratively minimizing with respect to each of its variables, in our case $X$ and  $( \Sig^{[r]} )_{r \in R}$.

%%%%%%%%%
\paragraph{Update $\Sigma^{[r]}$}

This corresponds to performing in parallel $|R|$ independent relaxed regularized OT. Fixing $X$, one solves independently for each $\Sig^{[r]}$ the convex problem
\eql{\label{eq-barycenter-min-Si}
	\umin{\Sig^{[r]} \in \Dd_k }
		\dotp{ \Cost{X^{[r]}}{X} }{\Sigma^{[r]}} + 
		\lambda \regul{ G_{X^{[r]}} (X^{[r]} - \Sig^{[r]} X) }.
}
For $(p,q)=(2,2)$ or $(p,q)=(1,1)$ this minimization can be solved using the algorithms detailed in Section~\ref{secalgosymm}.

%%%%%%%%%
\paragraph{Update $X$} 

Then, one solves for $X$ the following convex optimization problem
\eql{\label{eq-barycenter-min-X}
	\umin{X \in \RR^{N \times d} }
	\Hh(X) = 
	\sum_{r \in R} \rho_r \pa{
		\dotp{ \Cost{X^{[r]}}{X} }{\Sigma^{[r]}} + 
		\lambda \regul{ G_{X^{[r]}} (X^{[r]} - \Sig^{[r]} X) }
	}.
}	

%%%%%%%%%
\paragraph{Update $X$: Sobolev regularization} \label{algobarysobolev}

The minimization of~\eqref{eq-barycenter-min-X} when $p=q=2$ is an unconstrained quadratic problem, whose solution is obtained solving the following symmetric linear system
\eql{\label{eq-bar-l2}
	\sum_{r \in R} \rho_r \left( \Sig^{[r]} - \la \Sig^{[r]*} G_{X^{[r]}}^* G_{X^{[r]}} \right) X = \sum_{r \in R} \rho_r \left( \Sig^{[r]*} X^{[r]} - \la \Sig^{[r]*} G_{X^{[r]}}^* G_{X^{[r]}} X^{[r]} \right) ,
}
which corresponds to solving $\nabla \Hh(X) =0$. The solution to this symmetric linear system can be computed using for instance the conjugate gradient algorithm.

%%%%%%%%%
\paragraph{Update $X$: Anisotropic TV regularization}

When $(p,q)=(1,1)$, \eqref{eq-barycenter-min-X} is a linear program which can be solved using for instance interior point solvers~\cite{Nesterov-Nemirovsky-Book}. An alternative option, that we detail here, is to use first order proximal splitting schemes, that are well tailored for such highly structured problems.  We propose here to use the primal-dual splitting scheme developed in~\cite{Chambolle11}.

The problem~\eqref{eq-barycenter-min-X} can be re-casted as a minimization of the form 
\eql{\label{eq-bar-l1}
\begin{aligned} 
	& \umin{X \in \RR^{N \times d}}  & & F(K(X)) + H(X) \\ 
	& \qwhereq & &
	\choice{
	K(X) = \{ B_r X \}_{r \in R}, \\ 
	K^*(\{U_r\}_{r \in R}) = \sum_{r \in R} B^*_r U_r \\
	F(\{U_r\}_{r \in R}) = \la \sum_{r \in R} \rho_r \|G_{X^{[r]}} X^{[r]} - U_r \|_1 \\
	H(X) =  \sum_{r \in R} \rho_r \dotp{ \Cost{X^{[r]} }{X}}{\Sig^{[r]}}
	}
\end{aligned}
}
where $B_r = G_{X^{[r]}} \Sig^{[r]}$.  Let us now recall that the proximal operator of a function $F$ is defined as 
\eq{ 
	\Prox_{\ga F}(X) = \uargmin{\tilde X} \frac{1}{2}\norm{X-\tilde X}^2 + \ga F(\tilde X),
}
and being able to compute the proximal mapping of $F$ is equivalent to being able to compute the proximal mapping of the Legendre-Fenchel dual $F^*$ of $F$, thanks to Moreau's identity
\eq{
	X = \Prox_{\ga F^*}(X) + \ga \Prox_{F/\ga}(X/\ga).
}
Then, the primal-dual algorithm of~\cite{Chambolle11} to minimize $F \circ K + H$ reads
\begin{align}\label{eq-pd-bary}
	\nonumber \La^{k+1} &= \Prox_{\mu F^*}( \La^k + \mu K(\tilde X^k), \\
	 X^{k+1} &= \Prox_{\tau H}(  X^k-\tau K^*(\La^{k+1}) ), \\
	\nonumber \tilde X^{k+1} &= X^{k+1} + \theta (X^{k+1} - X^k) ,
\end{align}
with $\theta \in (0,1]$ and where 
\begin{align*}
\Prox_{\tau F}(U) &=\sum_{r \in R} G^{[r]} \Sig^{[r]} X^{[r]} + S_\tau(U^{[r]} - G^{[r]} \Sig^{[r]} X^{[r]}) \\
\Prox_{\tau H}(Y) &= \left( \Id+ \tau \sum_{r \in R}  \rho_r \Sig^{[r]}  \right)^{-1} \left(Y + \tau \sum_{r \in R} \rho_r \Sig^{[r]*} X^{[r]} \right)
\end{align*}
where $S_\tau$ is the soft thresholding function, defined as
\eq{
	\foralls i=1,\ldots,N, \quad
	S_\tau(U)_i =  \max\left(0,1- \frac{\tau}{\norm{U_i}}  \right)U_i.
}

%%%%%%%%%%%%%%%%%%%%%%%%%%%%%%%%%%%%%%%%%%%%%%%%%%
\subsection{Algorithm}

The algorithm starts by some initial point set $X^{(0)}$, which is typically chosen to be equal to $X^{[r]}$ where $r$ corresponds to the maximum value of $\rho_r$. It then constructs iterates $(X^{(\ell)})_{\ell}$ and $( \Sigma_{i,j}^{[r],(\ell)} )_r$ by solving respectively~\eqref{eq-barycenter-min-Si} and~\eqref{eq-barycenter-min-X}. This is detailed in Algorithm~\ref{algo-block-barycenters}.

\begin{algorithm}[ht!]
\caption{Regularized and relaxed OT barycenter}
\label{algo-block-barycenters}
\Require Point sets $(X^{[r]})_{r \in R}$, weights $(\rho_r)_{r \in R}$, initialization $X^{(0)}$.

\Ensure Barycenter point set $X^{(\ell)}$, computed for $\ell$ large enough.

\begin{enumerate}
	\algostep{Initialization} Set $\ell=0$.
	\algostep{Update of $\Sigma^{[r]}$} For each $r \in R$, compute $\Sigma^{[r],(\ell+1)}$ by solving~\eqref{eq-barycenter-min-Si} \\
		 where $X=X^{(\ell)}$ is fixed, using the algorithms detailed in Section~\ref{secalgosymm}.
	\algostep{Update of $X$} Compute $X^{(\ell+1)}$ by solving~\eqref{eq-barycenter-min-X} where $\Sigma^{[r]} = \Sigma^{[r],(\ell+1)}$ \\
			are fixed.
			If $(p,q)=(2,2)$ solve \eqref{eq-bar-l2}, if $(p,q)=(1,1)$, use the algorithm~\eqref{eq-pd-bary}.
	\algostep{Convergence} While not converged, set $\ell \leftarrow \ell+1$ and go back to 2.
\end{enumerate}
\end{algorithm}

%%%%%%%%%%%%%%%%%%%%%%%%%%%%%%%%%%%%%%%%%%%%%%%%%%
\subsection{Convergence}

The block coordinate descent methods are known to converge for smooth and differentiable energies~\cite{tseng-proximal}. The following theorem ensures the convergence of the proposed algorithm in the case of the Sobolev regularization. For the anisotropic regularization, one cannot ensure the convergence to stationary points, although in practice, we always observe it in our numerical tests.

\begin{thm}
	When $(p,q)=(2,2)$, the iterates $X^{(\ell)}$ of the algorithm are bounded and hence admit converging sub-sequences.
	The energies $\Ee_\rho(X^{(\ell)})$ (with  $\Ee_\rho$ defined in~(\ref{eqbar}))are decaying and converging to $\tilde{\Ee}$.
	 All converging sub-sequences converge to stationary points of $\Ee_\rho$ having the same energy $\tilde{\Ee}$.
\end{thm}

\begin{proof}
	By construction, the energy $\Ee_\rho(X^{(\ell)})$ is decaying and positive, hence converging.  The algorithm minimizes~\eqref{eq-barycenter-min}, which reads
	\eq{
		\umin{ \Sigma^{[r]} \in \Dd_k, X } 
		\bar \Ee( (\Sigma^{[r]})_r, X )
		= \sum_{i,j,r}  \rho_r \norm{X_i^{[r]} - X_i }^2 \Sigma_{i,j}^{[r]}
		+ \la J_{2,2}\pa{ G_{X^{[r]}} (X^{[r]} - \Sigma^{[r]} X) }.
	}

	Since $\Sigma^{[r]} \in \Dd_k$ which is a bounded set, the iterates $( \Sigma_{i,j}^{[r],(\ell)} )_r$ produced by the algorithm are bounded and hence they admit converging sub-sequences.
	
	For any iteration index $\ell$, one has
	\eql{\label{eq-proof-cv-1}
		\sum_{i,j,r} \rho_r \norm{X_i^{[r]} - X_j^{(\ell)} }^2 \Sigma_{i,j}^{[r],(\ell)} 
		\leq \bar \Ee( (\Sigma^{[r],(\ell)})_r, X^{(\ell)} ) 
		\leq \Ee_\rho(X^{(0)})
	}
	where $( \Sigma_{i,j}^{[r],(\ell)} )_r$ are the matrices obtained at the previous iteration of the method. We let $r$ be any index such that $\rho_r>0$. For any $j$ we denote  $\ga_i = \Sigma_{i,j}^{[r],(\ell)}$ (we ignore dependency with $(j,r,\ell)$ for ease of notations) that satisfy $\sum_i \ga_i=1$  and define the barycenter 	
	\eq{
		\bar X_j = \sum_i \ga_i X_i^{[r]}
	}
	which is a point in the convex hull of the $(X_i^{[r]})_i$, and is hence bounded independently of $j$ and $\ell$.
	
	Equation~\eqref{eq-proof-cv-1} implies
	\eq{
		\sum_{i} \ga_i \norm{X_i^{[r]} - X_j^{(\ell)} }^2 
		\leq \frac{\Ee_\rho(X^{(0)})}{\rho_r}
	}
	By convexity of the function $x \in \RR^d \mapsto \norm{ X_j^{(\ell)} - x }^2$, one has
	\eq{		
		\norm{ X_j^{(\ell)} - \bar X_j }^2 
		\leq
		\sum_{i} \ga_i \norm{X_j^{(\ell)} - X_i^{[r]} }^2  \leq \frac{\Ee_\rho(X^{(0)})}{\rho_r}
	}	
	This shows that the iterates $X^{(\ell)}$ of the algorithm are bounded, and hence admit converging sub-sequences.

	Given that the energy $\bar \Ee$ is convex with respect to the variables $(\Sigma^{[r]})_{r \in R}$ and $X$ (although not jointly convex) and the non convex terms $J_{2,2}( G_{X^{[r]}} (X^{[r]} - \Sigma^{[r]} X) )$ that mixes the variables is $C^1$ with Lipschitz gradient, one can apply the Theorem 4.1 of~\cite{tseng-proximal}, which shows that any converging sub-sequence converges to a stationary point of $\bar\Ee_\rho$.   
\end{proof}

%%%%%%%%%%%%%%%%%%%%%%%%%%%%%%%%%%%%%%%%%%%%%%%%%%

\section{Application to color normalization}

Color normalization is the process of imposing the same color palette on a group of images. This color palette is always somehow related to the color palettes of the original images. For instance, if the goal is to cancel the illumination of a scene (avoid color cast), then the imposed histogram should be the histogram of the same scene illuminated with white light. Of course, in many occasions this information is not available. Following Papadakis et al.~\cite{Papadakis_ip11}, we define an in-between histogram, which is chosen here as the regularized OT barycenter. 

%The advantage with respect to Papadakis et al.'s method is that the influence of the input histograms on the barycenter can be easily tuned by a change of parameters, in fact it has as a special case any of the original image histograms. 

%Some other examples for which color normalization is useful are the color balancing of videos,  or as a preprocessing to register/compare several images taken with different cameras (see \cite{Papadakis_ip11} for more examples). 

%%%%%%%%%%%%%%%%%%%%%%%%%%%%%%%%%%%%%%%%%%%%%%%%%%%%%%%%%%%%%%%%%%
\subsection{Algorithm}

Given a set of input images $(X^{0[r]})_{r \in R}$, the goal is to impose on all the images the same histogram $\mu_X$ associated to the barycenter $X$. As for the colorization problem tackled in Section~\ref{sec-appli-color}, the first step is to subsample the original cloud of points $X^{0[r]}$ to make the problem tractable. Thus, for every $X^{0[r]}$ we compute a smaller associated point set $X^{[r]}$ using K-means clustering. Then, we obtain the barycenter $X$ of all the point clouds $(X^{[r]})_{r \in R}$ with the algorithms presented in Section~\ref{algobar}. Figure~\ref{im:baryillu} first row, shows an example on two synthetic cloud of points, $X^{[1]}$ in blue and $X^{[2]}$ in red. The cloud of points in green corresponds to the barycenter $X$, which can change its position depending on the parameter $\rho=(\rho_1,\rho_2)$ in~\eqref{eqbar} from $X^{[1]}$ for $\rho=(1,0)$ to $X^{[2]}$ when $\rho=(0,1)$.  This data set $X$ represents the 3-D histogram we want to impose on all the input images. 

Once we have $X$, we compute the regularized and relaxed OT transport maps $T^{[r]}$ between each $X^{[r]}$ and the barycenter $X$, by solving~\eqref{eq-symm-reg-energy}. The line segments in Figure~\ref{im:baryillu} represent the transport between points clouds, i.e. if $\Sig^{[1]}_{i,j}>0$, $X^{[1]}_i$ is linked to $X_j$, and similarly for $\Sig^{[2]}$. 

We apply $T^{[r]}$ to $X^{[r]}$, obtaining  $\tilde X^{[r]}$, for all $r \in R$, that is to say, we obtain a set of point clouds $\tilde X^{[r]}$ with a color distribution close to $X$.  Finally, to recover a set of high resolution images, we compute each $\tilde X^{0[r]}$ from $X^{0[r]}$ by up-sampling. A detailed description of the method is given in Algorithm~\ref{alg-norm}. 

\begin{algorithm}[ht!]
\caption{Regularized OT Color Normalization}
\label{alg-norm}
% \begin{algorithmic}[1]
\Require Images $\left( X^{0[r]} \right)_{r \in R} \in \RR^{N_0 \times d}$, $\la \in \RR^+ $, $\rho \in [0,1]^{|R|}$ and $k \in \RR^+$.

\Ensure Images $\left( \tilde X^{0[r]} \right)_{r \in R} \in \RR^{N_0 \times d}$.
% \Statex
\begin{enumerate}
	\algostep{Histogram down-sample} Compute $X^{[r]}$ from $X^{0[r]}$ using K-means clustering.
	\algostep{Compute barycenter} Compute with either~\eqref{eq-bar-l2} or~\eqref{eq-bar-l1} a barycenter $\mu_X$ where $X$ is a local minimum of~\eqref{eqbar} using the block coordinate descent described in Section~\ref{algobar}, see Algorithm~\ref{algo-block-barycenters}. 
	\algostep{Compute transport mappings} For all $r \in R$ compute $T^{[r]}$ between \\ 
		$X$ and $X^{[r]}$ by solving~\eqref{eq-symm-reg-energy}, such that $T^{[r]}(X^{[r]}_i) = Z^{[r]}_i$, where $Z^{[r]}= \diag(\Sig^{[r]} \U)^{-1} \Sig^{[r]} X$.
	\algostep{Transport up-sample} For every $T^{[r]}$ compute $\tilde T^{0[r]}$ following~\eqref{eq-upsample}.
	\algostep{Obtain high resolution results} Compute $\foralls r, \tilde X^{0[r]} = \tilde T^{0[r]}(X^{0[r]})$.
\end{enumerate}
% \end{algorithmic}
\end{algorithm}

%%%%%%%%%%%%%%%%%%%%%%%%%%%%%%%%%%%%%%%%%%%%%%%%%%%%%%%%%%%%%%%%%%
\subsection{Results}

We now show some example of color normalization using Algorithm~\ref{alg-norm}.

\begin{figure*}[!h]
\centering
\setlength{\arrayrulewidth}{2pt}
%%% OT %%%%
\begin{tabular}{@{}c@{\hspace{1mm}}c@{\hspace{1mm}}c@{\hspace{1mm}}c@{}}
	$\rho=(1,0)$ & $\rho=(0.7,0.3)$ & $\rho=(0.4,0.6)$ & $\rho=(0,1)$ \\
\includegraphics[width=.23\linewidth]{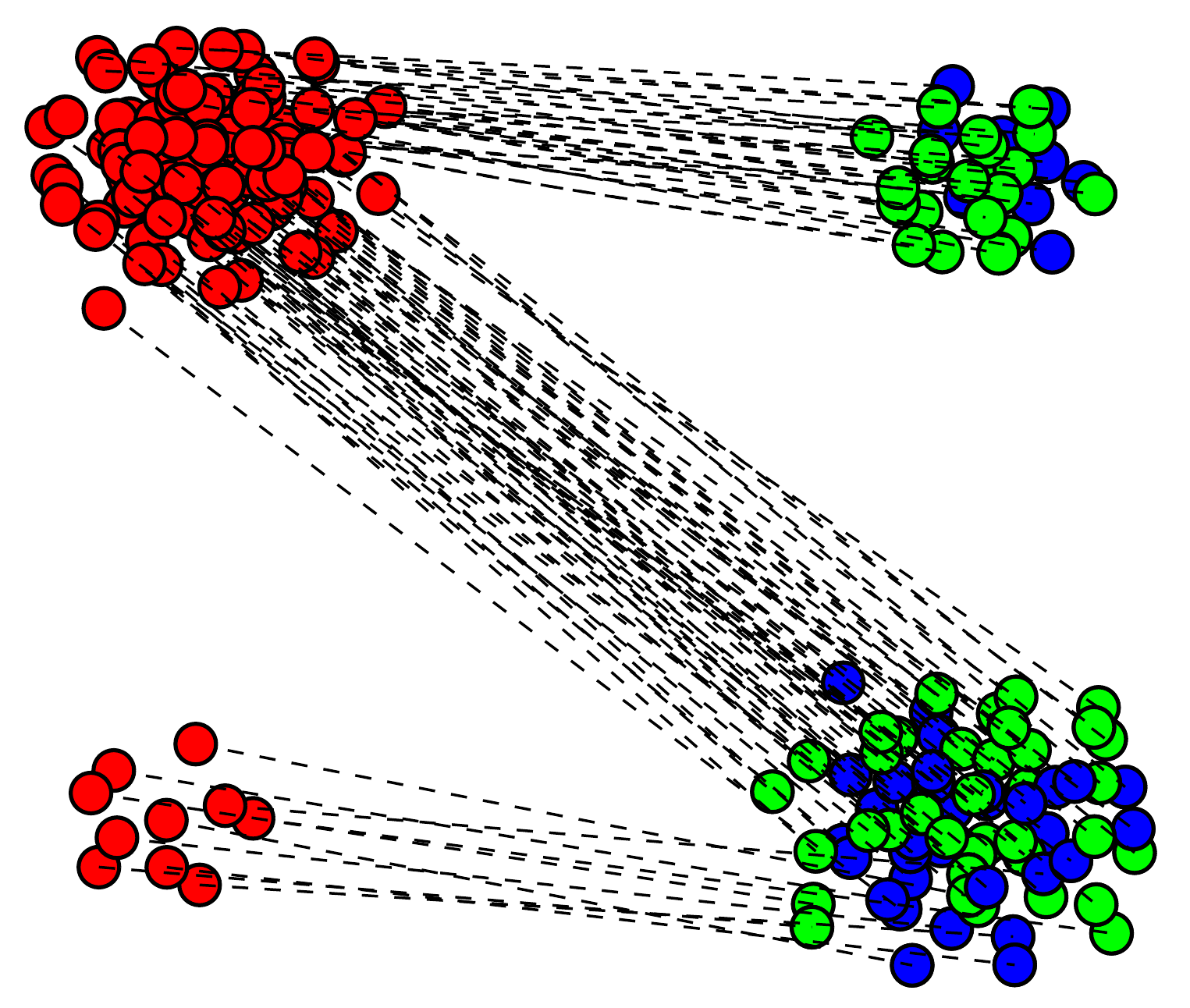} &  
\includegraphics[width=.23\linewidth]{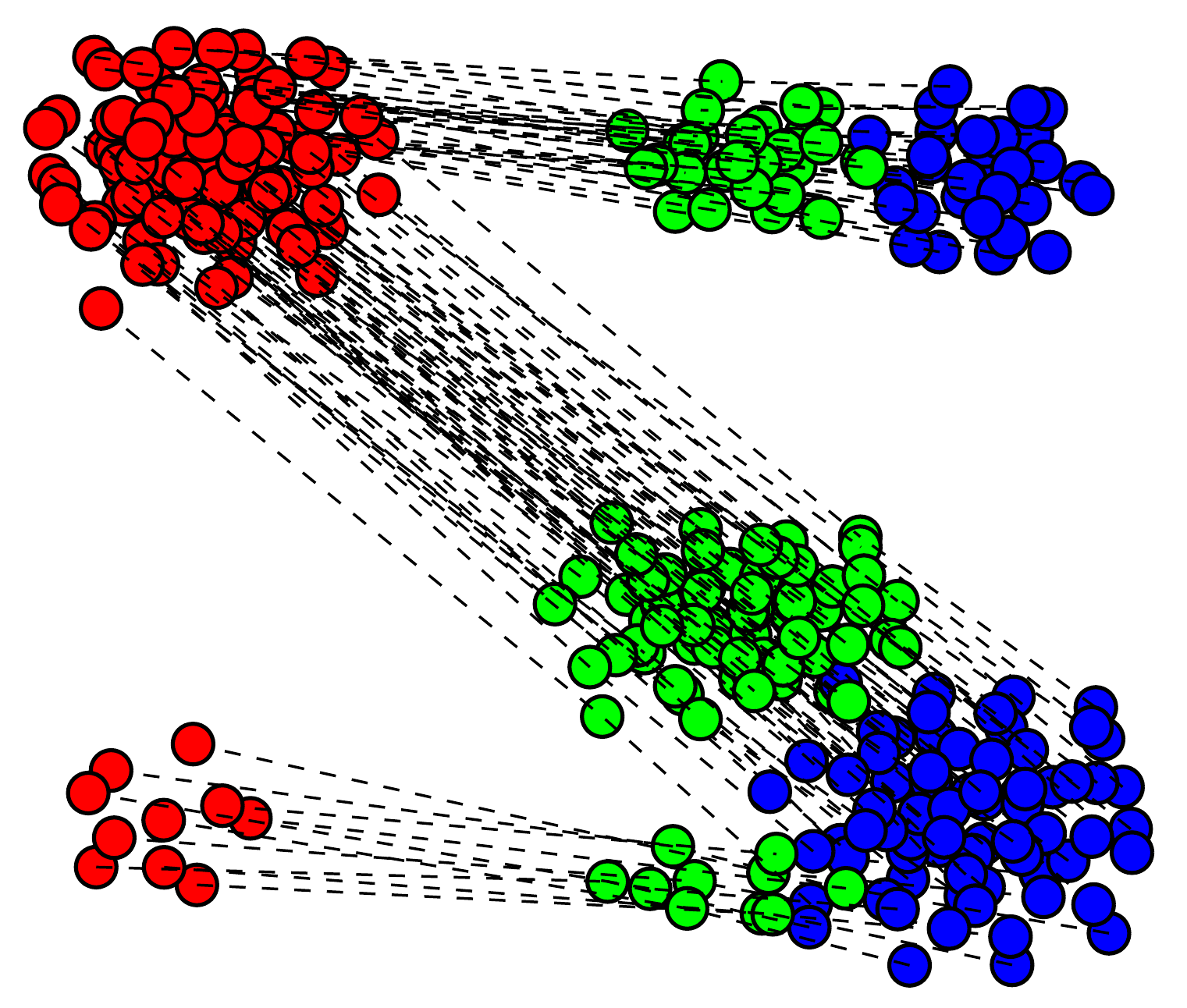} &
\includegraphics[width=.23\linewidth]{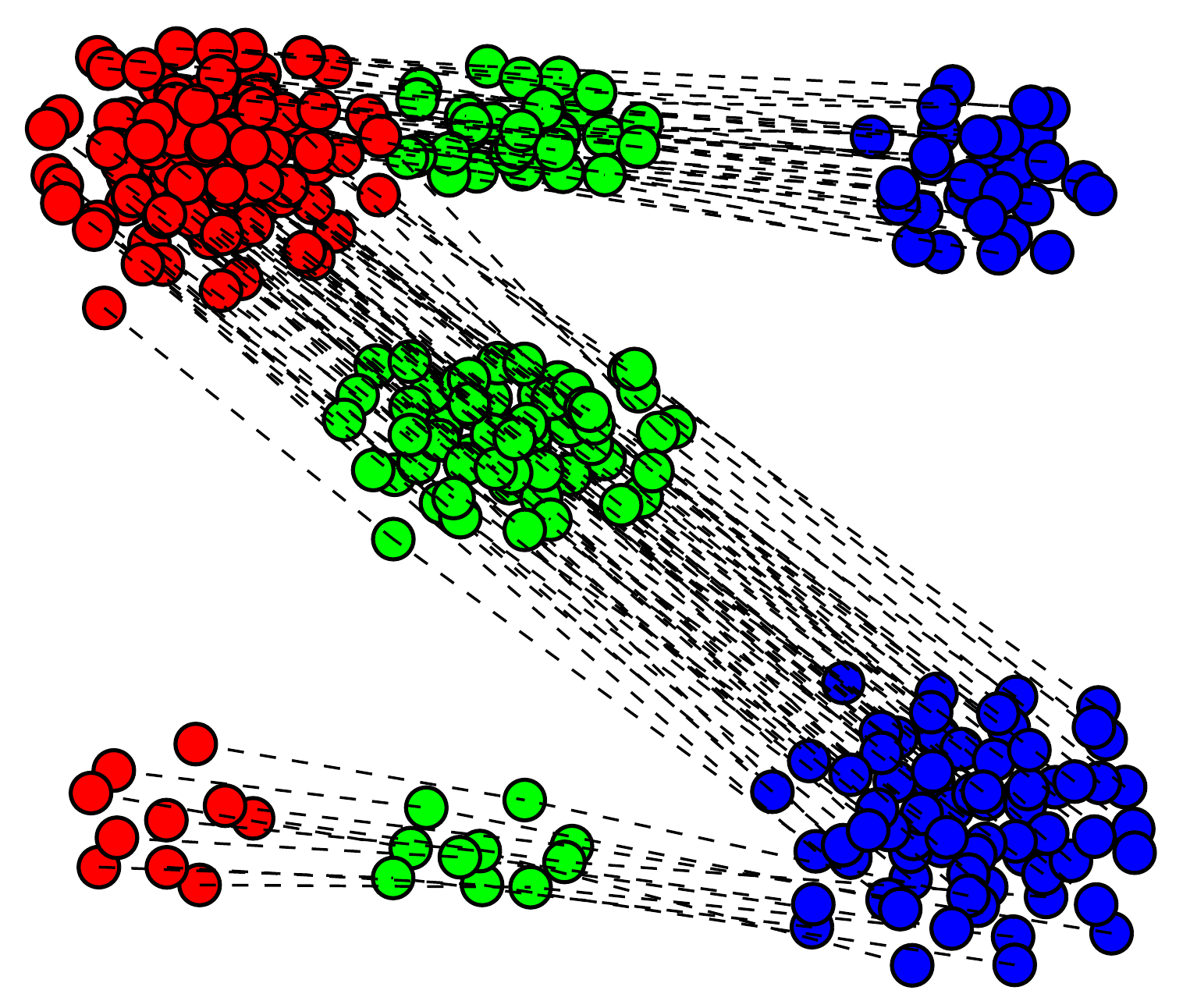} &
\includegraphics[width=.23\linewidth]{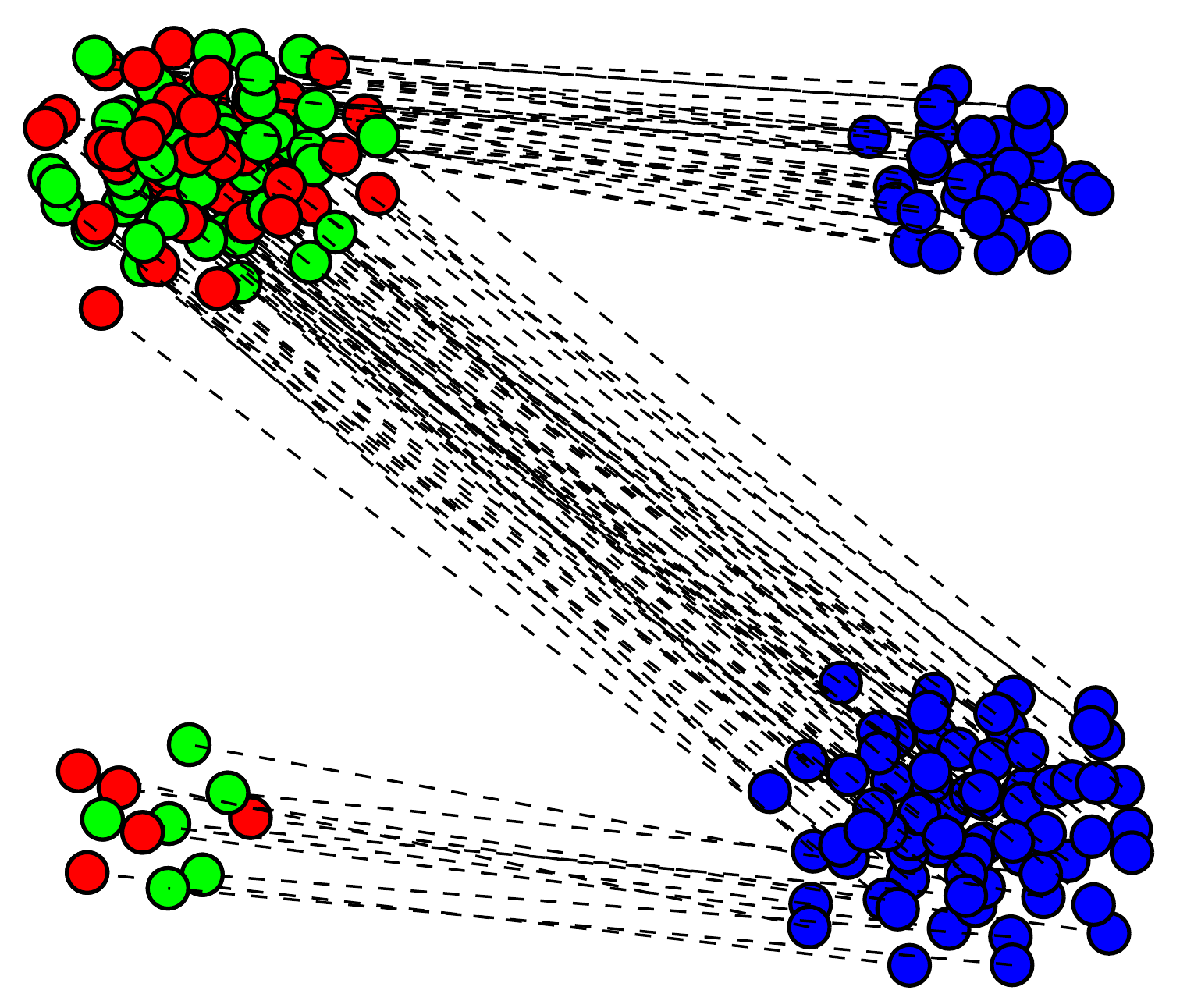} \\
\includegraphics[width=.23\linewidth,height=.2\linewidth]{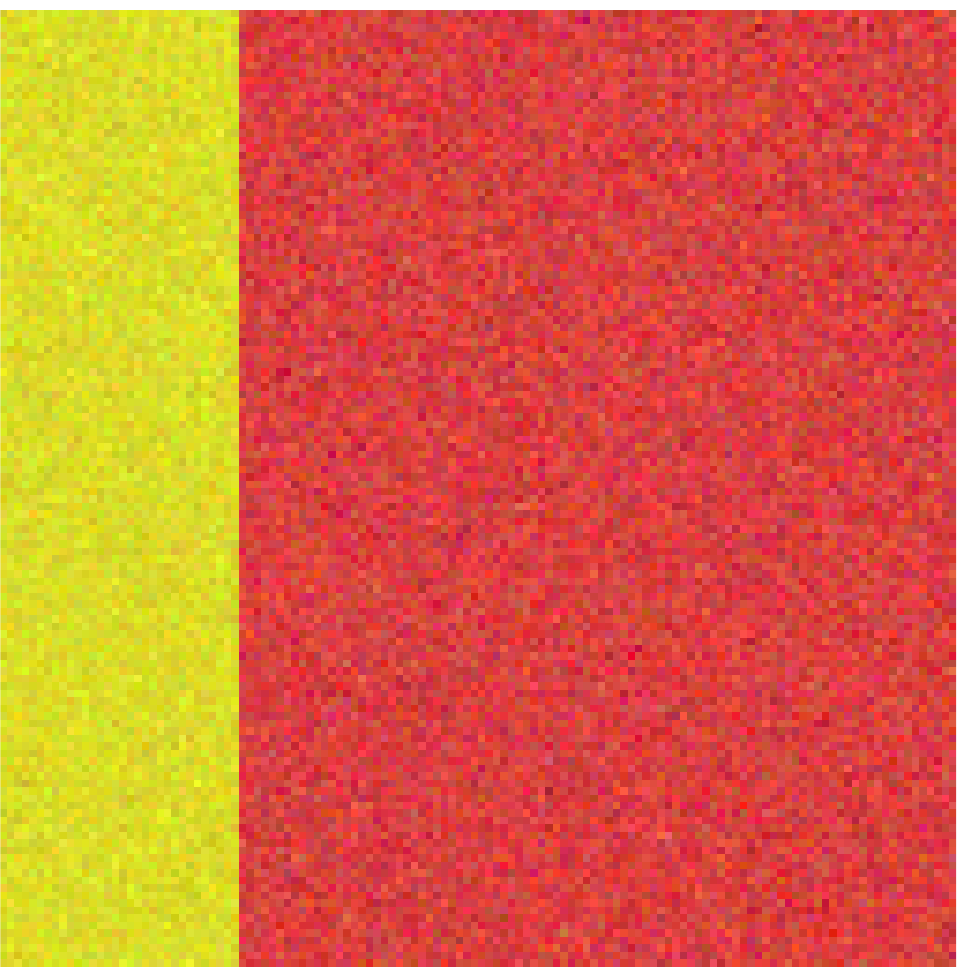} &
\includegraphics[width=.23\linewidth,height=.2\linewidth]{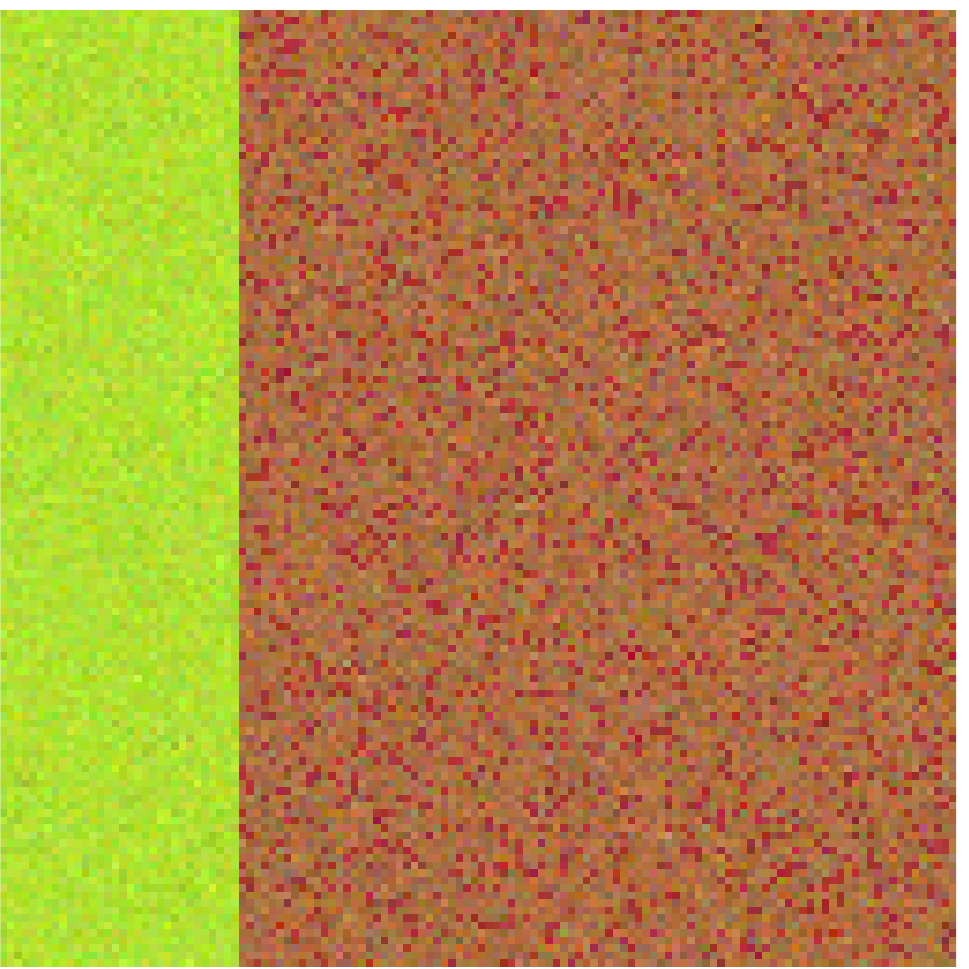} &
\includegraphics[width=.23\linewidth,height=.2\linewidth]{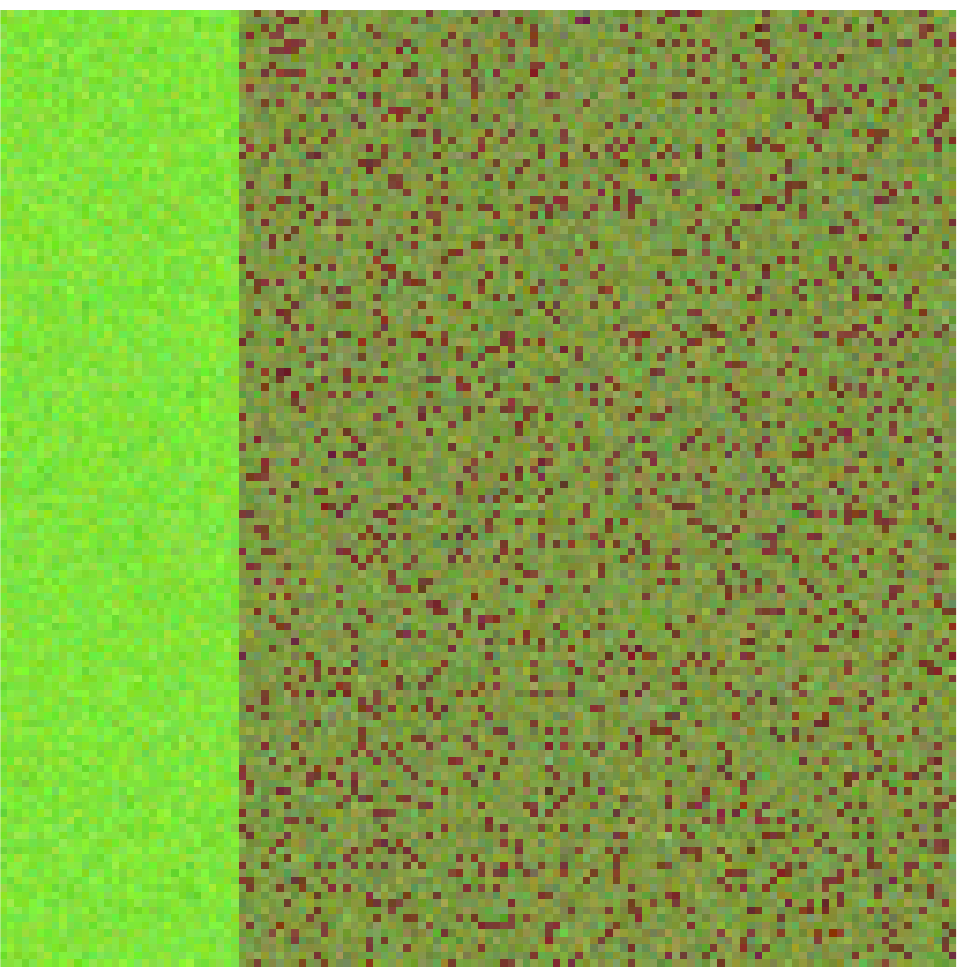} &
\includegraphics[width=.23\linewidth,height=.2\linewidth]{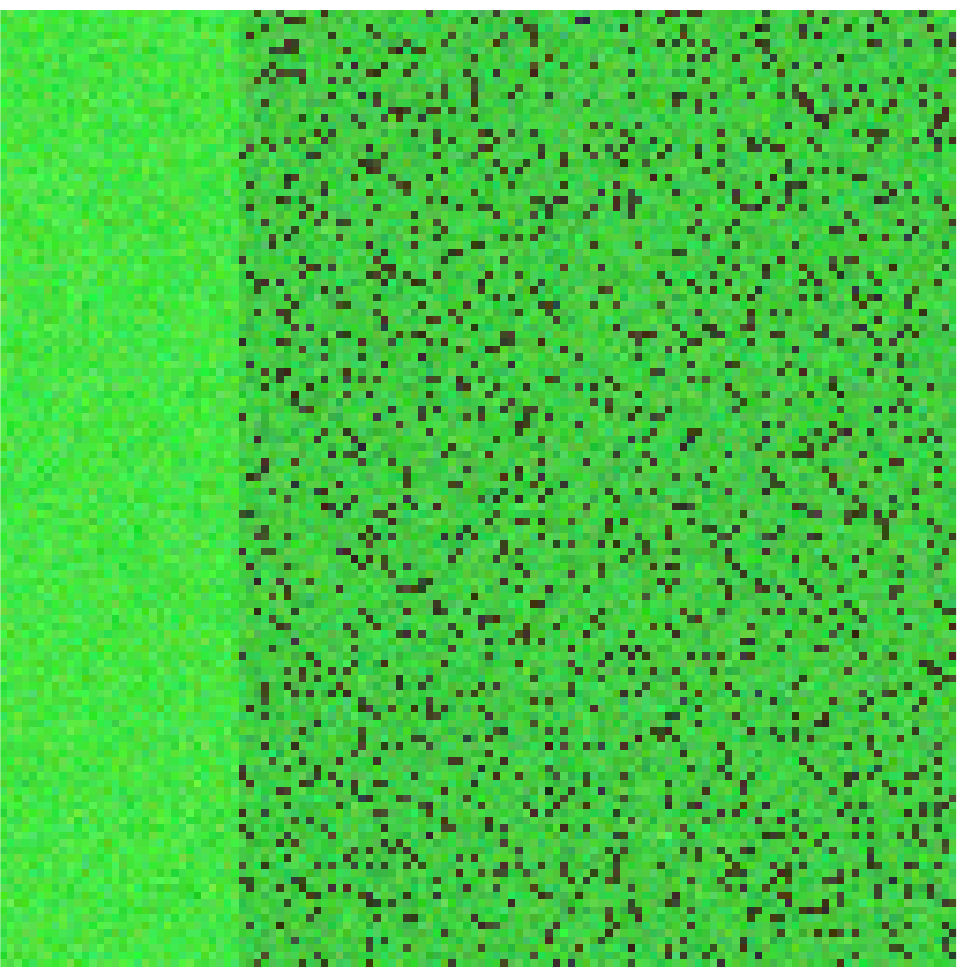}\\
\includegraphics[width=.23\linewidth,height=.2\linewidth]{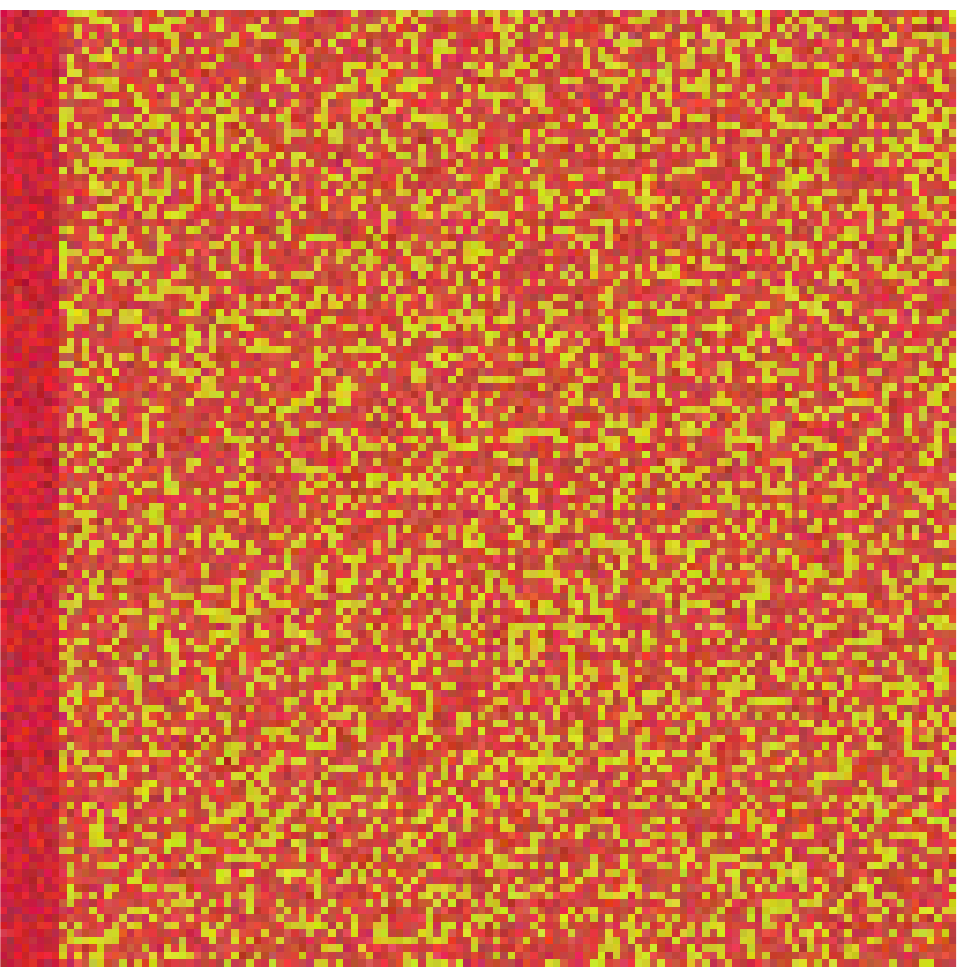} &
\includegraphics[width=.23\linewidth,height=.2\linewidth]{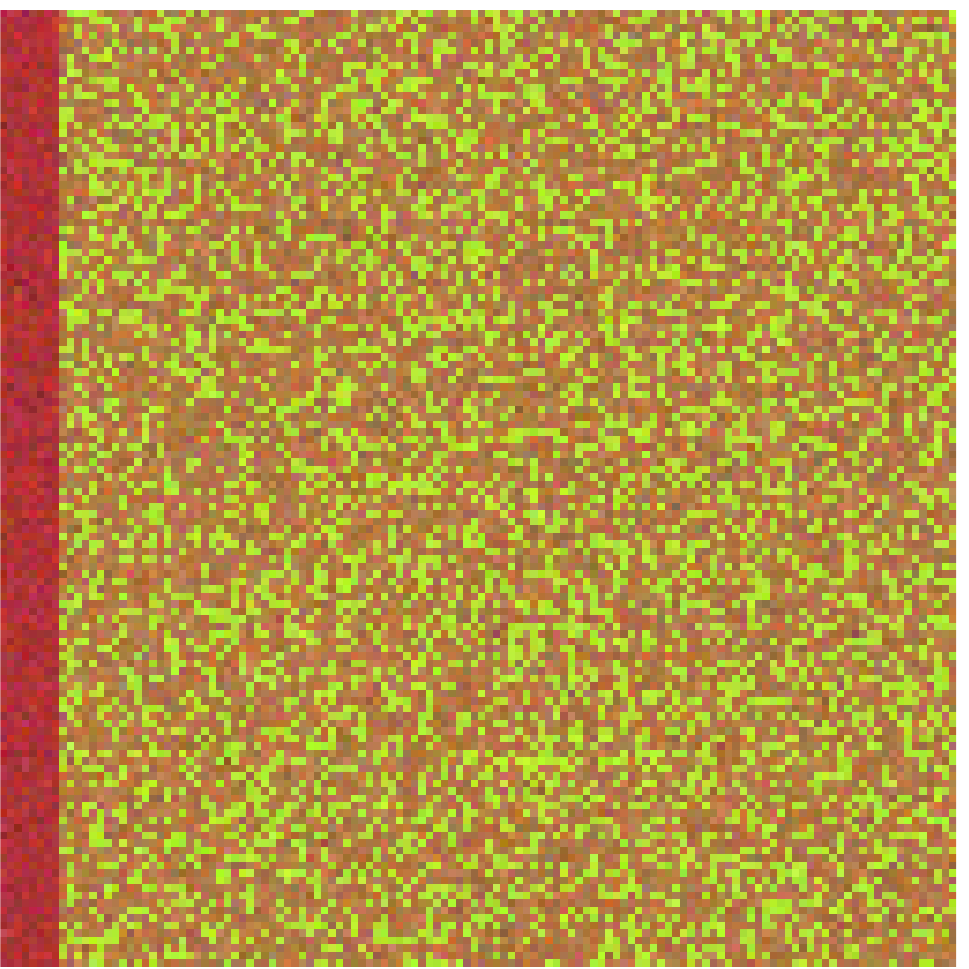} &
\includegraphics[width=.23\linewidth,height=.2\linewidth]{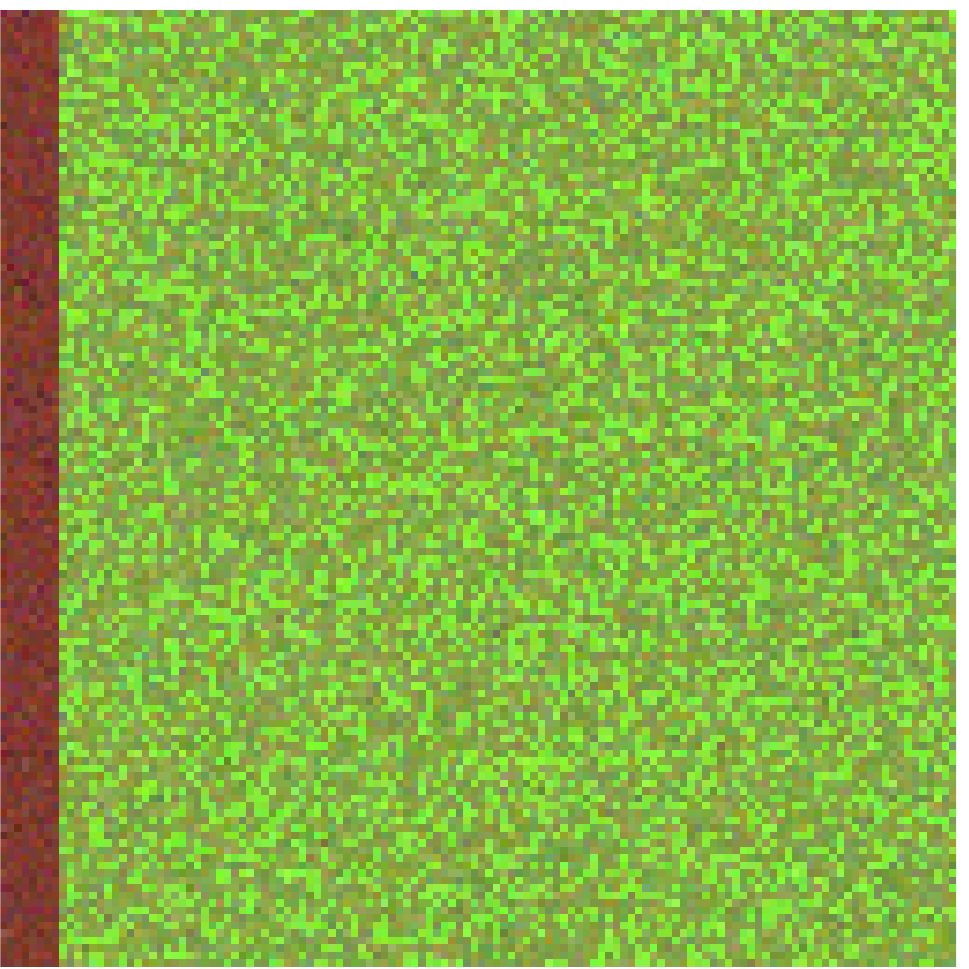} &
\includegraphics[width=.23\linewidth,height=.2\linewidth]{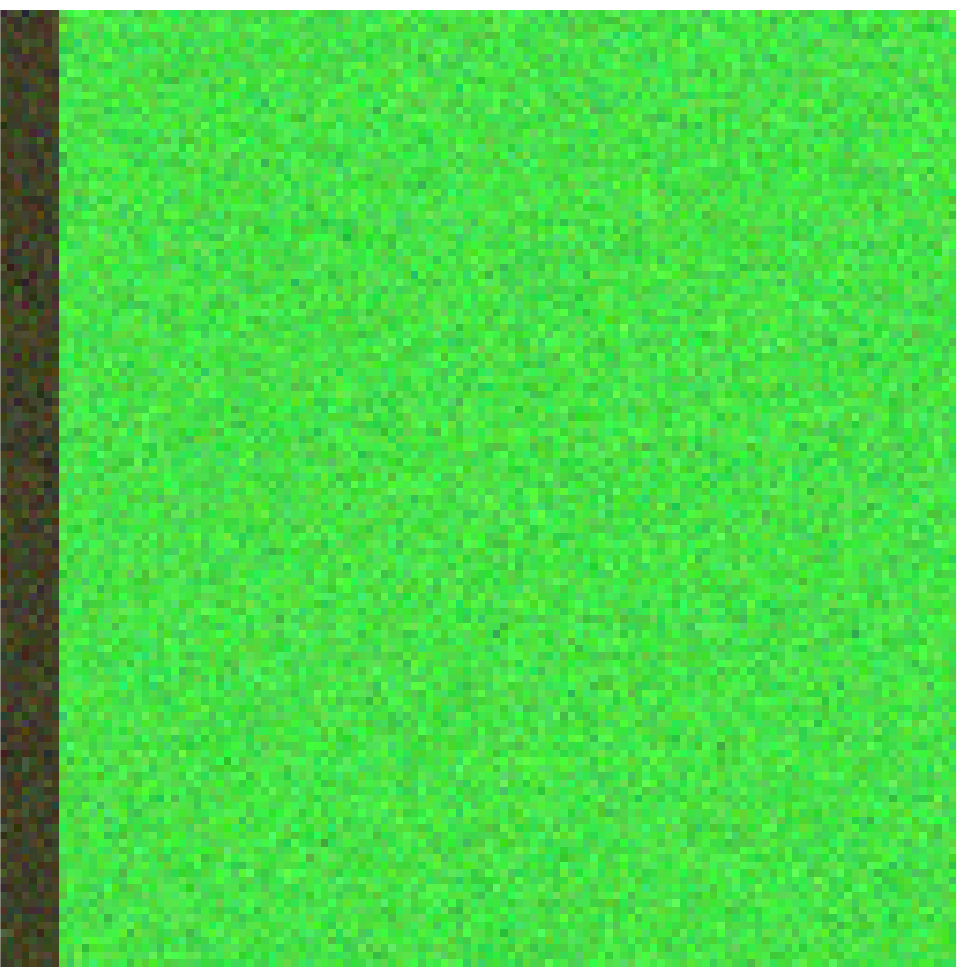} \\\hline
%%% Regularized
\includegraphics[width=.23\linewidth]{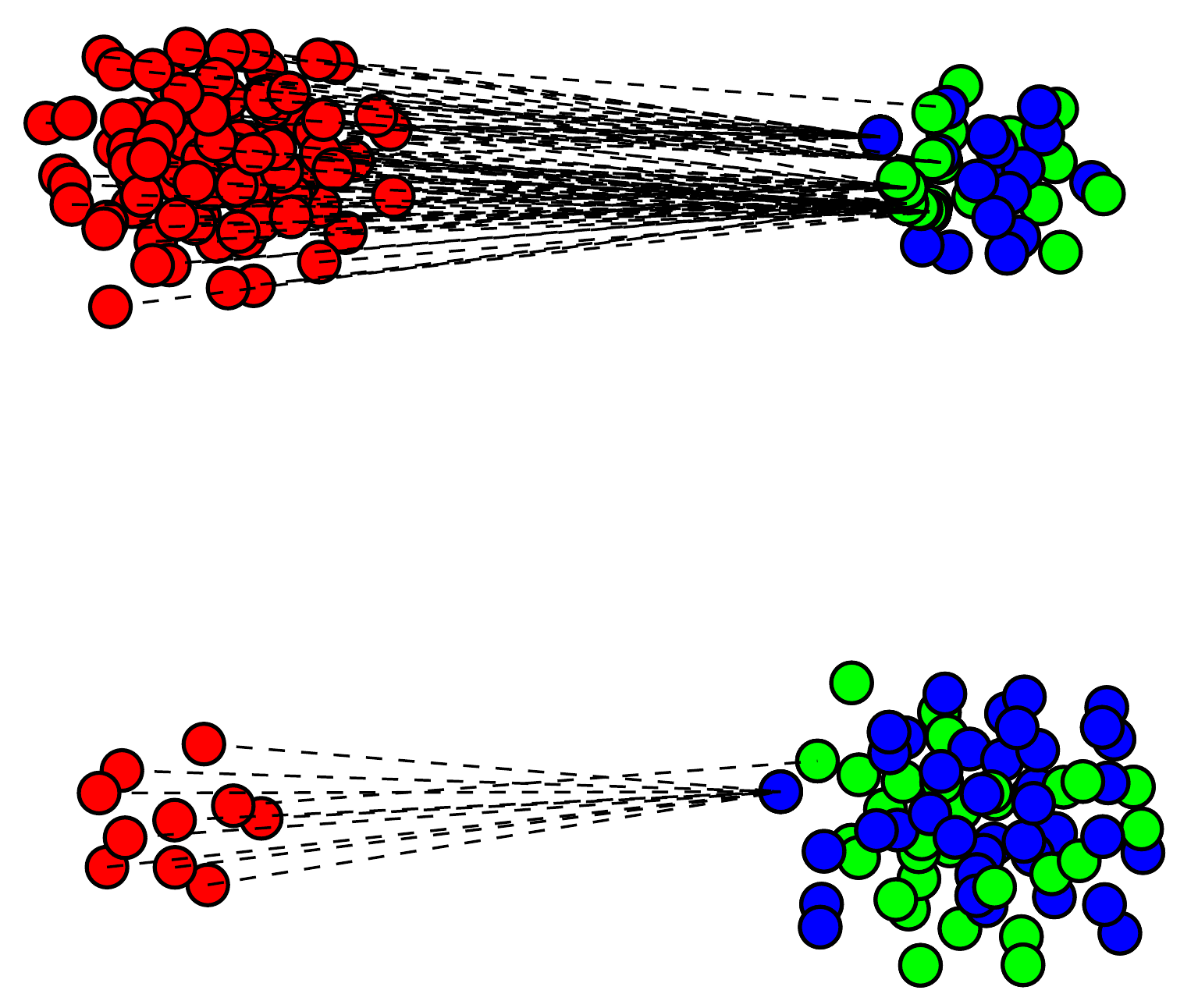} &
\includegraphics[width=.23\linewidth]{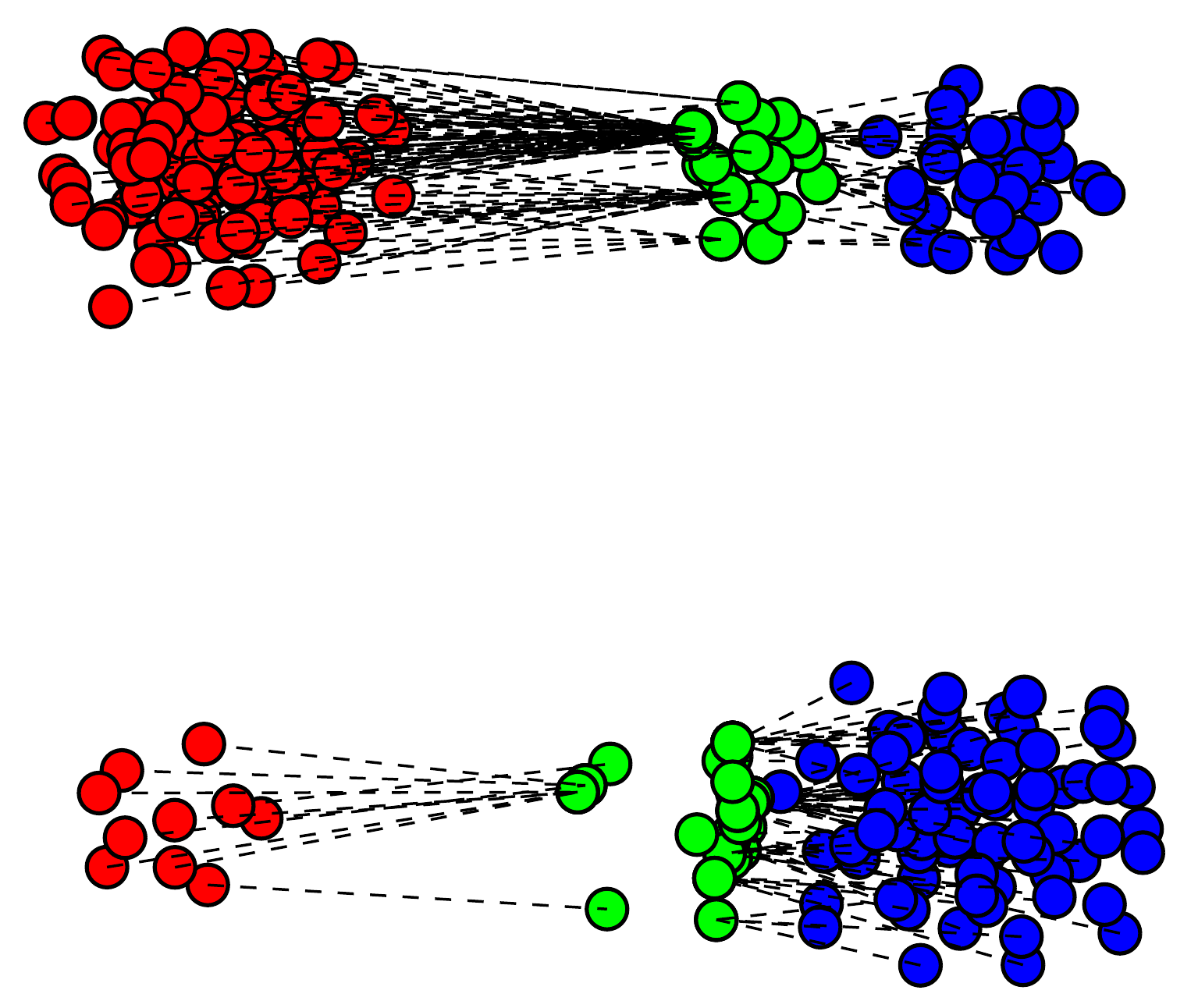} & 
\includegraphics[width=.23\linewidth]{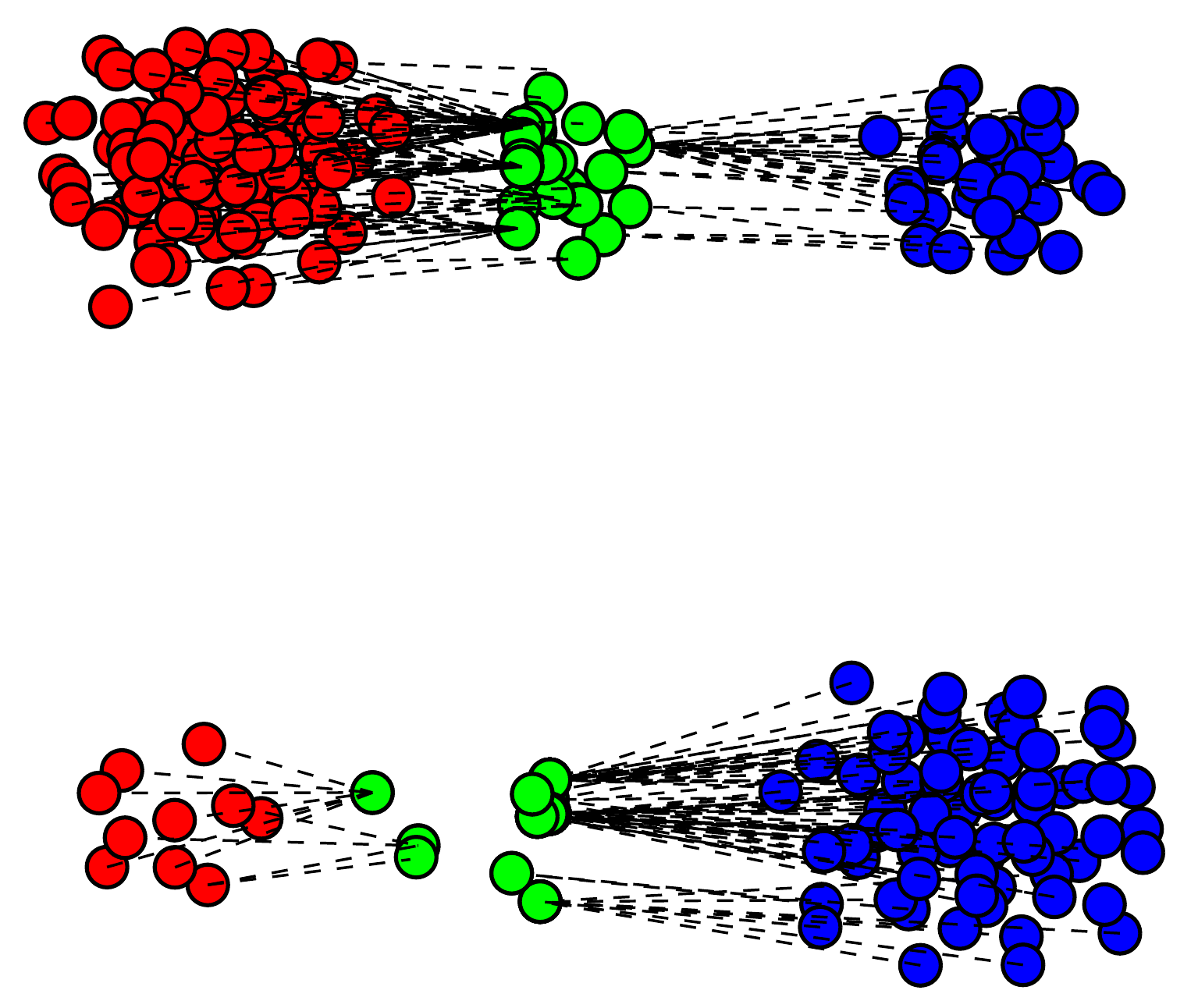} &
\includegraphics[width=.23\linewidth]{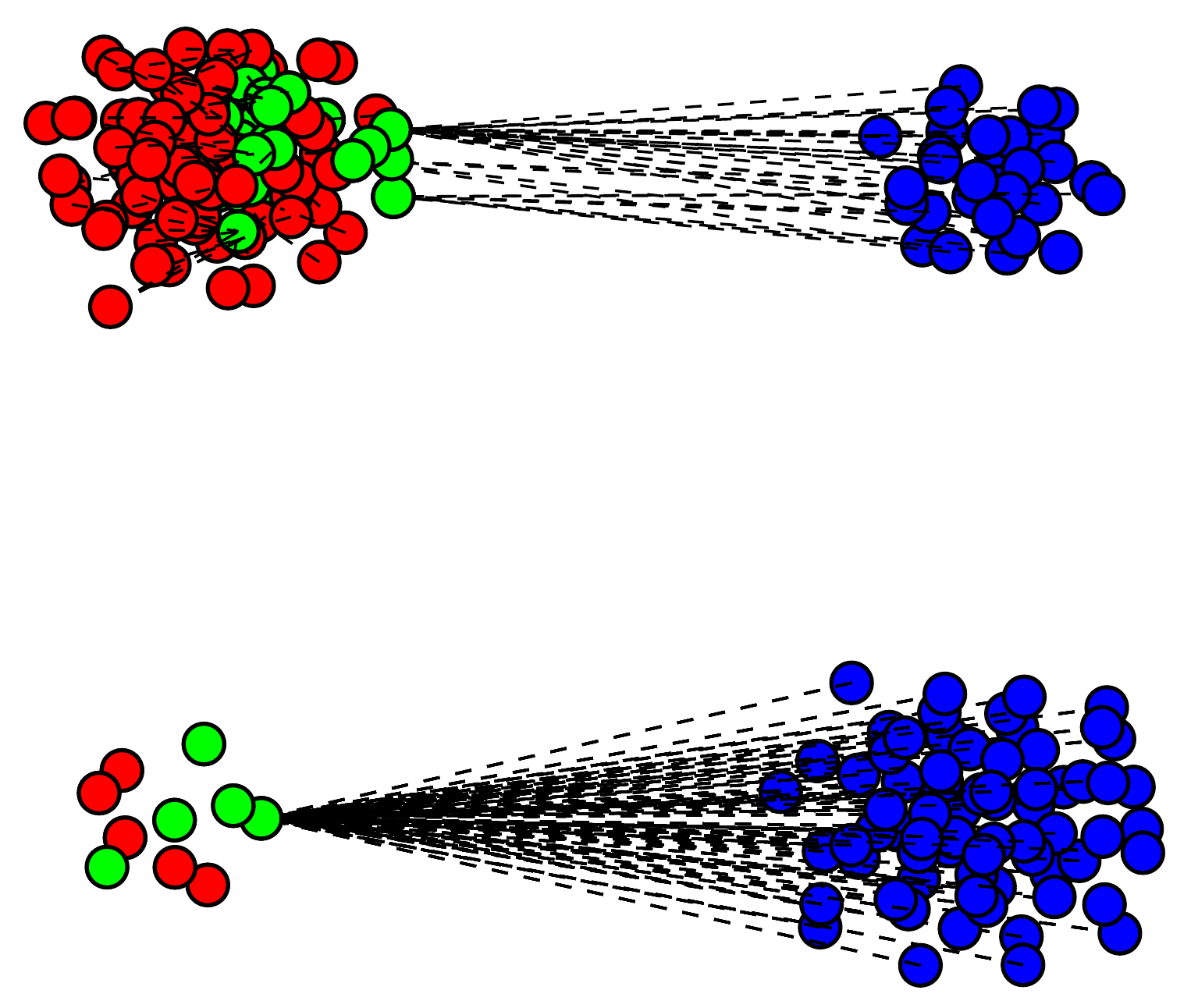} \\
\includegraphics[width=.23\linewidth,height=.2\linewidth]{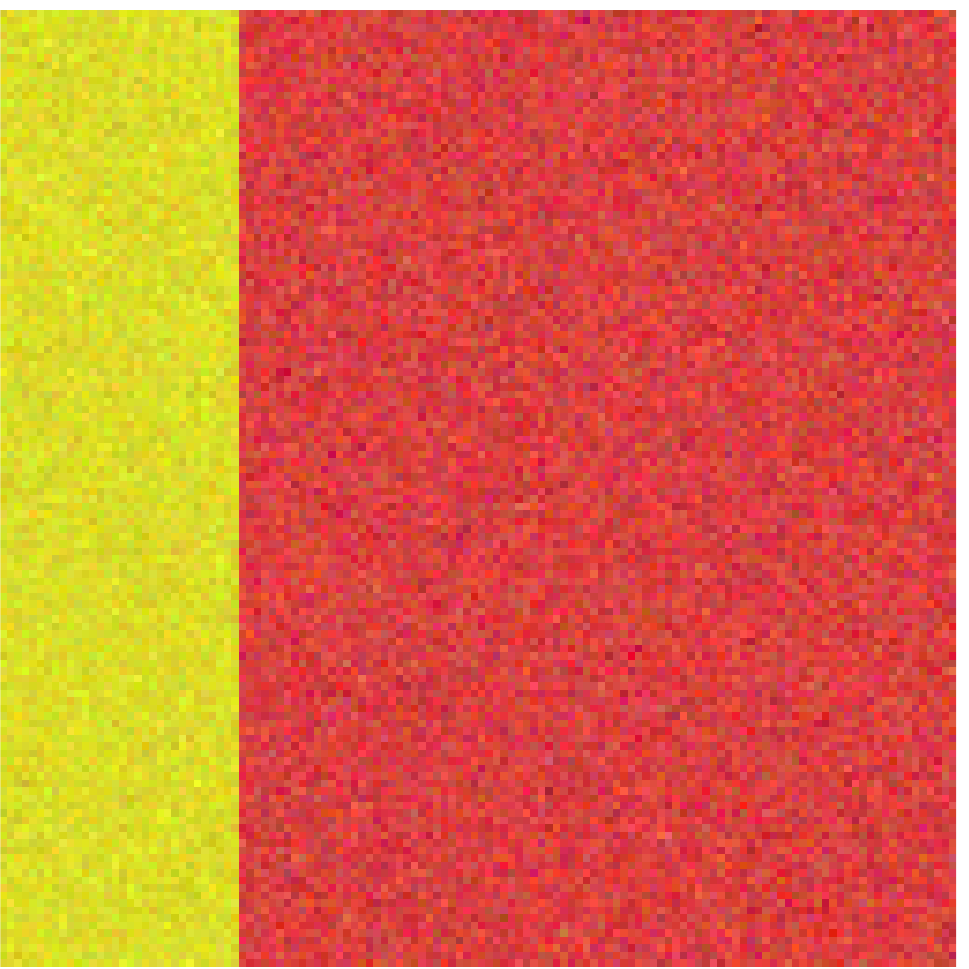} &
\includegraphics[width=.23\linewidth,height=.2\linewidth]{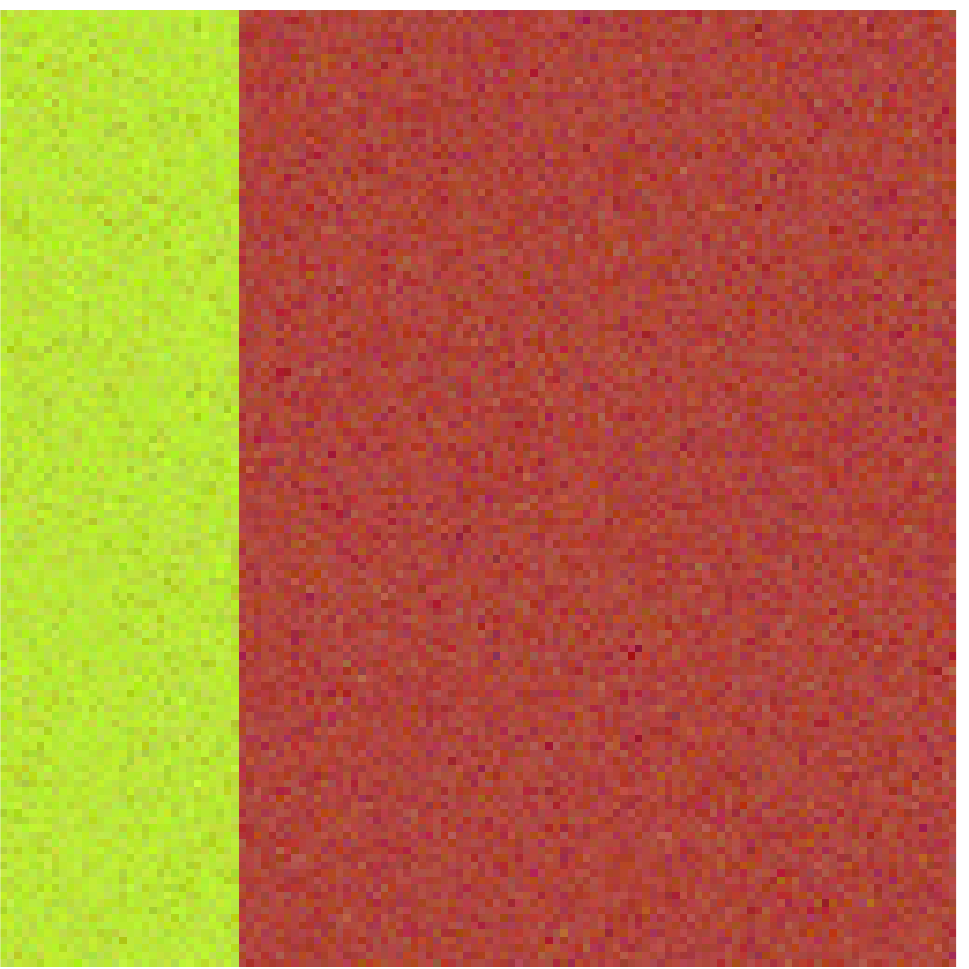}& 
\includegraphics[width=.23\linewidth,height=.2\linewidth]{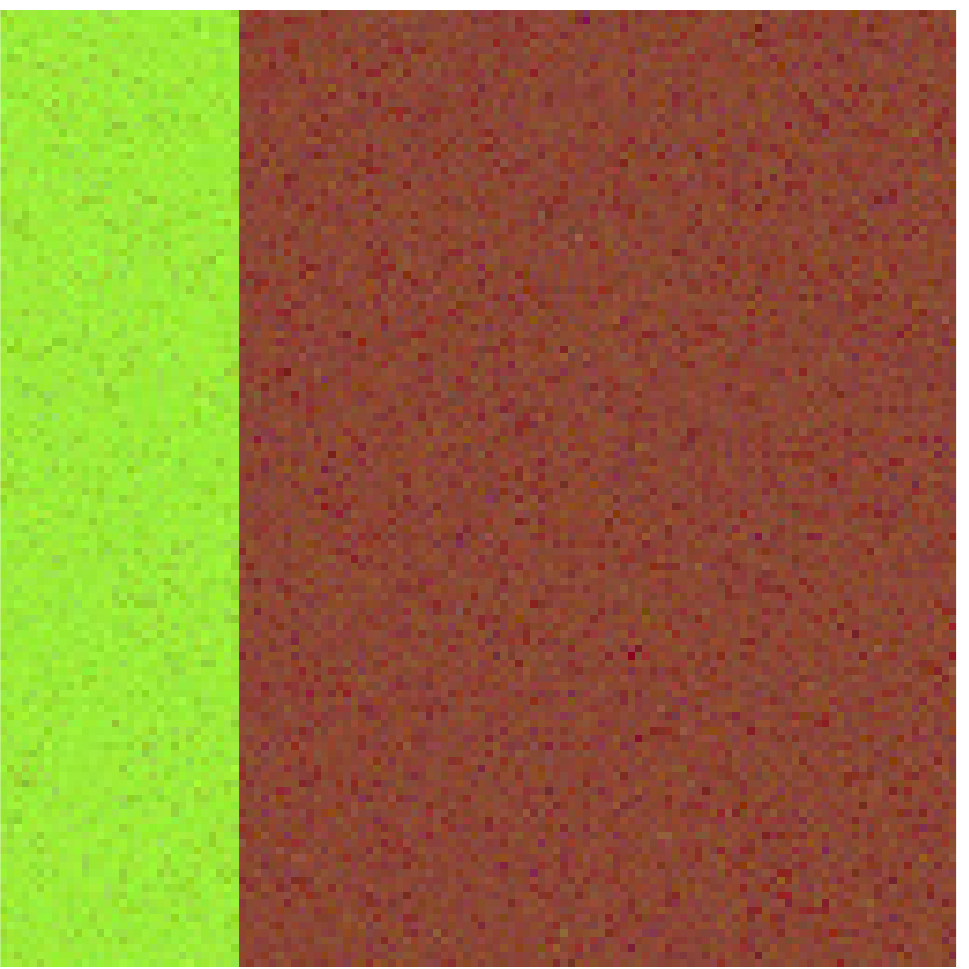} &
\includegraphics[width=.23\linewidth,height=.2\linewidth]{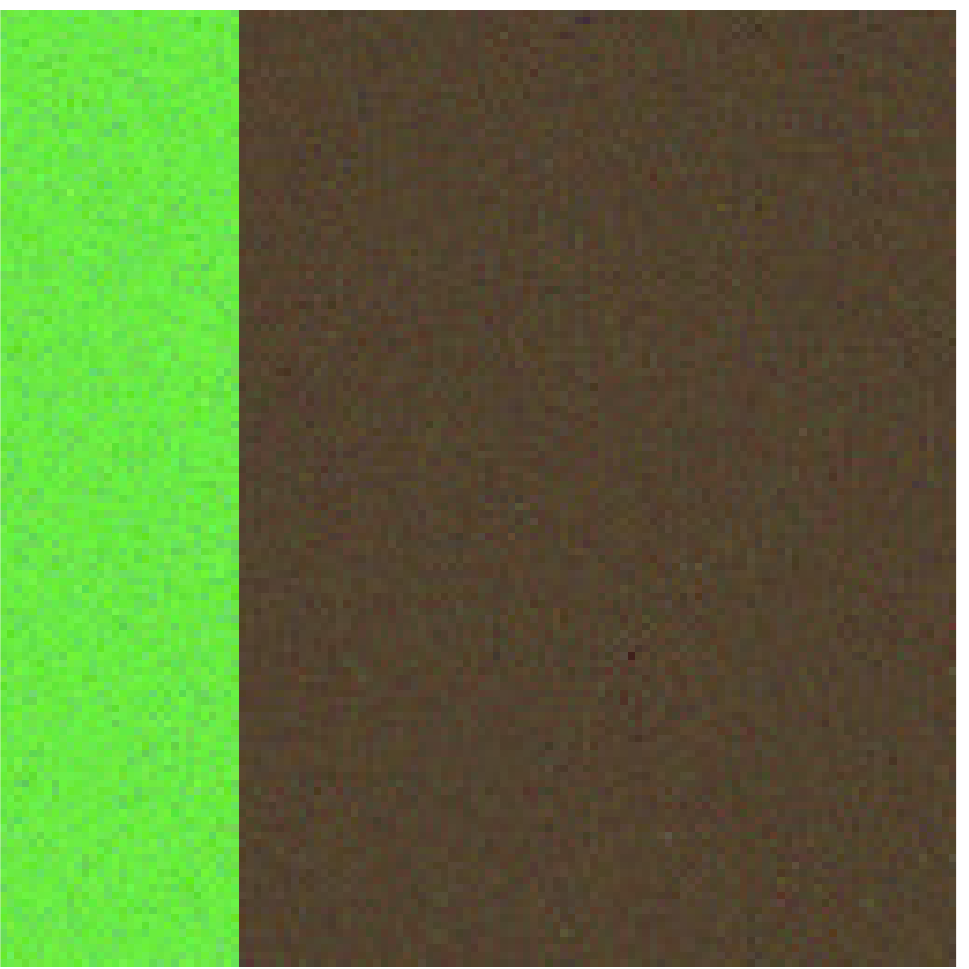}\\
\includegraphics[width=.23\linewidth,height=.2\linewidth]{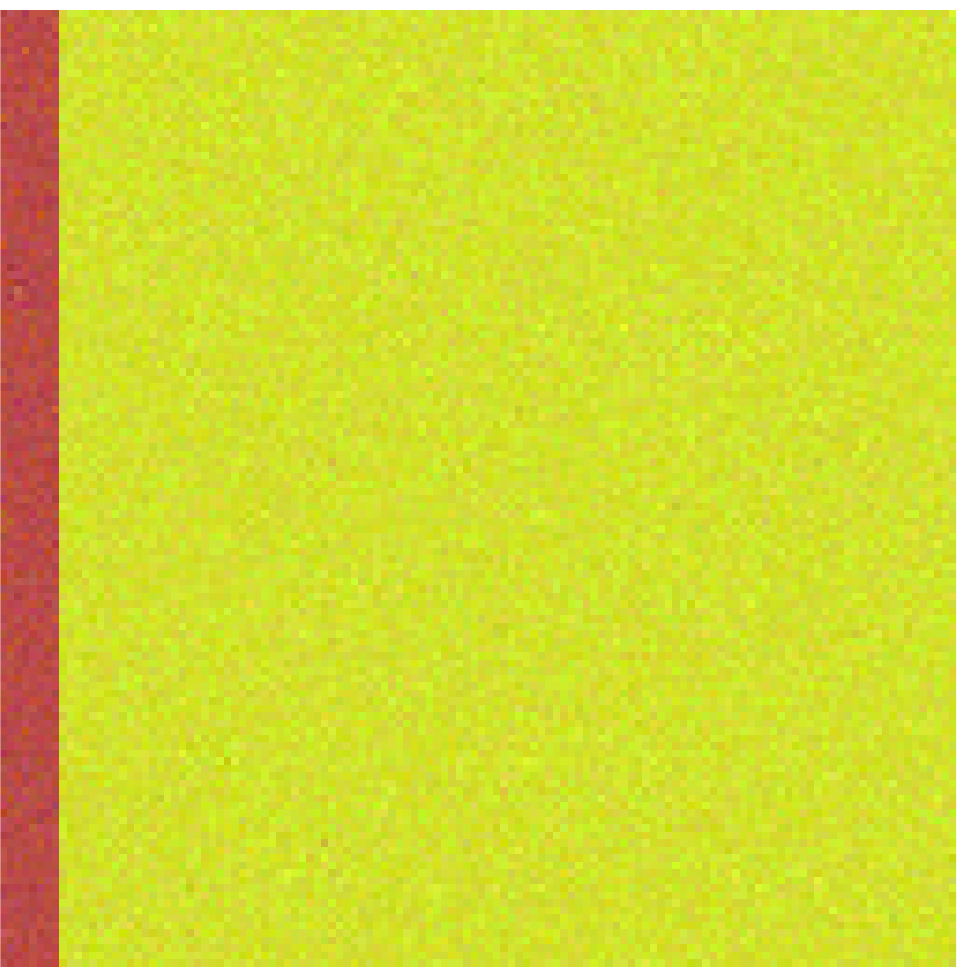} &
\includegraphics[width=.23\linewidth,height=.2\linewidth]{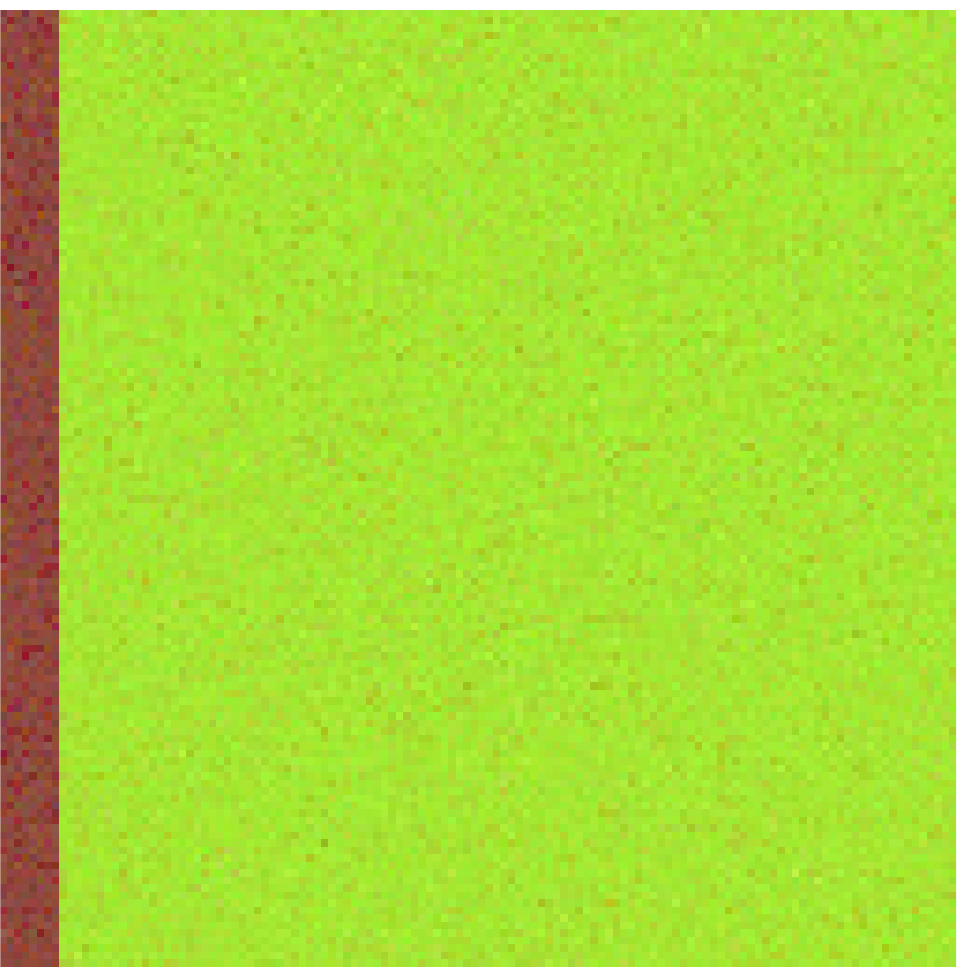} & 
\includegraphics[width=.23\linewidth,height=.2\linewidth]{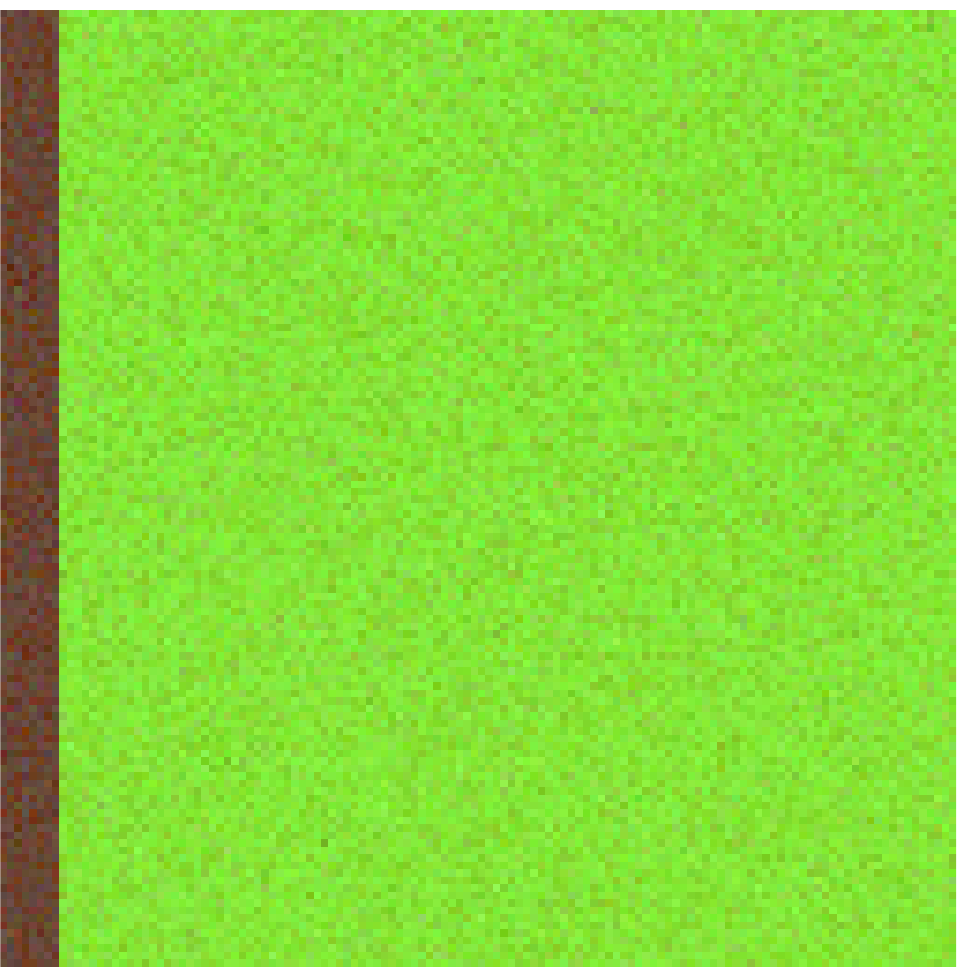} &
\includegraphics[width=.23\linewidth,height=.2\linewidth]{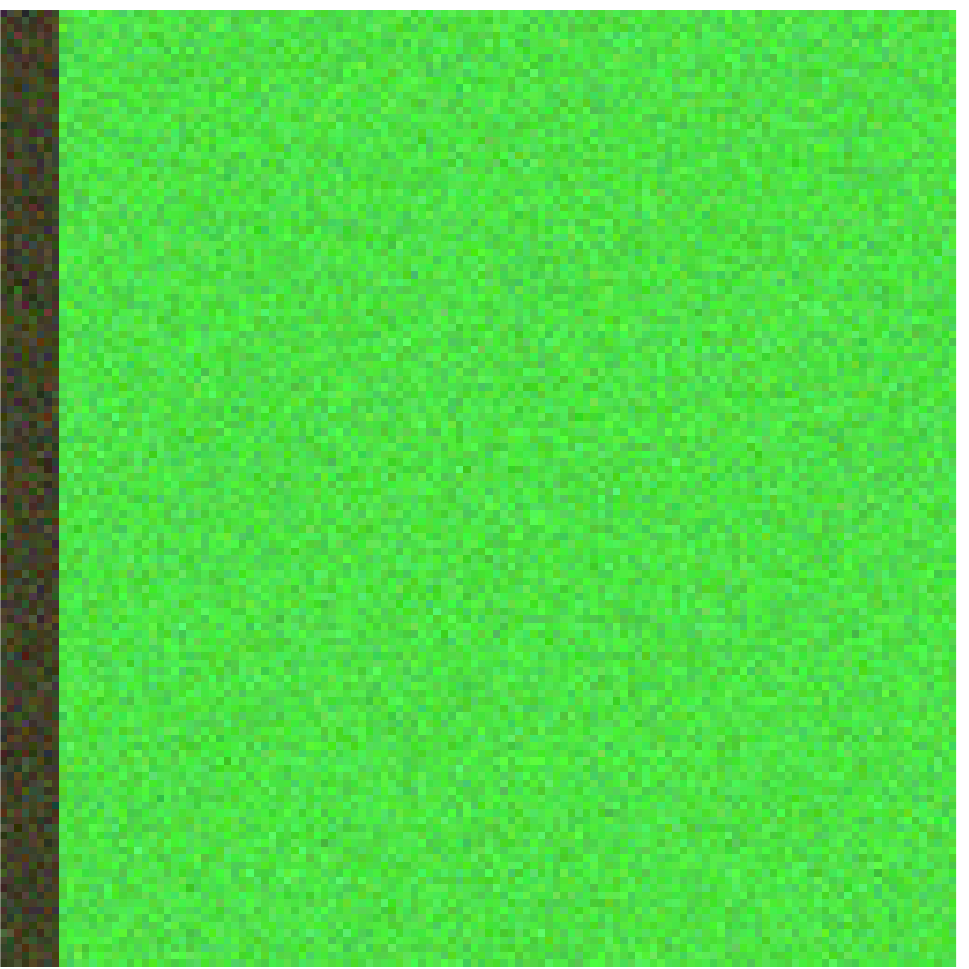} \\
\end{tabular}
 \caption{Comparison of classical OT (top 3 first rows) and relaxed/regularized OT (bottom 3 last rows). The original input images $X^{0,[1]}$ and $X^{0,[2]}$ are shown in Figure~\ref{im:synth}~(a).  
 		Rows \#1 and \#4 shows the 2-D projections of $X^{[1]}$ (blue) and $X^{[2]}$ (red), and in green the barycenter distribution for different values of $\rho$. We display a line between $X^{[r]}_i$ and $X_j$ if $\Sig^{[r]}_{i,j} > 0.1$. 
		Rows \#2 and \#5 (resp. \#3 and \#6) show the resulting normalized images $\tilde X^{0[1]}$ (resp. $\tilde X^{0[2]}$), 
		for each value of $\rho$.  
 	\textbf{Top 3 first rows:} classical OT corresponding to setting $k=1$ and $\la=0$. 
	\textbf{Bottom 3 last rows:}  regularized and relaxed OT, with parameters $k=20$ and $\lambda=0.0005$. See main text for comments. \vspace{0.5cm}
}
\label{im:baryillu}
\end{figure*}

%%%%%%%%%%%%%%%%%%
\paragraph{Synthetic example}

Figure~\ref{im:baryillu} shows a comparison of normalization of two synthetic images using classical OT and our proposed relaxed/regularized OT. The results obtained using Algorithm~\ref{alg-norm} (setting $p=q=2$), using the set of two images ($|R|=2$) already used in Figure~\ref{im:synth}~(a), denoting here $X^{0[1]}=X^0$ and $X^{0[2]}=Y^0$. Each column shows the same experiment but with different values of $\rho$, which allows to visualize the interpolation between the color palettes (the colors in the images  evolve from the colors in $X^{[1]}$ towards the colors of $X^{[2]}$). 

With classical OT, the structure of the original data sets in not preserved as we change $\rho$, and the consequence on the final images (second and third row), is that the geometry of the original images changes in the barycenters. In contrast to classical OT, for all values of $\rho$ the relaxed/regularized barycenters $X$ have the same number of clusters of the original sets. Note that the consequence of having a transport that maintains the clusters of the original images, is that the geometry is preserved, while the histograms change.

\newcommand{\sidecapY}[1]{ \begin{sideways}\parbox{.19\linewidth}{\centering #1}\end{sideways} }
\newcommand{\myimgY}[1]{\includegraphics[width=.26\linewidth,height=.22\linewidth]{#1}}

\begin{figure*}[!h]
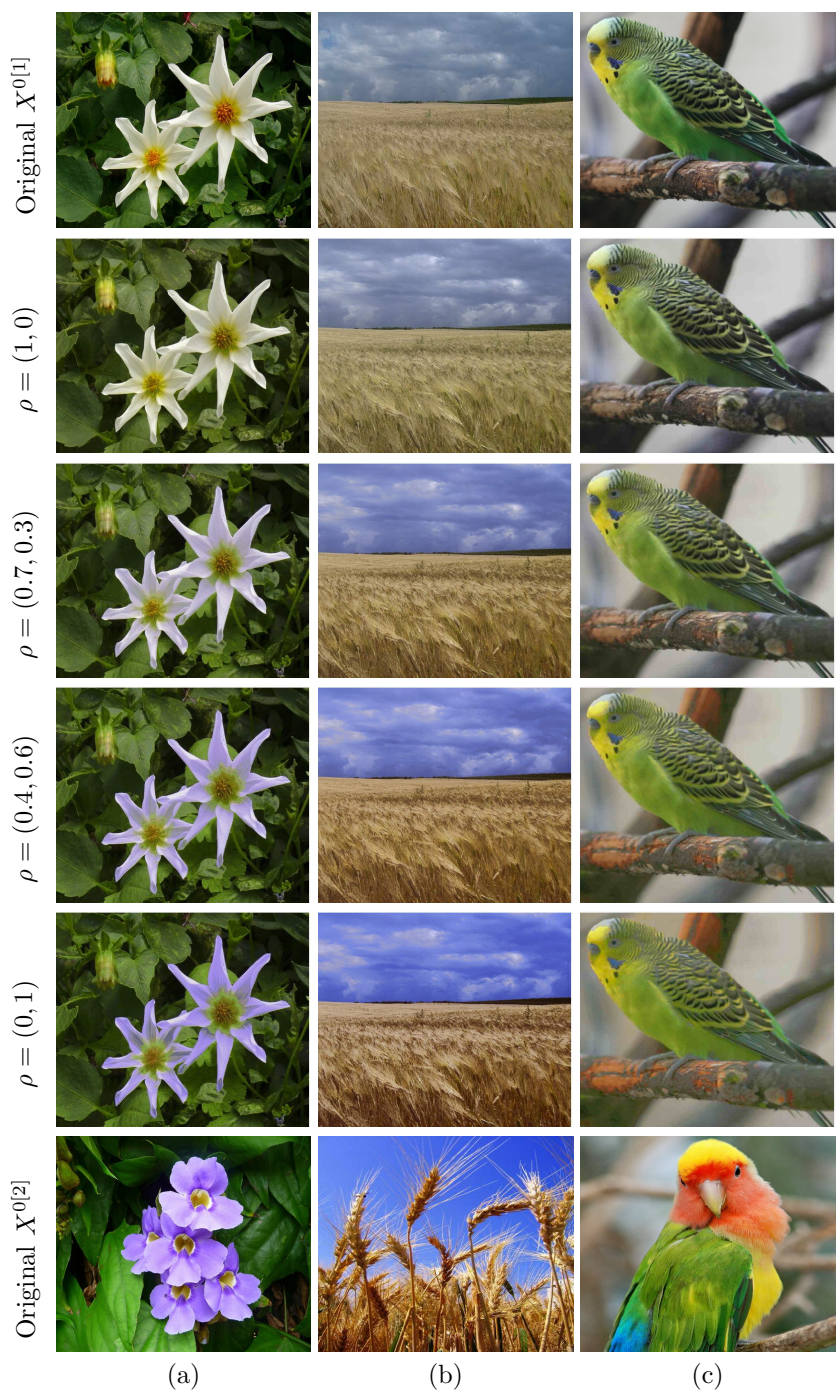

\centering
\begin{tabular}{@{}c@{\hspace{1mm}}c@{\hspace{1mm}}c@{\hspace{1mm}}c@{} }
\sidecapY{ Original $X^{0[1]}$ }  & 
\myimgY{star/fleur_1} &
\myimgY{star/wheat_1} &
\myimgY{star/parrot_1} \\
\sidecapY{$\rho=(1,0)$ } & 
\myimgY{barycenter/DiagYfleurrho-1-ksum11lambda00009nnx4QP1.jpeg} &
\myimgY{barycenter/wheat/Diag1-wheatrho-1-ksum13lambda001nnx4QP1.jpeg} &
\myimgY{barycenter/parrot/Diag1-parrotrho-1-ksum1lambda0001nnx4QP1.jpeg} \\
\sidecapY{$\rho=(0.7,0.3)$} & 
\myimgY{barycenter/DiagYfleurrho-06-ksum11lambda00009nnx4QP1} &
\myimgY{barycenter/wheat/Diag1-wheatrho-06-ksum13lambda001nnx4QP1}  &
\myimgY{barycenter/parrot/Diag1-parrotrho-06-ksum1lambda0001nnx4QP1} \\
\sidecapY{ $\rho=(0.4,0.6)$ } & 
\myimgY{barycenter/DiagYfleurrho-03-ksum11lambda00009nnx4QP1} &
\myimgY{barycenter/wheat/Diag1-wheatrho-03-ksum13lambda001nnx4QP1} &
\myimgY{barycenter/parrot/Diag1-parrotrho-03-ksum1lambda0001nnx4QP1} \\
\sidecapY{ $\rho=(0,1)$ } & 
\myimgY{barycenter/DiagYfleurrho-0-ksum11lambda00009nnx4QP1} &
\myimgY{barycenter/wheat/Diag1-wheatrho-0-ksum13lambda001nnx4QP1} &
\myimgY{barycenter/parrot/Diag1-parrotrho-0-ksum1lambda0001nnx4QP1} \\
\sidecapY{ Original $X^{0[2]}$ } & 
\myimgY{star/fleur_2} &
\myimgY{star/wheat_2} &
\myimgY{star/parrot_2} \\
& (a) &  (b) & (c)
\end{tabular}
\caption{Results for the barycenter algorithm on different images computed with the method proposed in Section~\ref{algobarysobolev}. The parameters were set to \textbf{(a)} $k=1.1,\la=0.0009$, \textbf{(b)} $k=1.3,\la=0.01$, and  \textbf{(b)} $k=1,\la=0.001$. Note how as $\rho$ approaches $(0,1)$, the histogram of the barycenter image becomes similar to the histogram of $X^{0[2]}$.}
\label{im:bar}
\end{figure*}

\paragraph{Example on natural images} Fig.~\ref{im:bar} shows the results of the same experiment as in Fig.~\ref{im:baryillu}, but on the natural images labeled as $X^{0[1]}$ and $X^{0[2]}$ in rows $\#1$ and $\#6$. In this case, we only show the transport from $X$ to $X^{0[1]}$, that is to say, we maintain the geometry of $X^{0[1]}$ (row $\#1$) and match its histogram to the barycenter distribution. As in the previous experiment, note how the colors change smoothly from $(1,0)$ to $(0,1)$ without generating artifacts and match the color and contrast of image $X^{0[2]}$ for $\rho=(0,1)$. The change in contrast is specially visible for the (b) wheat image.

%%%%%%%
\paragraph{Color Normalization} 

Computing the barycenter distribution of the histograms of a set of images is useful for color normalization. We show in Figures~\ref{barflower}, and~\ref{barclock} the results obtained with Algorithm~\ref{alg-norm}, and compare them with the standard OT and the method proposed by Papadakis et al.~\cite{Papadakis_ip11}. The improvement of the relaxation and regularization is specially noticeable in Figures~\ref{barflower} where OT creates artifacts such as coloring the leaves on violet for Figure~\ref{barflower}~(a), or introducing new colors on the background in Figure~\ref{barflower}~(c). In Figure~\ref{barclock}, OT and Papadakis et al.'s method introduce artifacts mostly on the sky of Figure~\ref{barclock}~(a) and  Figure~\ref{barclock}~(b), while the relaxed and regularized version displays a smoother result for Figure~\ref{barclock}~(a) and~(c) and a more meaningful color transformation (all the clouds have the same color in the fourth row) for Figure~\ref{barclock}~(b).

\newcommand{\myimgZ}[1]{\includegraphics[width=.28\linewidth,height=.26\linewidth]{#1}}
\newcommand{\myTriTab}[1]{ \begin{tabular}{@{}c@{\hspace{2mm}}c@{\hspace{2mm}}c@{} } #1	\end{tabular} }

\begin{figure*}[ht]
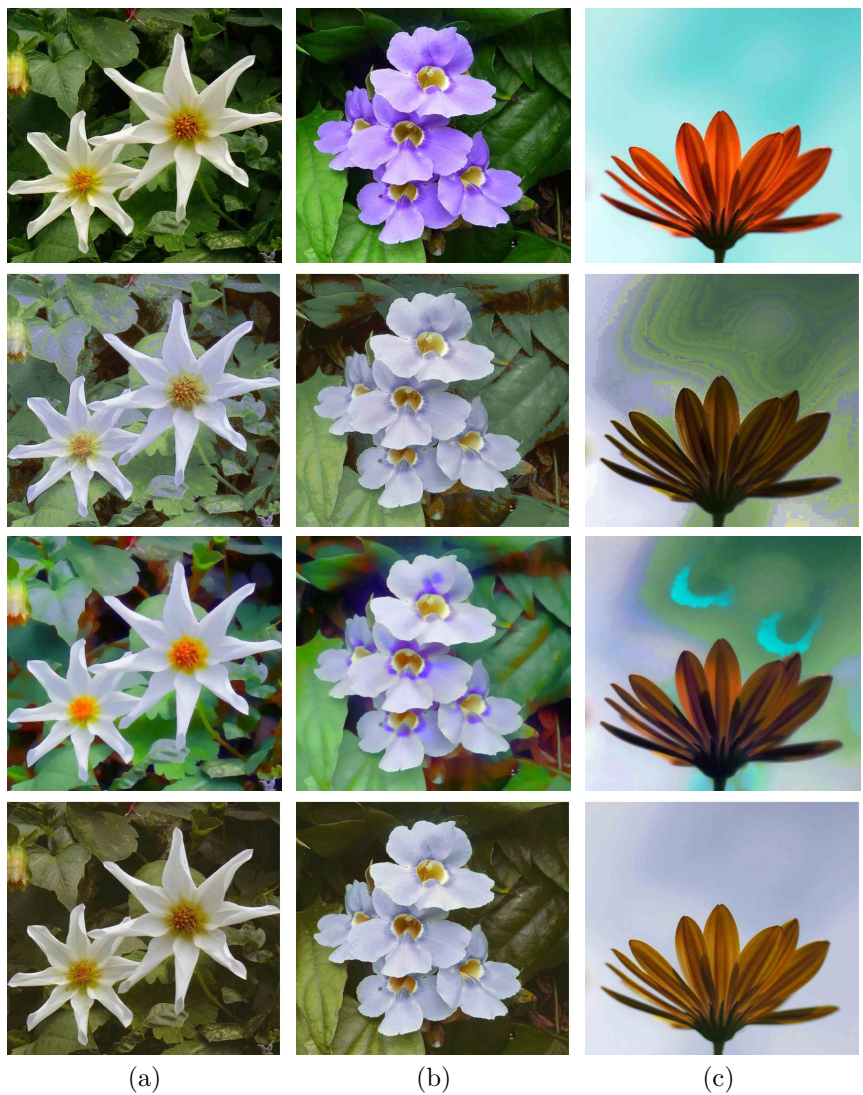

\centering
\myTriTab{ 
\myimgZ{./barycenter/flowers/flowers-1} &
\myimgZ{./barycenter/flowers/flowers-2} &
\myimgZ{./barycenter/flowers/flowers-3} \\
\myimgZ{./barycenter/flowers/Diag1-OT} &
\myimgZ{./barycenter/flowers/Diag2-OT} &
\myimgZ{./barycenter/flowers/Diag3-OT} \\
\myimgZ{./barycenter/nicoflowers1} &
\myimgZ{./barycenter/nicoflowers2} &
\myimgZ{./barycenter/nicoflowers3} \\
\myimgZ{./barycenter/flowers/Diag1-flowersrho-033333-ksum2lambda0005nnx4QP1} &
\myimgZ{./barycenter/flowers/Diag2-flowersrho-033333-ksum2lambda0005nnx4QP1} &
\myimgZ{./barycenter/flowers/Diag3-flowersrho-033333-ksum2lambda0005nnx4QP1} \\ 
(a) & (b) & (c)
}
\caption{In the first row, we show the original images. In the following rows, we show the result of computing the barycenter histogram and imposing it on each of the original images, with different algorithms. In the second row, we use OT. In the third row, the results were obtained with the method proposed by Papadakis et al.~\cite{Papadakis_ip11}. On the last row, we show the results obtained with the relaxed and regularized OT barycenter with $k=2,\la=0.005$. Note how the proposed algorithm is the only one that does not produce artifacts on the final images such as (a) color artifacts on the leaves and (c) different colors on the background.}
\label{barflower}
\end{figure*}

\begin{figure*}[ht]
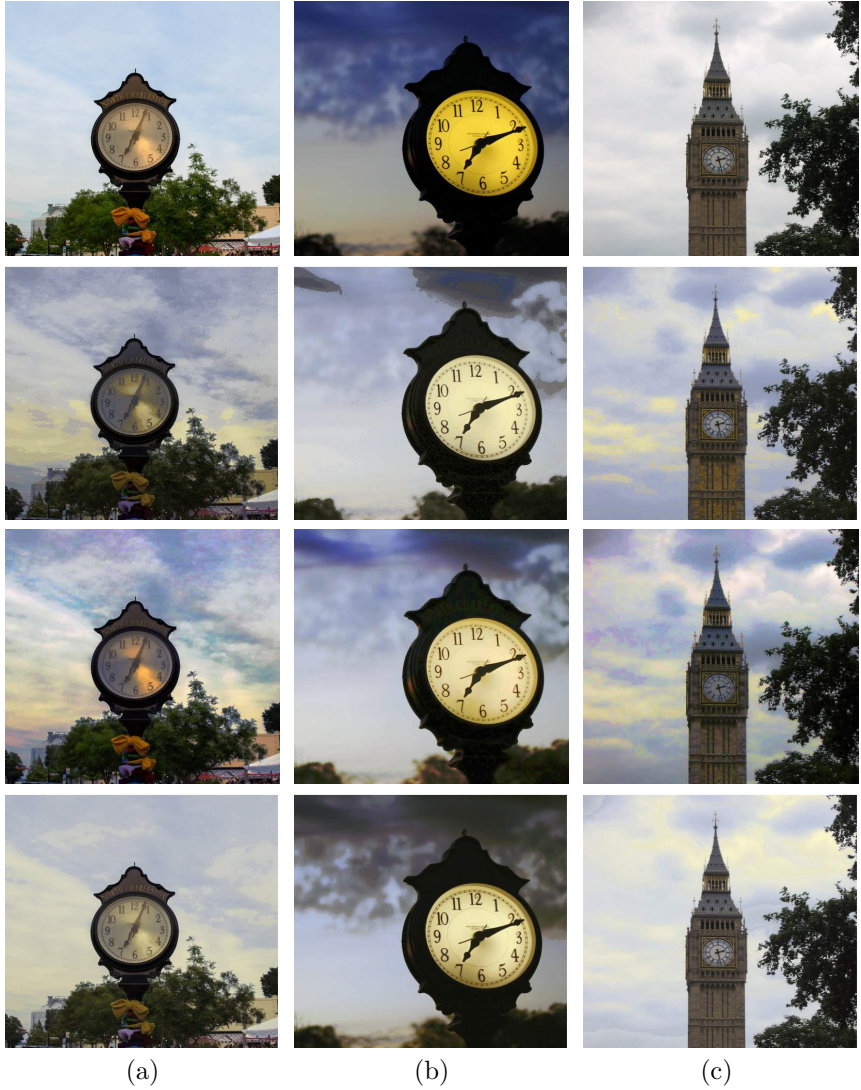

\centering
\myTriTab{ 
\myimgZ{./barycenter/clockmontague-1} &
\myimgZ{./barycenter/clockmontague-2} &
\myimgZ{./barycenter/clockmontague-3} \\
\myimgZ{./barycenter/Diag1-clockmontaguerho-033333-ksum1lambda0nnx4QP1} & %OT
\myimgZ{./barycenter/Diag2-clockmontaguerho-033333-ksum1lambda0nnx4QP1} &
\myimgZ{./barycenter/Diag3-clockmontaguerho-033333-ksum1lambda0nnx4QP1} \\
\myimgZ{./barycenter/nicoclock1} &
\myimgZ{./barycenter/nicoclock2} &
\myimgZ{./barycenter/nicoclock3} \\
\myimgZ{./barycenter/Diag1-clockmontaguerho-033333-ksum13lambda00005nnx4QP1} &
\myimgZ{./barycenter/Diag2-clockmontaguerho-033333-ksum13lambda00005nnx4QP1} &
\myimgZ{./barycenter/Diag3-clockmontaguerho-033333-ksum13lambda00005nnx4QP1} \\
(a) & (b) & (c)
}
\caption{Same experiment as in Figure~\ref{barflower}, but setting for the final row $k=1.3,\la=0.0005$. Note how the proposed method does not create artifacts on the sky and the clock for images (a) and (c), as OT or the method proposed by Papadakis et al.~\cite{Papadakis_ip11}.}
\label{barclock}
\end{figure*}

\begin{figure*}[ht]
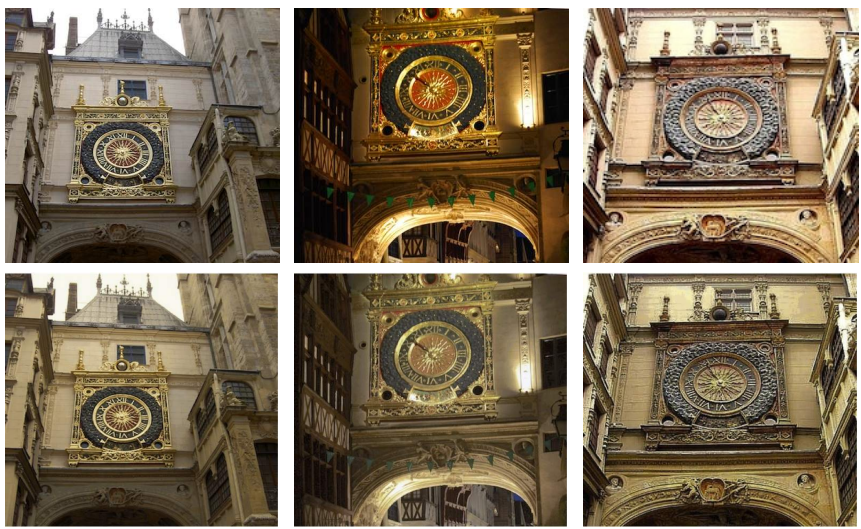

\centering
\myTriTab{ 
\myimgZ{./barycenter/clockHD-1} &
\myimgZ{./barycenter/clockHD-2} &
\myimgZ{./barycenter/clockHD-3} \\
\myimgZ{./barycenter/Diag1-clockHDrho-033333-ksum11lambda00005nnx4QP1} &
\myimgZ{./barycenter/Diag2-clockHDrho-033333-ksum11lambda00005nnx4QP1} &
\myimgZ{./barycenter/Diag3-clockHDrho-033333-ksum11lambda00005nnx4QP1} \\
}
\caption{The proposed method can be applied as a preprocessing step in a pipeline for objects detection or image registration, where canceling illumination is important. On the first row, we show a set of pictures of the same object taken at different hours of the day or night, and on the second row, the result of our algorithm setting  $(p,q)=(2,2)$, $k=1$ and $\la=0.0005$. Note how the algorithm is able to normalize the illumination conditions of all the images.}
\label{im:colornorm}
\end{figure*}

As a final example, we would like to show in Figure~\ref{im:colornorm} how this method can be applied as a preprocessing before comparing/registering images of the same object obtained under different illumination conditions.  

\section*{Conclusion}

In this paper, we have proposed a generalization of the discrete optimal transport that enables to relax the mass conservation constraint and to regularize the transport map. We showed how this novel class of transports can be applied to color transfer and that regularization is crucial to reduce noise amplification artifacts, while relaxation enables to cope with mass variation of the modes of the color palettes. We have extended these ideas to compute the relaxed and regularized barycenter of a set of input distributions. We illustrate the usefulness of this novel barycenter to perform color palette normalization of a group of input images.

\section*{Acknowledgements}
The authors would like to thank Julien Rabin for advises on color transfer and stimulating discussions.
%\label{app:perComponent}

\bibliographystyle{siam}
\bibliography{refs}

\begin{thebibliography}{10}

\bibitem{Carlier_wasserstein_barycenter}
{\sc M.~Agueh and G.~Carlier}, {\em Barycenters in the wasserstein space}, SIAM
  Journal on Mathematical Analysis, 43 (2011), pp.~904--924.

\bibitem{Almohamad-graph-match}
{\sc H.~A. Almohamad and S.~O. Duffuaa}, {\em A linear programming approach for
  the weighted graph matching problem}, IEEE Transactions on Pattern Analysis
  and Machine Intelligence, 15 (1993), pp.~522--525.

\bibitem{Amestoy09}
{\sc R.~Palma Amestoy, E.~Provenzi, M.~Bertalm\'{\i}o, and V.~Caselles}, {\em A
  perceptually inspired variational framework for color enhancement}, IEEE
  Transactions on Pattern Analysis and Machine Intelligence, 31 (2009),
  pp.~458--474.

\bibitem{Belongie-graph-match}
{\sc S.~Belongie, J.~Malik, and J.~Puzicha}, {\em Shape matching and object
  recognition using shape contexts}, IEEE Transactions on Pattern Analysis and
  Machine Intelligence, 24 (2002), pp.~509--522.

\bibitem{Benamou00}
{\sc J.-D. Benamou and Y.~Brenier}, {\em A computational fluid mechanics
  solution of the {M}onge-{K}antorovich mass transfer problem}, Numerische
  Mathematik, 84 (2000), pp.~375--393.

\bibitem{Bertsekas1988}
{\sc D.P. Bertsekas}, {\em The auction algorithm: A distributed relaxation
  method for the assignment problem}, Annals of Operations Research, 14 (1988),
  pp.~105--123.

\bibitem{Bonneel-displacement}
{\sc N.~Bonneel, M.~{van de Panne}, S.~Paris, and W.~Heidrich}, {\em
  Displacement interpolation using lagrangian mass transport}, ACM Transactions
  on Graphics (Proceedings of SIGGRAPH Asia 2011), 30 (2011).

\bibitem{Chambolle11}
{\sc A.~Chambolle and T.~Pock}, {\em A first-order primal-dual algorithm for
  convex problems with applications to imaging}, Journal of Mathematical
  Imaging and Vision, 40 (2011), pp.~120--145.

\bibitem{Csink98}
{\sc L.~Csink, D.~Paulus, U.~Ahlrichs, and B.~Heigl}, {\em Color normalization
  and object localization}, in In Rehrmann, 1998, pp.~49--55.

\bibitem{Dantzig-Book}
{\sc G.~B. Dantzig}, {\em Linear Programming and Extensions}, Princeton
  University Press, Princeton, NJ, 1963.

\bibitem{Delon04}
{\sc J.~Delon}, {\em Midway image equalization}, Journal of Mathematical
  Imaging and Vision, 21 (2004), pp.~119--134.

\bibitem{Delon:2006}
{\sc J.~Delon}, {\em Movie and video scale-time equalization application to
  flicker reduction}, IEEE Transactions on Image Processing, 15 (2006),
  pp.~241--248.

\bibitem{elmoataz-graph}
{\sc A.~Elmoataz, O.~Lezoray, and S.~Bougleux}, {\em Nonlocal discrete
  regularization on weighted graphs: A framework for image and manifold
  processing}, IEEE Transactions on Image Processing, 17 (2008),
  pp.~1047--1060.

\bibitem{2013-ssvm-regul-ot}
{\sc S.~Ferradans, N.~Papadakis, G.~Peyr{\'e}, and J-F. Aujol}, {\em
  Regularized discrete optimal transport}, in Scale Space and Variational
  Methods in Computer Vision, SSVM'13, 2013, pp.~428--439.

\bibitem{2013-ssvm-mixing}
{\sc S.~Ferradans, G-S. Xia, G.~Peyr{\'e}, and J-F. Aujol}, {\em Static and
  dynamic texture mixing using optimal transport}, in Scale Space and
  Variational Methods in Computer Vision, SSVM'13, 2013, pp.~137--148.

\bibitem{gangbo1998optimal}
{\sc W.~Gangbo and A.~{\'S}wi\c{e}ch}, {\em Optimal maps for the
  multidimensional {M}onge-{K}antorovich problem}, Communications on Pure and
  Applied Mathematics, 51 (1998), pp.~23--45.

\bibitem{guilboa07}
{\sc G.~Gilboa and S.~Osher}, {\em Nonlocal linear image regularization and
  supervised segmentation}, SIAM Multiscale Modeling and Simulation, 6 (2007),
  pp.~595---630.

\bibitem{Gonzalez:2001}
{\sc R.~C. Gonzalez and R.~E. Woods}, {\em Digital Image Processing},
  Addison-Wesley Longman Publishing Co., Inc., Boston, MA, USA, 2nd~ed., 2001.

\bibitem{Kuhn-hungarian}
{\sc {H. W. Kuhn}}, {\em The {Hungarian} method for the assignment problem},
  Naval Research Logistic Quarter, 2 (1955), pp.~83--97.

\bibitem{haker-ijcv}
{\sc S.~Haker, L.~Zhu, A.~Tannenbaum, and S.~Angenent}, {\em Optimal mass
  transport for registration and warping}, International Journal of Computer
  Vision, 60 (2004), pp.~225--240.

\bibitem{Kantorovitch-OT}
{\sc L.~Kantorovitch}, {\em On the translocation of masses}, Management
  Science, 5 (1958), pp.~1--4.

\bibitem{Land:71}
{\sc E.~H. Land and J.~J. McCann}, {\em Lightness and retinex theory}, Journal
  of the Optical Society of America, 61 (1971), pp.~1--11.

\bibitem{Lloyd57}
{\sc S.P. Lloyd}, {\em Least squares quantization in {PCM}}, IEEE Transactions
  on Information Theory, IT-20, 28 (1982), pp.~129--137.

\bibitem{louet-regularizaton-1d}
{\sc J.~Louet and F.~Santambrogio}, {\em A sharp inequality for transport maps
  in via approximation}, Applied Mathematics Letters, 25 (2012), pp.~648 --
  653.

\bibitem{mccann1997convexity}
{\sc R.J. McCann}, {\em A convexity principle for interacting gases}, Advances
  in Mathematics, 128 (1997), pp.~153--179.

\bibitem{McCollum07}
{\sc A.~J. McCollum and W.~F. Clocksin}, {\em Multidimensional histogram
  equalization and modification}, in International Conference on Image Analysis
  and Processing, ICIAP'07, 2007, pp.~659--664.

\bibitem{Morovic03}
{\sc J.~Morovic and P-L. Sun}, {\em Accurate 3d image colour histogram
  transformation}, Pattern Recognition Letters, 24 (2003), pp.~1725--1735.

\bibitem{Nesterov-Nemirovsky-Book}
{\sc Y.~E. Nesterov and A.~S. Nemirovsky}, {\em Interior Point Polynomial
  Methods in Convex Programming~: Theory and Algorithms}, SIAM Publishing,
  1993.

\bibitem{Papadakis_ip11}
{\sc N.~Papadakis, E.~Provenzi, and V.~Caselles}, {\em A variational model for
  histogram transfer of color images}, IEEE Transactions on Image Processing,
  20 (2011), pp.~1682--1695.

\bibitem{Pele-ICCV}
{\sc O.~Pele and M.~Werman}, {\em Fast and robust earth mover's distances}, in
  IEEE International Conference on Computer Vision, ICCV'09, 2009,
  pp.~460--467.

\bibitem{Pitie07}
{\sc F.~Piti\'{e}, A.~C. Kokaram, and R.~Dahyot}, {\em Automated colour grading
  using colour distribution transfer}, Computer Vision and Image Understanding,
  107 (2007), pp.~123--137.

\bibitem{Rabin_icip11}
{\sc J.~Rabin and G.~Peyr\'{e}}, {\em Wasserstein regularization of imaging
  problem}, in IEEE International Conderence on Image Processing, ICIP'11,
  2011, pp.~1541--1544.

\bibitem{Rabin_ssvm11}
{\sc J.~Rabin, G.~Peyr\'{e}, J.~Delon, and M.~Bernot}, {\em Wasserstein
  barycenter and its application to texture mixing}, in Scale Space and
  Variational Methods in Computer Vision, vol.~6667 of SSVM'11, 2011,
  pp.~435--446.

\bibitem{Reinhard01}
{\sc E.~Reinhard, M.~Adhikhmin, B.~Gooch, and P.~Shirley}, {\em Color transfer
  between images}, IEEE transactions on Computer Graphics and Applications, 21
  (2001), pp.~34 --41.

\bibitem{Rubner98}
{\sc Y.~Rubner, C.~Tomasi, and L.J. Guibas}, {\em A metric for distributions
  with applications to image databases}, in International Conference on
  Computer Vision, ICCV'98, 1998, pp.~59--66.

\bibitem{schellewald-ivc}
{\sc C.~Schellewald, S.~Roth, and C.~Schn\"{o}rr}, {\em Evaluation of a convex
  relaxation to a quadratic assignment matching approach for relational object
  views}, Image and Vision Computing, 25 (2007), pp.~1301--1314.

\bibitem{schrijver-book}
{\sc A.~Schrijver}, {\em Theory of linear and integer programming}, John Wiley
  \& Sons, Inc., New York, NY, USA, 1986.

\bibitem{Tai-cvpr-colortransfer}
{\sc Y-W. Tai, J.~Jia, and C-K. Tang}, {\em Local color transfer via
  probabilistic segmentation by expectation-maximization}, in IEEE Conference
  on Computer Vision and Pattern Recognition, vol.~1 of CVPR'05, 2005,
  pp.~747--754.

\bibitem{isomap}
{\sc J.~B. Tenenbaum, V.~de~Silva, and J.~C. Langford}, {\em A {G}lobal
  {G}eometric {F}ramework for {N}onlinear {D}imensionality {R}eduction},
  Science, 290 (2000), pp.~2319--2323.

\bibitem{tseng-proximal}
{\sc P.~Tseng}, {\em Convergence of a block coordinate descent method for
  nondifferentiable minimization}, Journal of Optimization Theory and
  Applications, 109 (2001), pp.~475--494.

\bibitem{Villani03}
{\sc C.~Villani}, {\em Topics in Optimal Transportation}, Graduate Studies in
  Mathematics Series, American Mathematical Society, 2003.

\bibitem{WangH04}
{\sc C-M. Wang and Y-H. Huang}, {\em A novel color transfer algorithm for image
  sequences}, Journal of Information Science and Engineering, 20 (2004),
  pp.~1039--1056.

\bibitem{Xiao:2006}
{\sc X.~Xiao and L.~Ma}, {\em Color transfer in correlated color space}, in ACM
  International Conference on Virtual Reality Continuum and Its Applications,
  VRCIA'06, 2006, pp.~305--309.

\bibitem{Zaslavskiy09}
{\sc M.~Zaslavskiy, F.~Bach, and J.-P. Vert}, {\em A path following algorithm
  for the graph matching problem}, IEEE Transactions on Pattern Analysis and
  Machine Intelligence, 31 (2009), pp.~2227--2242.

\bibitem{Zaslavskiy-graph-match}
{\sc M.~Zaslavskiy, F.~Bach, and J-P. Vert}, {\em Many-to-many graph matching:
  a continuous relaxation approach}, in European Conference on Machine Learning
  and Practice of Knowledge Discovery in Databases, ECML PKDD'10, 2010,
  pp.~515--530.

\bibitem{Yefeng-graph-match}
{\sc Y.~Zheng and D.~Doermann}, {\em Robust point matching for nonrigid shapes
  by preserving local neighborhood structures}, IEEE Transactions on Pattern
  Analysis and Machine Intelligence, 28 (2006), pp.~643--649.

\end{thebibliography}
\end{document}